\documentclass[sigconf]{acmart}
\AtBeginDocument{%
  }

\copyrightyear{2026}
\acmYear{2026}
\acmConference[KDD '26]{Proceedings of the 32nd ACM SIGKDD Conference on Knowledge Discovery and Data Mining V.1}{August 09--13, 2026}{Jeju Island, Republic of Korea}
\acmBooktitle{Proceedings of the 32nd ACM SIGKDD Conference on Knowledge Discovery and Data Mining V.1 (KDD '26), August 09--13, 2026, Jeju Island, Republic of Korea}
\acmPrice{}
\acmISBN{979-8-4007-2258-5/2026/08}

\acmDOI{10.1145/XXXXXXX.XXXXXXX}

\usepackage{enumitem}
\usepackage{graphicx}
\usepackage{caption}
\usepackage{subcaption}
\usepackage{bbm}
\usepackage{multirow}
\usepackage{wrapfig}
\usepackage{float}

\usepackage{algorithm}
\usepackage{algorithmic}

\usepackage{amsmath}
\usepackage{mathtools}
\usepackage{amsthm}

\theoremstyle{plain}
\newtheorem{theorem}{Theorem}
\newtheorem*{theorem*}{Theorem}

\newtheorem{lemma}{Lemma}
\newtheorem*{lemma*}{Lemma}

\theoremstyle{definition}

\theoremstyle{remark}

\newcommand{\method}{GradMix}

\begin{document}

\title{GradMix: Gradient-based Selective Mixup for Robust Data Augmentation in Class-Incremental Learning}

\author{Minsu Kim}
\affiliation{%
  \institution{Korea Advanced Institute of Science and Technology}
  \city{Daejeon}
  \country{Republic of Korea}
}
\email{ms716@kaist.ac.kr}

\author{Seong-Hyeon Hwang}
\affiliation{%
  \institution{Korea Advanced Institute of Science and Technology}
  \city{Daejeon}
  \country{Republic of Korea}
}
\email{sh.hwang@kaist.ac.kr}

\author{Steven Euijong Whang}
\authornote{Corresponding author.}
\affiliation{%
  \institution{Korea Advanced Institute of Science and Technology}
  \city{Daejeon}
  \country{Republic of Korea}
}
\email{swhang@kaist.ac.kr}


\begin{abstract}

In the context of continual learning, acquiring new knowledge while maintaining previous knowledge presents a significant challenge. Existing methods often use experience replay techniques that store a small portion of previous task data for training. In experience replay approaches, data augmentation has emerged as a promising strategy to further improve the model performance by mixing limited previous task data with sufficient current task data. However, we theoretically and empirically analyze that training with mixed samples from random sample pairs may harm the knowledge of previous tasks and cause greater catastrophic forgetting. We then propose GradMix, a robust data augmentation method specifically designed for mitigating catastrophic forgetting in class-incremental learning. GradMix performs gradient-based selective mixup using a class-based criterion that mixes only samples from helpful class pairs and not from detrimental class pairs for reducing catastrophic forgetting. Our experiments on various real datasets show that GradMix outperforms data augmentation baselines in accuracy by minimizing the forgetting of previous knowledge. The source code is available at: https://github.com/minsu716-kim/GradMix.
\end{abstract}

\ccsdesc[500]{Computing methodologies~Machine learning}

\keywords{Continual learning, Class-incremental learning, Data augmentation}



\maketitle


\section{Introduction}
\label{intro}

Humans can learn incrementally and accumulate knowledge as they encounter new experiences throughout their lifetime. In contrast, traditional machine learning models do not have the ability to update themselves after learning from the given data once. Continual learning, also known as lifelong learning\,\cite{DBLP:series/synthesis/2016Chen}, refers to a paradigm in artificial intelligence where a model continuously learns new tasks over time while retaining the knowledge of previously-acquired tasks. Continual learning is becoming crucial for real-world applications where the environment or tasks change over time, such as robotics\,\cite{DBLP:journals/inffus/LesortLSMFR20}, autonomous driving\,\cite{DBLP:journals/nn/VerwimpYPHMPLT23}, personalized recommendations\,\cite{DBLP:conf/sc/XieRLYXWLAXS20, DBLP:conf/sigir/CaiZDD0TZ22}, and more.

There are three main scenarios of continual learning\,\cite{DBLP:journals/corr/abs-1904-07734}: task-incremental, domain-incremental, and class-incremental learning. In this paper, we focus on class-incremental learning, which is the most challenging scenario where a model incrementally learns new classes of data over tasks. A key challenge for class-incremental learning is catastrophic forgetting\,\cite{MCCLOSKEY1989109}, where a model forgets knowledge of previously-learned classes while learning new classes. Hence, balancing the trade-off between retaining previous knowledge and adapting to new information is important to achieve effective class-incremental learning\,\cite{ABRAHAM200573, 10.3389/fpsyg.2013.00504, DBLP:conf/cvpr/KimH23}.

Several approaches have been proposed to mitigate catastrophic forgetting in class-incremental learning. One simple and effective approach is experience replay\,\cite{DBLP:journals/corr/abs-1902-10486, DBLP:conf/nips/RolnickASLW19}, which stores a subset of data from previous tasks in a buffer and uses it for training with the recently-arrived current task data. As the performance of experience replay is better than other buffer-free approaches\,\cite{DBLP:conf/eccv/LiH16, DBLP:journals/corr/KirkpatrickPRVD16, DBLP:conf/icml/ZenkePG17}, it is now a standard basis method and widely used with other approaches in the framework of class-incremental learning\,\cite{DBLP:conf/cvpr/RebuffiKSL17, DBLP:conf/nips/BuzzegaBPAC20, DBLP:conf/cvpr/YanX021}.

However, experience replay does not fully solve the problem of catastrophic forgetting. Since there is a constraint on the buffer size in the continual learning literature due to memory capacity, computational efficiency, and privacy issues\,\cite{DBLP:journals/pami/LangeAMPJLST22, DBLP:journals/pami/WangZSZ24}, we can only store a small portion of data for previous tasks in a buffer. This memory constraint leads to an imbalance between the previous task data and the current task data, where there is a large amount of current task data compared to a small amount of buffer data for previous tasks. After training with the imbalanced data at each task, the model performance of previous tasks is much lower than that of the current task.

One solution to address the problem of limited buffer data is data augmentation. For image datasets, transformations like random rotation, flipping, and cropping can be used to increase the proportion of buffer data. However, transforming images still has limitations in improving the diversity of buffer data because the source buffer size is already too small. Here mixup-based data augmentation methods may be better to improve the diversity of buffer data by mixing them with a large amount of current task data. In practice, several works empirically show that applying mixup-based methods to class-incremental learning can further improve the model performance\,\cite{DBLP:conf/cvpr/MiKLYF20, DBLP:conf/cvpr/BangKY0C21}. However, these works overlook catastrophic forgetting while applying mixup-based methods, focusing only on their augmentation and generalization capabilities. We theoretically and empirically find that applying vanilla Mixup can worsen catastrophic forgetting, where the cause is an interference between the average gradient vectors of augmented training data and previous task data when they are in the opposite direction.

\begin{figure}[t]
    \centering
    \includegraphics[width=\linewidth]{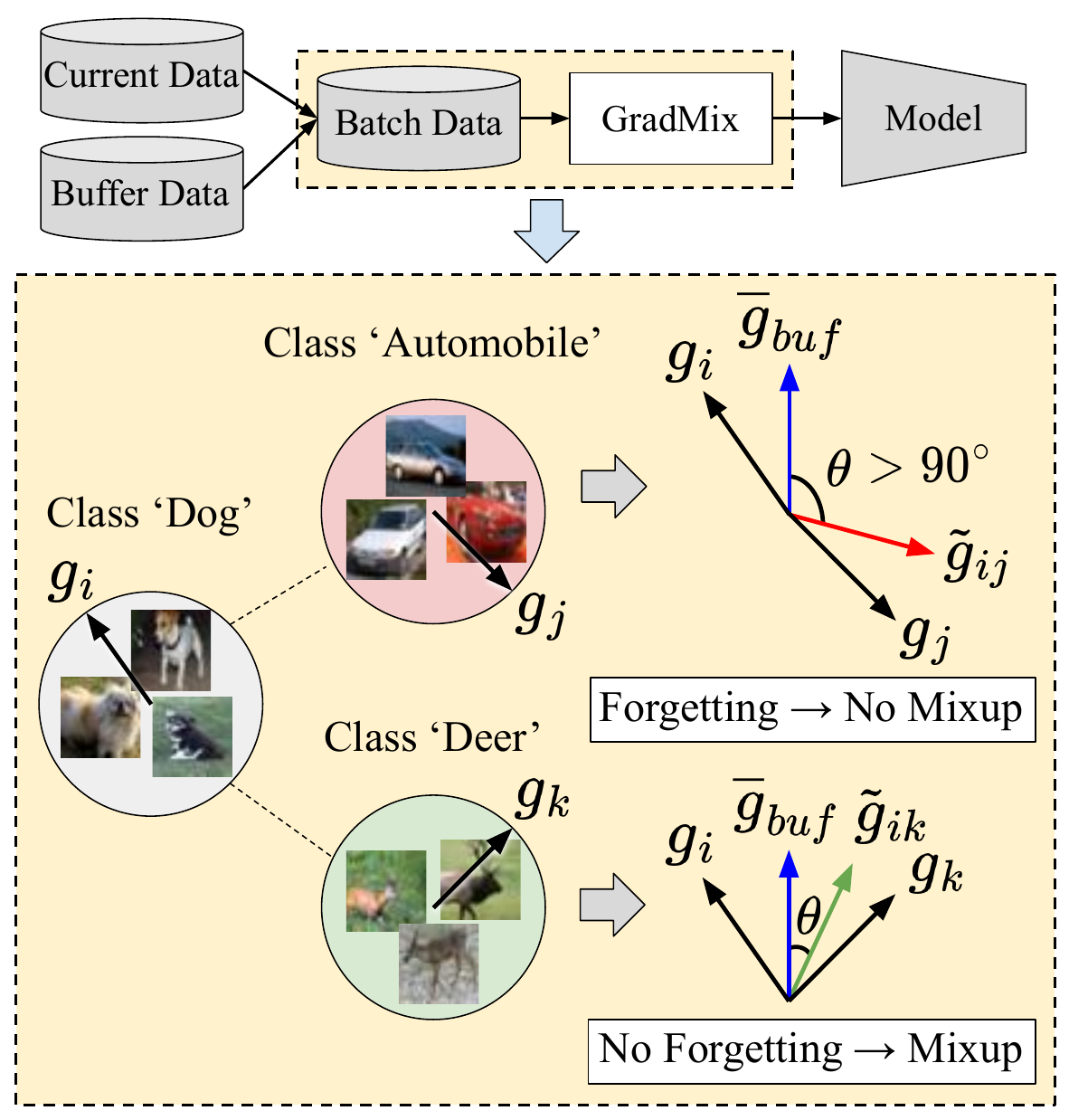}
    \caption{Overview of \method{}. The gradients of the original samples and the mixed samples are denoted by $g$ and $\tilde{g}$, respectively. Given the average gradient of buffer data (denoted as $\overline{g}_{buf}$), \method{} performs gradient-based selective mixup to mitigate catastrophic forgetting by mixing only samples from helpful class pairs (e.g., `Dog'--`Deer', where $\theta \leq 90^{\circ}$) and not from detrimental class pairs (e.g., `Dog'--`Automobile', where $\theta > 90^{\circ}$).}
    \label{fig:gradmix_overview}
    \vspace{-0.4cm}
\end{figure}

We then propose \method{}, a robust data augmentation method specifically designed for mitigating catastrophic forgetting in class-incremental learning. \method{} performs gradient-based selective mixup as shown in Fig.~\ref{fig:gradmix_overview}, instead of mixing random samples. In the experience replay-based framework, if we apply vanilla Mixup to the batch data that mixes randomly selected sample pairs, the gradients of the mixed samples can interfere (i.e., have an opposite direction) with the gradient of buffer data. This gradient interference results in catastrophic forgetting of the unstable knowledge from previous tasks due to the limited buffer data. Specifically, the magnitude of catastrophic forgetting can be approximated by the inner product of the gradient vectors of the mixed training data and the previous task data. Using this result, we propose a \textit{gradient-based selective mixup criterion} when choosing sample pairs for mixup, which only mixes helpful sample pairs and avoids mixing detrimental sample pairs to minimize catastrophic forgetting. In addition, by exploiting the observation that training data within the same class share similar gradients, we propose a class-based criterion instead of a sample-based criterion to improve efficiency. Experiments on real datasets show that \method{} obtains better accuracy results than various data augmentation baselines by using our gradient-based selective mixup algorithm.

\textbf{Summary of Contributions:} (1) We theoretically and empirically demonstrate the occurrence and cause of greater catastrophic forgetting while applying vanilla Mixup in class-incremental learning; (2) We propose \method{}, which is the first robust data augmentation method against catastrophic forgetting in class-incremental learning; (3) We show \method{} outperforms data augmentation baselines in terms of accuracy on various datasets.


\section{Related Work}
\label{related_work}


\paragraph{Class-Incremental Learning}

Class-incremental learning is a challenging scenario within the field of continual learning, focusing on incrementally learning new classes of data over time. The main problem in class-incremental learning is catastrophic forgetting\,\cite{MCCLOSKEY1989109}, where a model forgets knowledge about previously-learned classes while acquiring new classes. Existing continual learning techniques can be categorized as experience replay, parameter regularization, and dynamic networks. Experience replay methods store and replay a subset of previous data to explicitly mitigate forgetting\,\cite{DBLP:journals/corr/abs-1902-10486, DBLP:conf/nips/AljundiBTCCLP19, DBLP:conf/aaai/ShimMJSKJ21, DBLP:conf/cvpr/MaiLKS21}. Parameter regularization methods add constraints on model parameters to prevent drastic changes that could lead to forgetting\,\cite{DBLP:journals/corr/KirkpatrickPRVD16, DBLP:conf/icml/ZenkePG17, DBLP:conf/eccv/AljundiBERT18, DBLP:conf/iclr/KolouriKSP20}. Dynamic network methods dynamically expand or adapt the network architecture to learn new tasks without overwriting prior knowledge\,\cite{DBLP:conf/cvpr/YanX021, DBLP:conf/eccv/WangZYZ22, DBLP:conf/iclr/0001WYZ23}. There are also other lines of research including knowledge distillation\,\cite{DBLP:conf/eccv/LiH16, DBLP:conf/cvpr/RebuffiKSL17, DBLP:conf/cvpr/HouPLWL19} and model rectification\,\cite{DBLP:conf/cvpr/WuCWYLGF19, DBLP:conf/cvpr/ZhaoXGZX20}. Among these categories, experience replay is widely used in the class-incremental learning literature, and it is often combined with other techniques from different categories\,\cite{DBLP:conf/nips/BuzzegaBPAC20, DBLP:conf/cvpr/YanX021}. In this paper, we focus on the limitation of experience replay methods where the amount of stored buffer data is limited and propose a robust data augmentation method that minimizes the forgetting of previous tasks.


\paragraph{Data Augmentation}

Data augmentation is an effective method to enhance the generalization and robustness of models by generating additional training data. Traditional data augmentation techniques include image transformations such as random rotation, flipping, and cropping, which introduce variations to the input images\,\cite{DBLP:journals/jbd/ShortenK19, DBLP:journals/corr/abs-2204-08610}. Recent works propose an automatic strategy to find an effective data augmentation policy\,\cite{DBLP:conf/cvpr/CubukZMVL19, DBLP:conf/nips/CubukZS020, DBLP:conf/iccv/MullerH21}. However, these policy-based approaches do not consider continual learning scenarios, which may result in severe catastrophic forgetting as we show in experiments.

Another line of research is Mixup\,\cite{DBLP:conf/iclr/ZhangCDL18}, which generates new synthetic data by linearly interpolating between pairs of features and their corresponding labels, and there are variant mixup-based methods\,\cite{DBLP:conf/icml/VermaLBNMLB19, DBLP:conf/iccv/YunHCOYC19, harris2021fmixenhancingmixedsample}. In addition, imbalance-aware mixup methods have been proposed to improve model robustness on imbalanced datasets\,\cite{DBLP:conf/eccv/ChouCPWJ20, DBLP:conf/miccai/GaldranCB21}. Furthermore, selective mixup is introduced by selecting sample pairs for mixup based on specific criteria instead of randomly, and it improves the model's ability for domain adaptation\,\cite{DBLP:conf/aaai/XuZNLWTZ20} and out-of-distribution (OOD) generalization\,\cite{DBLP:conf/icml/Yao0LZL0F22}.

A recent line of research applies existing data augmentation methods in class-incremental learning to further improve accuracy. Most existing works simply combine experience replay methods with vanilla Mixup\,\cite{DBLP:conf/cvpr/MiKLYF20, DBLP:conf/nips/KumariW0B22}, RandAugment\,\cite{DBLP:conf/nips/ZhangPFBLJ22}, or a combination of existing data augmentation methods\,\cite{DBLP:conf/cvpr/BangKY0C21}. Extending this line of work,\,\cite{DBLP:conf/aaai/QiangH0L0Z23} improves the sampling and mixing factor strategies in Mixup to address distribution discrepancies between the imbalanced training set and the balanced test set in continual learning, which is similar to existing imbalance-aware mixup methods\,\cite{DBLP:conf/eccv/ChouCPWJ20, DBLP:conf/miccai/GaldranCB21}. Another work\,\cite{DBLP:conf/nips/KimXKK022} addresses class-incremental learning by using out-of-distribution detection methods that utilize image rotation as a data augmentation technique. Another line of research is using generative models to generate data for previous tasks\,\cite{DBLP:conf/nips/ShinLKK17, DBLP:conf/icml/GaoL23a, DBLP:conf/cvpr/KimCKTB24}, but this approach incurs additional costs for continuously training generative models across tasks. In addition, there is a work\,\cite{DBLP:conf/nips/ZhuCZL21} to achieve comparable performance with experience replay methods without storing buffer data by utilizing class augmentation and semantic augmentation techniques. However, all these works simply use existing data augmentation techniques for further improvements and do not propose any novel data augmentation methods specifically designed for continual learning scenarios. In comparison, \method{} is a robust data augmentation method designed to minimize catastrophic forgetting in class-incremental learning. We leave more detailed comparisons with these works in Sec.~\ref{appendix:related_work}.


\section{Preliminaries}


\paragraph{Notations}

We set notations for class-incremental learning. We denote each task as $T_l = \{d_i=(x_i, y_i)\}_{i=1}^{k}$, where $x$, $y$, and $k$ represent the feature, the true label, and the number of samples, respectively. We assume disjoint sets of classes as tasks, where there are no overlapping classes between different tasks. The sequence of tasks is given by $T = \{T_1, \ldots , T_L\}$, where $L$ is the total number of tasks. At the $l^{th}$ step, we consider $T_l$ as the current task and $\{T_1, \ldots , T_{l-1}\}$ as previous tasks. Let the set of previous classes be $\mathbb{Y}_p$ and the set of current classes be $\mathbb{Y}_c$. We store a small subset of data for the previous tasks in a buffer to employ experience replay during training. We denote the subset of each previous task as $\mathcal{M}_t = \{d_j=(x_j, y_j)\}_{j=1}^{m}$, where $m$ is the number of samples. The buffer data for the previous tasks is then represented as $\mathcal{M} = \{\mathcal{M}_1, \ldots , \mathcal{M}_{l-1}\}$. Since we cannot store large buffer data due to the possible problems of memory, efficiency, and privacy during continual learning, we assume the buffer size is small (i.e., $m \ll k$)\,\cite{DBLP:journals/corr/abs-1902-10486}. 


\paragraph{Catastrophic Forgetting}

Catastrophic forgetting is the main challenge of continual learning, where the performance on previously-learned tasks degrades while adapting to the current task. The condition for the occurrence of catastrophic forgetting is theoretically derived using the gradients of previous and current data with respect to the model\,\cite{DBLP:conf/nips/Lopez-PazR17, DBLP:conf/iclr/ChaudhryRRE19, DBLP:conf/nips/AljundiLGB19}. At the $l^{th}$ step, let $G_y$ be a group of samples whose label is $y \in \mathbb{Y}_p$, and $T_l$ be the current task data used for training. We can approximate the average loss of $G_y$ after model training with $T_l$ by employing first-order Taylor series approximation as follows:
\begin{equation} \label{eq:loss}
\tilde{\ell}(f_{\theta}, G_y) = \ell(f_{\theta}^{l-1}, G_y) - \eta \nabla_{\theta} \ell(f_{\theta}^{l-1}, G_y)^\top \nabla_{\theta} \ell(f_{\theta}^{l-1}, T_l),
\end{equation}
where $\tilde{\ell}(\cdot,\cdot)$ is the approximated average cross-entropy loss, $\ell(\cdot,\cdot)$ is the exact average cross-entropy loss, $f_{\theta}^{l-1}$ is a previous model, $f_{\theta}$ is an updated model, and $\eta$ is a learning rate. Hence, the inner product of the two average gradient vectors of the previous and current data determines whether catastrophic forgetting occurs or not. If the inner product value is negative, it leads to the increase of the loss for $G_y$ after training, which implies that catastrophic forgetting occurs for $G_y$ while learning $T_l$, and vice versa.


\paragraph{Mixup}

Mixup\,\cite{DBLP:conf/iclr/ZhangCDL18} is a data augmentation method that generates new training samples by combining random pairs of existing samples of features and their labels. Given two arbitrary samples $d_i=(x_i, y_i)$ and $d_j=(x_j, y_j)$ from the dataset, a new mixed sample $\tilde{d}_{ij} = (\tilde{x}_{ij}, \tilde{y}_{ij})$ is generated from the two original samples as follows:
\begin{equation} \label{eq:mixup}
\tilde{x}_{ij} = \lambda x_i + (1 - \lambda) x_j, \ \tilde{y}_{ij} = \lambda y_i + (1 - \lambda) y_j,
\end{equation}
where $\lambda \sim \text{Beta}(\alpha, \alpha)$, with a parameter $\alpha > 0$. Mixup often uses the value of $\lambda$ close to 0 or 1 to prevent the generation of out-of-distribution samples, and it has been shown to improve model generalization and robustness across various tasks.


\section{Catastrophic Forgetting in Mixup}


In this section, we demonstrate that simply applying Mixup to experience replay methods in class-incremental learning can lead to greater catastrophic forgetting of previously learned classes. In this paper, we use a simple yet representative experience replay method, ER\,\cite{DBLP:journals/corr/abs-1902-10486}, as a basis method for class-incremental learning, although other methods could also be used. 



\subsection{Theoretical Analysis}
\label{mixup_theoretical}

We first provide a theoretical analysis of how applying Mixup to ER may result in greater catastrophic forgetting. In the training procedure of ER, batch data consisting of both buffer data and current task data are constructed and used for training. If we apply Mixup to ER, we randomly select two samples in the batch data and mix them using Eq.~\ref{eq:mixup} to generate mixed samples for training. By employing the first-order Taylor series approximation as in Eq.~\ref{eq:loss}, we compute and compare the approximated average losses of previous tasks after training with either the original sample or the mixed sample. The following theorem provides a sufficient condition under which training with the mixed sample may lead to greater catastrophic forgetting compared to training with the original sample.
\begin{theorem} \label{th1}
Let $\tilde{d}_{ij}$ be a mixed sample by mixing an original training sample $d_i$ and a randomly selected sample $d_j$. If the gradient of the mixed sample satisfies the following condition: 
\begin{align*}
\sum_{y \in \mathbb{Y}_p} 
\nabla_{\theta} \ell(f_{\theta}^{l-1}, G_y)^\top 
\left( \nabla_{\theta} \ell(f_{\theta}^{l-1}, d_i) 
- \nabla_{\theta} \ell(f_{\theta}^{l-1}, \tilde{d}_{ij}) \right)
> 0,
\end{align*}
then training with the mixed sample leads to a higher average loss for previous tasks and worsens catastrophic forgetting compared to training with the original sample.
\end{theorem}
The proof is in Sec.~\ref{appendix:theoretical_analysis_mixup}. Although Mixup has the advantage of augmenting limited buffer data, an excessive number of mixed samples satisfying the sufficient condition may reduce the average accuracy of previous tasks during training.


\subsection{Empirical Analysis}
\label{emp_analysis}

We next perform an empirical analysis to show that mixing random pairs of samples may induce greater forgetting of previous tasks. In the experiments, we compare the model performance between two methods: ER and ER applied with Mixup. We consider not only the average task-wise accuracy, which is commonly used in evaluations, but also the individual class-wise accuracy for each task on the MNIST and CIFAR-10 datasets as shown in Fig.~\ref{fig:mnist_cifar10_forget_main}. Although MNIST is not widely used in the Mixup literature, we include it because it is a common benchmark in the continual learning literature. We configure a class-incremental learning setup, where a total of 10 classes are distributed across 5 tasks, with 2 classes per task. For the class-wise accuracy, we present the results on the last task without loss of generality.

\begin{figure}[t]
    \centering
    \begin{minipage}[t]{0.24\textwidth}
        \centering
        \includegraphics[width=\linewidth]{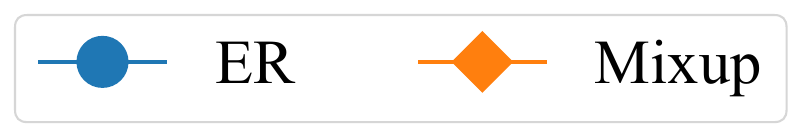}
    \end{minipage}
    \begin{minipage}[t]{0.48\textwidth}
        \centering
        \includegraphics[width=0.494\linewidth]{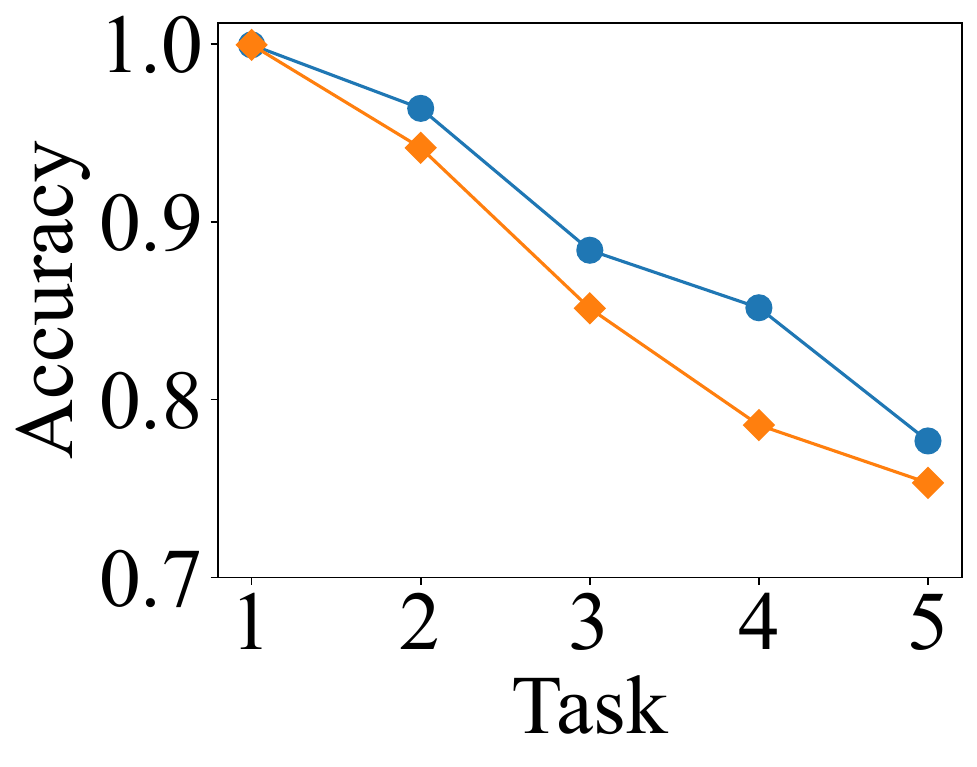}
        \includegraphics[width=0.494\linewidth]{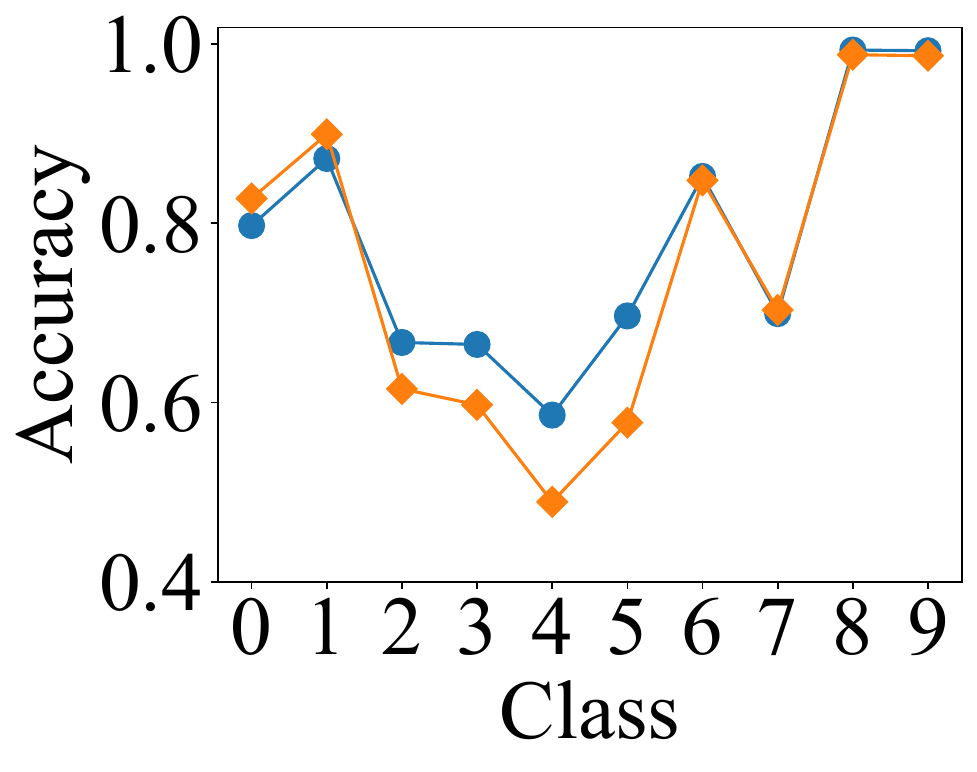}
        \vskip -0.1in
        \caption*{(a) Task-wise and class-wise accuracy (MNIST).}
        \vspace{+0.2cm}
        \label{fig:mnist_forget_main}
    \end{minipage}
    \begin{minipage}[t]{0.48\textwidth}
        \centering
        \includegraphics[width=0.494\linewidth]{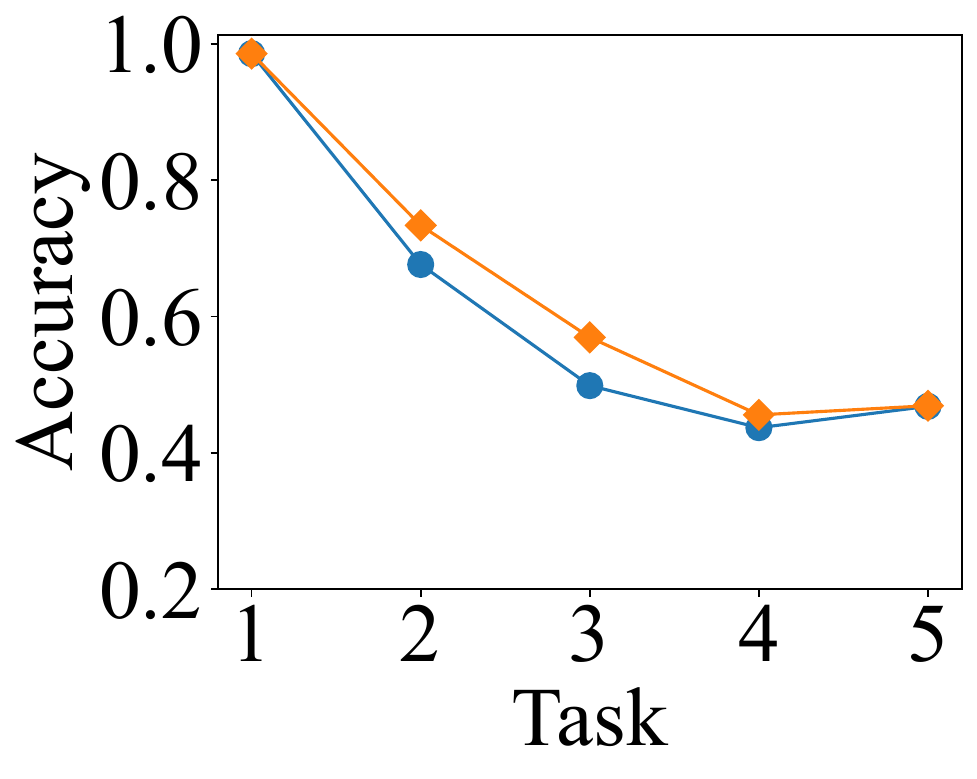}
        \includegraphics[width=0.494\linewidth]{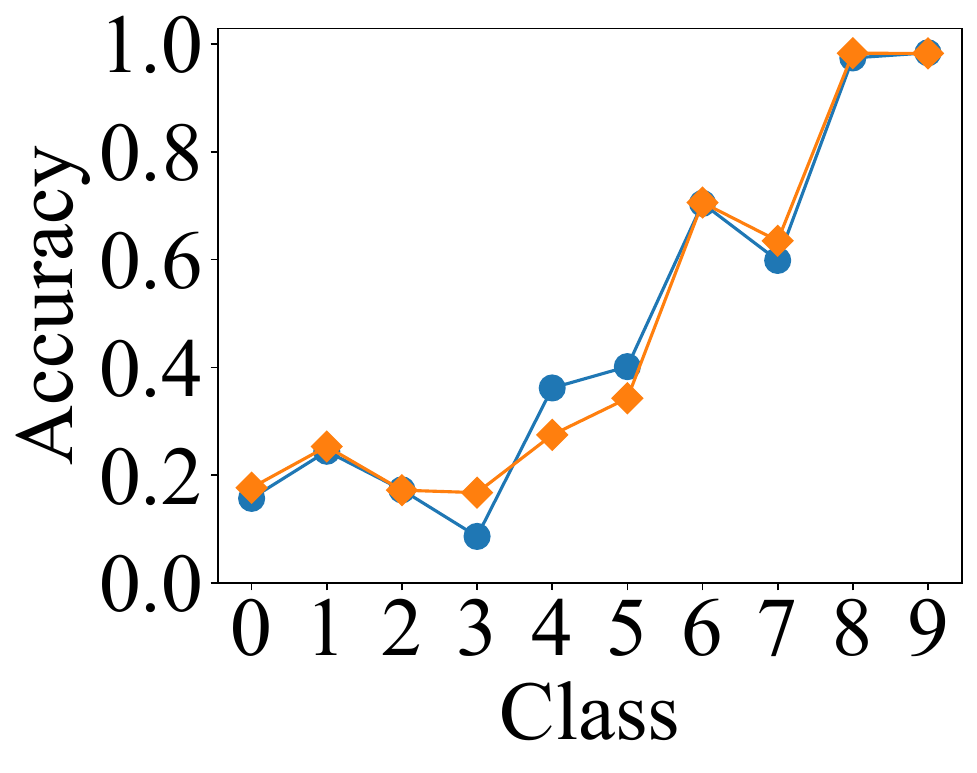}
        \vskip -0.1in
        \caption*{(b) Task-wise and class-wise accuracy (CIFAR-10).}
        \label{fig:cifar10_forget_main}
    \end{minipage}
    \vskip -0.05in
    \caption{Comparison of task-wise accuracy over all tasks and class-wise accuracy on the last task between ER and Mixup on the MNIST and CIFAR-10 datasets.}
    \label{fig:mnist_cifar10_forget_main}
    \vspace{-0.35cm}
\end{figure}

For MNIST, Mixup degrades the accuracy both in task-wise and class-wise. For CIFAR-10, Mixup helps improve task-wise accuracy, but it reduces the accuracy for some previous classes (4th and 5th) at the last task. The accuracy of the current classes (8th and 9th) is not affected by Mixup due to the sufficient amount of current task data compared to the limited buffer data from previous tasks. We find that 51.4\% and 47.2\% of mixed samples update the model with gradients that satisfy the sufficient condition for greater catastrophic forgetting (in Theorem~\ref{th1}) on the MNIST and CIFAR-10 datasets, respectively. Although Mixup sometimes improves the accuracy, we believe that the gain primarily comes from augmenting the limited buffer data, rather than from mitigating catastrophic forgetting. The results for the Fashion-MNIST, CIFAR-100, and Tiny ImageNet datasets are similar and shown in Sec.~\ref{appendix:exp-forgetting-mixup}. Based on our empirical analysis, we believe that greater catastrophic forgetting may be caused by the unpredictable gradients of randomly mixed samples, and the accuracy can be further improved by using a selective mixup. In the next section, we theoretically derive the gradients of mixed samples and propose a selective mixup algorithm to mitigate catastrophic forgetting using these gradients.

\section{GradMix}


\subsection{Problem Formulation}


To minimize the updated average loss of previous tasks by mitigating catastrophic forgetting, we formulate a data selection optimization problem that selects the optimal mixing sample corresponding to each training sample $d_i$ as follows:
\begin{equation}
\underset{d_j \in T_l \cup \mathcal{M}}{\arg\min} \frac{1}{|\mathbb{Y}_p|} \sum_{y \in \mathbb{Y}_p} \tilde{\ell}(f_{\theta}, G_y),
\end{equation}
where the updated loss after training with the mixed sample $\tilde{d}_{ij}$ is approximated as
\begin{equation}
\tilde{\ell}(f_{\theta}, G_y) = \ell(f_{\theta}^{l-1}, G_y) - \eta \nabla_{\theta} \ell(f_{\theta}^{l-1}, G_y)^\top \nabla_{\theta} \ell(f_{\theta}^{l-1}, \tilde{d}_{ij})
\end{equation}
for a group of data $G_y$ by using Eq.~\ref{eq:loss}. Since $\ell(f_{\theta}^{l-1}, G_y)$ and $\eta$ are constant values, we can ignore them and change the argmin to an argmax:
\begin{equation} \label{optimization}
\underset{d_j \in T_l \cup \mathcal{M}}{\arg\max} \biggl(\frac{1}{|\mathbb{Y}_p|} \sum_{y \in \mathbb{Y}_p} \nabla_{\theta} \ell(f_{\theta}^{l-1}, G_y)\biggr)^\top \nabla_{\theta} \ell(f_{\theta}^{l-1}, \tilde{d}_{ij}),
\end{equation}
where the first term is the average gradient of buffer data with previous classes, and the second term is the gradient of the mixed sample. To solve the optimization problem, we derive the gradient of the mixed sample in the next section.

\subsection{Gradients of Mixed Samples}
\label{gradient_mixup}

We first compute the approximated gradient of the original sample by employing the last layer approximation\,\cite{DBLP:conf/icml/KatharopoulosF18, DBLP:conf/iclr/AshZK0A20, DBLP:conf/icml/MirzasoleimanBL20} that only considers the last layer gradients, which is known to be reasonable and efficient\,\cite{DBLP:conf/aaai/KillamsettySRI21, DBLP:conf/icml/KillamsettySRDI21, DBLP:conf/aaai/KimHW24}. The following Lemma derives the approximated gradient vector of the original sample.

\begin{lemma}
Let $w$ and $b$ denote the weights and biases of the last layer of the model. By using the last layer approximation, the gradient of the original sample $d = (x, y)$ is approximated as follows:
\begin{equation} \label{original_grad}
g = (\nabla_{b} \ell(f_{\theta}^{l-1}, d), \nabla_{w} \ell(f_{\theta}^{l-1}, d))^\top = (\hat{y} - y, (\hat{y} - y)^\top X)^\top,
\end{equation}
where $X$ is the feature embedding before the last layer, and $\hat{y}$ is the model output of $x$ with respect to the model $f_{\theta}^{l-1}$.
\end{lemma}
The proof is in Sec.~\ref{appendix:theoretical_analysis_gradient}. We next derive the approximated gradient vector of a mixed sample. Similar to the works that theoretically analyze Mixup\,\cite{DBLP:conf/iclr/ZhangDKG021, DBLP:conf/icml/Yao0LZL0F22, DBLP:conf/icml/ZhangDK022}, suppose we obtain a linear model from the training data. Specifically, the model follows the Fisher’s linear discriminant analysis\,\cite{10.5555/59551, 0ceee346-049e-3934-a519-92e72b194bcf, 10.1214/20-AOS2012} where the optimal classification rule is linear. Under these assumptions, we present the following theorem that derives the approximated gradient vector of a mixed sample.

\begin{theorem}
If the model is linear, the gradient of the mixed sample $\tilde{d}_{ij} = (\tilde{x}_{ij}, \tilde{y}_{ij})$, generated by mixing two samples $d_i = (x_i, y_i)$ and $d_j = (x_j, y_j)$ with a mixing ratio $\lambda$, is approximated as follows:
\begin{equation} \label{mixed_grad}
\begin{aligned}
\tilde{g}_{ij} &= (\nabla_{b} \ell(f_{\theta}^{l-1}, \tilde{d}_{ij}), 
                    \nabla_{w} \ell(f_{\theta}^{l-1}, \tilde{d}_{ij}))^\top \\
&= \Bigl(
    \lambda (\hat{y}_{ij} - y_i) + (1 - \lambda) (\hat{y}_{ij} - y_j), \\
&\hspace{1.3em} (\lambda (\hat{y}_{ij} - y_i) + (1 - \lambda) (\hat{y}_{ij} - y_j))^\top X_{ij}
\Bigr)^\top
\end{aligned}
\end{equation}
where $\hat{y}_{ij} = \biggl[\frac{p_1^{\lambda} q_1^{1 - \lambda}}{p_1^{\lambda} q_1^{1 - \lambda} + \cdots + p_c^{\lambda} q_c^{1 - \lambda}}, \ldots , \frac{p_c^{\lambda} q_c^{1 - \lambda}}{p_1^{\lambda} q_1^{1 - \lambda} + \cdots + p_c^{\lambda} q_c^{1 - \lambda}} \biggr]$ with $\hat{y}_i = [p_1, \ldots , p_c]$ and $\hat{y}_j = [q_1, \ldots , q_c]$ for $c$ classes. $X_{ij} = \lambda X_i + (1 - \lambda) X_j$, where $X_i$ and $X_j$ are the feature embeddings of $x_i$ and $x_j$ before the last layer.
\end{theorem}
The proof is in Sec.~\ref{appendix:theoretical_analysis_gradient}. If the value of $\lambda$ is 0 or 1, $\tilde{g}_{ij}$ corresponds to the original gradient of $g_j$ or $g_i$, respectively. However, if the model is nonlinear, and the value of $\lambda$ is between 0 and 1, there may be an approximation error between the approximated gradient and the true gradient for the mixed sample. Since we use the inner product of the average gradient of buffer data and the gradient of mixed training sample as the objective function in Eq.~\ref{optimization}, we empirically compare both the inner product values and the resulting optimal sample pairs when using the approximated gradient and the true gradient for the mixed sample. In practice, the average error in the inner product values is quite small, resulting in high precision and recall for the optimal sample pairs as shown in Sec.~\ref{appendix:exp-error}.

\begin{figure}[t]
\centering
\vspace{-0.2cm}
    \begin{minipage}[t]{0.475\textwidth}
    \centering
        \begin{algorithm}[H]
        \begin{minipage}[t]{\linewidth}
        {\bfseries Input:} Current task data $T_l$, buffer data $\mathcal{M}$, current class set $\mathbb{Y}_c$, previous class set $\mathbb{Y}_p$, previous model $f_{\theta}^{l-1}$, weights $w$ and biases $b$ in the last layer of $f_{\theta}^{l-1}$, loss function $\ell$, learning rate $\eta$, and hyperparameter $\alpha$
        \end{minipage}
        \begin{algorithmic}[1]
        \FOR {each epoch}
            \STATE /* Get the sets of negative and optimal class pairs. */
            \STATE $\overline{g}_{buf} \leftarrow (\nabla_{b} \ell(f_{\theta}^{l-1}, \mathcal{M}), \nabla_{w} \ell(f_{\theta}^{l-1}, \mathcal{M}))^\top$
            \STATE Sample $\lambda \sim Beta(\alpha, \alpha)$
            \STATE $\mathbb{Y} \leftarrow \mathbb{Y}_p \cup \mathbb{Y}_c$
            \STATE $\tilde{G}_{y_i y_j} \leftarrow \lambda G_{y_i} + (1 - \lambda) G_{y_j}, \forall y_i, y_j \in \mathbb{Y}$
            \STATE $\tilde{g}_{y_i y_j} \leftarrow (\nabla_{b} \ell(f_{\theta}^{l-1}, \tilde{G}_{y_i y_j}), \nabla_{w} \ell(f_{\theta}^{l-1}, \tilde{G}_{y_i y_j}))^\top$, $\forall y_i, y_j \in \mathbb{Y}$
            \STATE $S^- = \{(y_i, y_j) \mid \overline{g}_{buf}^\top \tilde{g}_{y_i y_j} < 0, \forall y_i, y_j \in \mathbb{Y} \}$
            \STATE $S^* = \{(y_i, y_k) \mid y_k = \underset{y_j \in \mathbb{Y}}{\arg\max} \ \overline{g}_{buf}^\top \tilde{g}_{y_i y_j}, \forall y_i \in \mathbb{Y} \}$
            \STATE /* Perform selective mixup. */
            \FOR {$B_1 \sim T_l$}
                \STATE Sample $B_2 \sim \mathcal{M}$
                \STATE $D_i \leftarrow B_1 \cup B_2$
                \STATE $D_j \leftarrow$ Randomly shuffle $D_i$
                \STATE Sample $\lambda \sim Beta(\alpha, \alpha)$
                \FOR {$d_i \in D_i$ and $d_j \in D_j$}
                    \IF {$(y_i, y_j) \in S^{-}$}
                        \STATE Randomly sample $d_k (x_k, y_k) \in T_l \cup \mathcal{M}$ that satisfies $(y_i, y_k) \in S^*$ 
                        \STATE Replace $d_j$ with $d_k$ in $D_j$
                    \ENDIF
                \ENDFOR
                \STATE $\tilde{D}_{ij} \leftarrow \lambda D_i + (1 - \lambda) D_j$
                \STATE $\theta \leftarrow \theta - \eta \nabla_{\theta} \ell(f_{\theta}^{l-1}, \tilde{D}_{ij})$
            \ENDFOR
        \ENDFOR
        \STATE Sample $\mathcal{M}_l \sim T_l$
        \STATE $\mathcal{M} \leftarrow \mathcal{M} \cup \mathcal{M}_l$
        \caption{GradMix combined with ER}
        \label{alg:gradmix}
        \end{algorithmic}
        \end{algorithm}
    \end{minipage}
\vspace{-0.3cm}
\end{figure}

\subsection{Algorithm}
\label{alg}

We now design the \method{} algorithm to solve the data selection optimization problem. By substituting Eq.~\ref{original_grad} and Eq.~\ref{mixed_grad} into Eq.~\ref{optimization} and using the last layer approximation, we can efficiently compute the inner products of the average gradient of buffer data and the gradients of all mixed samples to find the optimal mixing sample. However, formulating a sample-based criterion remains practically infeasible, as it would require explicit per-sample gradient evaluations over the entire set of mixed samples, resulting in excessive computational and memory overhead. For better efficiency, we group samples that have similar gradients and reformulate the optimization problem from data selection to group selection by replacing the gradient of a mixed sample in Eq.~\ref{optimization} with the average gradient of mixed samples from two arbitrary groups. By leveraging the observation that training samples within the same class share similar gradients as described in Sec.~\ref{appendix:exp-class}, we group samples by class and propose a class-based selective mixup criterion. If the inner product of the average gradients of buffer data and mixed samples from two classes is negative, it implies that training with the mixed samples increases the average loss of previous tasks and worsens catastrophic forgetting. In this case, we define the class pair as a negative class pair and replace the mixed samples with those generated from the optimal class pair that maximizes the inner product. Conversely, if the inner product is positive, indicating that the average loss of previous tasks does not increase, and catastrophic forgetting does not occur, we retain the mixed samples as part of the training data. Since generating mixed samples solely from the optimal class pair may reduce sample diversity from a data augmentation perspective, our design choice balances the trade-off between mitigating forgetting and enhancing diversity.

We describe the overall process of \method{} in Algorithm~\ref{alg:gradmix}. At each epoch during training, we compute the average gradient of buffer data (denoted as $\overline{g}_{buf}$) using Eq.~\ref{original_grad} (Step 3). Next, we set a mixing ratio $\lambda$ and compute the average gradients of mixed samples from all class pairs (denoted as $\tilde{g}_{y_i y_j}$) using Eq.~\ref{mixed_grad} (Steps 4--7). By computing the inner products of $\overline{g}_{buf}$ and $\tilde{g}_{y_i y_j}$ for all class pairs, we get the set of negative class pairs (Step 8) and the set of optimal class pairs for each class (Step 9). The selective mixup criterion is determined using the two sets once per epoch and applied to all batches to ensure computational efficiency. For each batch consisting of both current task data and buffer data (Steps 11--13), we first randomly shuffle the batch to generate sample pairs to be mixed and set a mixing ratio $\lambda$ following the configuration in Mixup\,\cite{DBLP:conf/iclr/ZhangCDL18} (Steps 14--15). However, if the class pairs of the sample pairs belong to the negative set, we replace them with sample pairs from the optimal class pairs (Steps 16--19). After this replacement process, we form a new batch consisting of mixed samples that satisfy our selective mixup criterion (Step 22) and update the model (Step 23). Lastly, we store a random subset of the current task data to the buffer before moving onto the next task (Steps 26--27). We provide an analysis of the computational complexity and runtime of \method{} algorithm in Sec.~\ref{appendix:exp-runtime}, which shows that \method{} is scalable by only computing the last layer gradients of a small exemplar set for each class.

\section{Experiments}
\label{exp}


We perform experiments to evaluate \method{}. All evaluations are performed on separate test sets and repeated with five random seeds, and we write the mean and standard deviation of performance results. We implement \method{} using Python and PyTorch, and all experiments are run on Intel Xeon Silver 4114 CPUs and NVIDIA TITAN RTX GPUs.


\paragraph{Metrics}

We evaluate methods using average accuracy on all the tasks. We define per-task accuracy $A_l = \frac{1}{l}\sum_{t=1}^{l} a_{l, t}$ as the accuracy at the $l^{th}$ task, where $a_{l, t}$ is the accuracy of the $t^{th}$ task after learning the $l^{th}$ task. We then measure the average accuracy across all tasks, denoted as $\overline{A_l} = \frac{1}{L}\sum_{l=1}^{L} A_l$ where $L$ represents the total number of tasks.




\paragraph{Datasets}

We use a total of five datasets for the experiments as shown in Table~\ref{tbl:datasets}: MNIST\,\cite{DBLP:journals/pieee/LeCunBBH98}, Fashion-MNIST (FMNIST)\,\cite{DBLP:journals/corr/abs-1708-07747}, CIFAR-10\,\cite{Krizhevsky2009LearningML}, CIFAR-100\,\cite{Krizhevsky2009LearningML}, and Tiny ImageNet\,\cite{DBLP:journals/ijcv/RussakovskyDSKS15}. To setup experiments for class-incremental learning, we divide the datasets into sequences of tasks consisting of equal numbers of disjoint classes and assume that the task boundaries are available\,\cite{DBLP:journals/corr/abs-1904-07734, DBLP:conf/nips/MirzadehFPG20}.


\paragraph{Models}

For MNIST and FMNIST, we train a two-layer MLP each with 256 neurons using the SGD optimizer with an initial learning rate of 0.01 for 20 epochs. For CIFAR-10, CIFAR-100, and Tiny ImageNet, we train a ResNet-18\,\cite{DBLP:conf/cvpr/HeZRS16, DBLP:conf/eccv/HeZRS16} model using the SGD optimizer with an initial learning rate of 0.1 for 50, 250, and 100 epochs, respectively. During training, we reduce the learning rate by one-tenth at 100, 150, and 200 epochs for CIFAR-100, and at 30, 60, and 90 epochs for Tiny ImageNet. The last layer of the model is shared for all the tasks for a single-head evaluation\,\cite{DBLP:journals/corr/abs-1805-09733, DBLP:conf/eccv/ChaudhryDAT18}. We train all models using cross-entropy loss.

\paragraph{Baselines}
\label{baselines}

Following the experimental setups of \cite{DBLP:conf/nips/KumariW0B22, DBLP:conf/nips/ZhangPFBLJ22}, we use ER\,\cite{DBLP:journals/corr/abs-1902-10486}, MIR\,\cite{DBLP:conf/nips/AljundiBTCCLP19}, GSS\,\cite{DBLP:conf/nips/AljundiLGB19}, DER\,\cite{DBLP:conf/nips/BuzzegaBPAC20}, FOSTER\,\cite{DBLP:conf/eccv/WangZYZ22}, and MEMO\,\cite{DBLP:conf/iclr/0001WYZ23} as basis methods for class-incremental learning and apply data augmentation techniques to them. The experiments also can be extended to include recent class-incremental learning methods\,\cite{DBLP:conf/cvpr/YanX021, DBLP:conf/iccv/ChenC23, DBLP:conf/cvpr/Zheng0YZ25}. 
We compare \method{} with three types of data augmentation baselines:
\begin{itemize}[leftmargin=*]
    \setlength\itemsep{0cm}
    \item {\bf Mixup-based methods}: {\it Mixup}\,\cite{DBLP:conf/iclr/ZhangCDL18} generates synthetic training data through linear interpolations of pairs of input features and their corresponding labels; {\it Manifold Mixup}\,\cite{DBLP:conf/icml/VermaLBNMLB19} performs linear interpolations of hidden layer representations and corresponding labels of training data; {\it CutMix}\,\cite{DBLP:conf/iccv/YunHCOYC19} generates new training data by randomly cutting and pasting patches between input images, while adjusting the corresponding labels based on the proportion of each patch.
    \item {\bf Imbalance-aware mixup methods}: {\it Remix}\,\cite{DBLP:conf/eccv/ChouCPWJ20} mixes features like {\it Mixup}, but assigns labels by providing higher weights to the minority class; {\it Balanced-MixUp}\,\cite{DBLP:conf/miccai/GaldranCB21} performs both regular and balanced sampling when selecting pairs of training data for {\it Mixup}.
    \item {\bf Policy-based method}: {\it RandAugment}\,\cite{DBLP:conf/nips/CubukZS020} randomly applies different image transformations to training data with adjustable magnitude.
\end{itemize}

\paragraph{Parameters}
\label{parameters}

We store 32 samples per class for buffering and set the batch size to 64 for all experiments. For the hyperparameter $\alpha$ used in the beta distribution, we follow the setting of Mixup\,\cite{DBLP:conf/iclr/ZhangCDL18}. Since the Mixup paper does not specify $\alpha$ values for the MNIST and FMNIST datasets, we set $\alpha$ to 1.0, which corresponds to a uniform distribution. In practice, the performance of \method{} is not sensitive to the value of $\alpha$ as shown in Sec.~\ref{appendix:exp-alpha}. See Sec.~\ref{appendix:exp-details} for more experimental details.

\begin{table}[t]
  \setlength{\tabcolsep}{6.5pt}
  \caption{For each of the five datasets, we show the total training set size, the number of features, the number of classes, and the number of tasks.}
  \centering
  \begin{tabular}{lcccc}
    \toprule
    {\bf Dataset} & {\bf Size} & {\bf \#Features} & {\bf \#Classes} & {\bf \#Tasks} \\
    \midrule
    {\sf MNIST} & 50K & 28$\times$28 & 10 & 5 \\
    {\sf FMNIST} & 60K & 28$\times$28 & 10 & 5 \\
    {\sf CIFAR-10} & 50K & 3$\times$32$\times$32 & 10 & 5 \\
    {\sf CIFAR-100} & 50K & 3$\times$32$\times$32 & 100 & 10 \\
    {\sf Tiny ImageNet} & 100K & 3$\times$64$\times$64 & 200 & 10 \\
    \bottomrule
  \end{tabular}
  \label{tbl:datasets}
  \vspace{-0.3cm}
\end{table}



\begin{table*}[t]
  \setlength{\tabcolsep}{12pt}
  \caption{Average accuracy results on the MNIST, FMNIST, CIFAR-10, CIFAR-100, and Tiny ImageNet datasets. We use ER as a basis method of class-incremental learning and compare \method{} with three types of data augmentation baselines: (1) mixup-based methods: {\it Mixup}, {\it Manifold Mixup}, and {\it CutMix}; (2) imbalance-aware mixup methods: {\it Remix} and {\it Balanced-MixUp}; and (3) a policy-based method: {\it RandAugment}.} 
  \centering
  \begin{tabular}{c|l|ccccc}
  \toprule
    {Type} & {Method} & {\sf MNIST} & {\sf FMNIST} & {\sf CIFAR-10} & {\sf CIFAR-100} & {\sf Tiny ImageNet} \\
    \midrule
    {Experience Replay} & {ER} & {0.896}\tiny{$\pm$0.002} & {0.756}\tiny{$\pm$0.004} & {0.613}\tiny{$\pm$0.009} & {0.491}\tiny{$\pm$0.019} & {0.385}\tiny{$\pm$0.009} \\
    \cmidrule{1-7}
    \multirow{3}{*}{Mixup-based} & {+ Mixup} & {0.866}\tiny{$\pm$0.005} & {0.766}\tiny{$\pm$0.009} & {0.643}\tiny{$\pm$0.008} & {0.507}\tiny{$\pm$0.010} & {0.371}\tiny{$\pm$0.008} \\
    & {+ Manifold Mixup} & {0.896}\tiny{$\pm$0.004} & {0.774}\tiny{$\pm$0.004} & {0.620}\tiny{$\pm$0.011} & {0.521}\tiny{$\pm$0.008} & {0.386}\tiny{$\pm$0.009} \\
    & {+ CutMix} & {0.878}\tiny{$\pm$0.005} & {0.789}\tiny{$\pm$0.004} & {0.516}\tiny{$\pm$0.009} & {0.404}\tiny{$\pm$0.009} & {0.344}\tiny{$\pm$0.010} \\
    \cmidrule{1-7}
    \multirow{2}{*}{Imbalance-aware} & {+ Remix} & {0.905}\tiny{$\pm$0.004} & {0.789}\tiny{$\pm$0.005} & {0.639}\tiny{$\pm$0.014} & {0.528}\tiny{$\pm$0.007} & {0.391}\tiny{$\pm$0.008} \\
    & {+ Balanced-MixUp} & {0.883}\tiny{$\pm$0.006} & {0.767}\tiny{$\pm$0.006} & {0.630}\tiny{$\pm$0.014} & {0.516}\tiny{$\pm$0.009} & {0.379}\tiny{$\pm$0.008} \\
    \cmidrule{1-7}
    {Policy-based} & {+ RandAugment} & {0.765}\tiny{$\pm$0.008} & {0.640}\tiny{$\pm$0.008} & {0.636}\tiny{$\pm$0.006} & {0.484}\tiny{$\pm$0.007} & {0.368}\tiny{$\pm$0.007} \\
    \cmidrule{1-7}
    {Selective Mixup} & {\bf + GradMix (ours)} & \textbf{{0.918}\tiny{$\pm$0.004}} & \textbf{{0.802}\tiny{$\pm$0.009}} & \textbf{{0.667}\tiny{$\pm$0.025}} & \textbf{{0.546}\tiny{$\pm$0.008}} & \textbf{{0.405}\tiny{$\pm$0.006}} \\
    \bottomrule
  \end{tabular}
  \label{tbl:main_er}
  \vspace{-0.2cm}
\end{table*}

\subsection{Model Performance Results}
\label{main-results}

We compare \method{} against other data augmentation baselines on the five datasets with respect to average accuracy as shown in Table~\ref{tbl:main_er} and Sec.~\ref{appendix:exp-comp}. Overall, \method{} consistently achieves the best accuracy results compared to the baselines for all the datasets by considering catastrophic forgetting while augmenting data. In contrast, since existing data augmentation techniques are designed for static and balanced data, they have limitations in class-incremental learning setups with dynamic and imbalanced data. We also present sequential accuracy results for each task as shown in Fig.~\ref{fig:seq_exp_main} and Sec.~\ref{appendix:exp-seq} where \method{} shows better accuracy performance than other baselines at each task. In addition, \method{} helps further improvement of advanced class-incremental learning methods, including MIR, GSS, DER, FOSTER, and MEMO as shown in Table~\ref{tbl:compatibility_main}.


\begin{figure}[t]
    \centering
    \begin{minipage}[t]{0.475\textwidth}
        \centering
        \includegraphics[width=\textwidth]{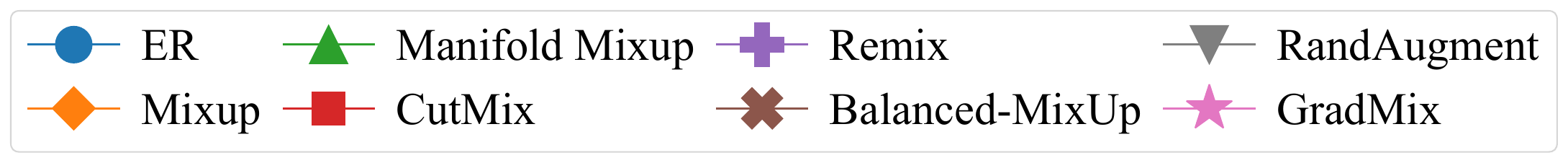}
    \end{minipage}
    \begin{minipage}[t]{0.235\textwidth}
        \centering
        \includegraphics[width=\textwidth]{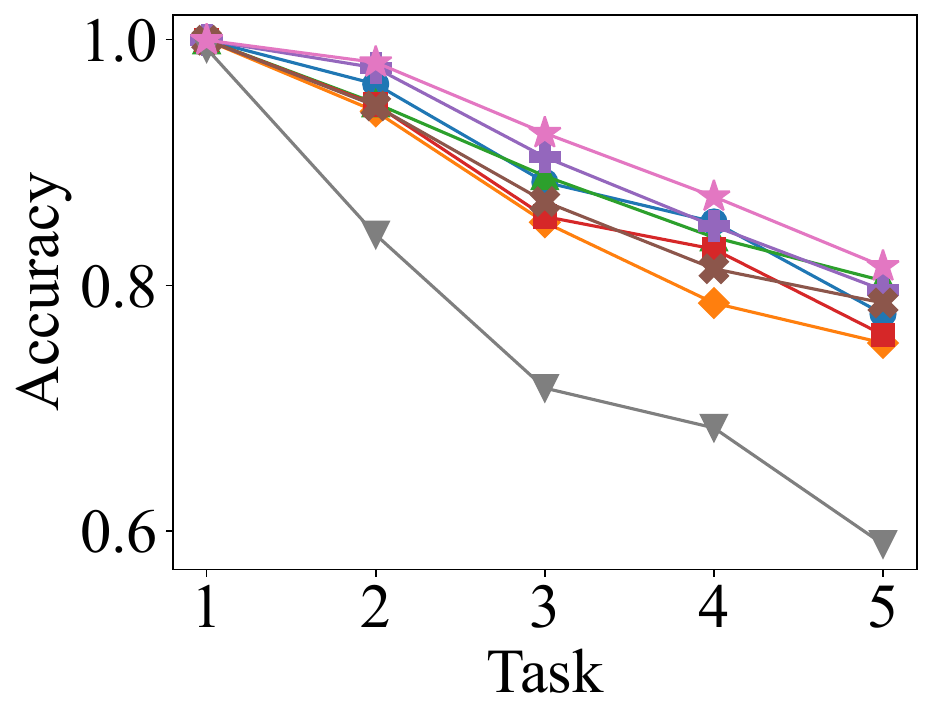}
        \vskip -0.1in
        \caption*{(a) MNIST.}
        \label{fig:seq_mnist}
    \end{minipage}
    \begin{minipage}[t]{0.235\textwidth}
        \centering
        \includegraphics[width=\textwidth]{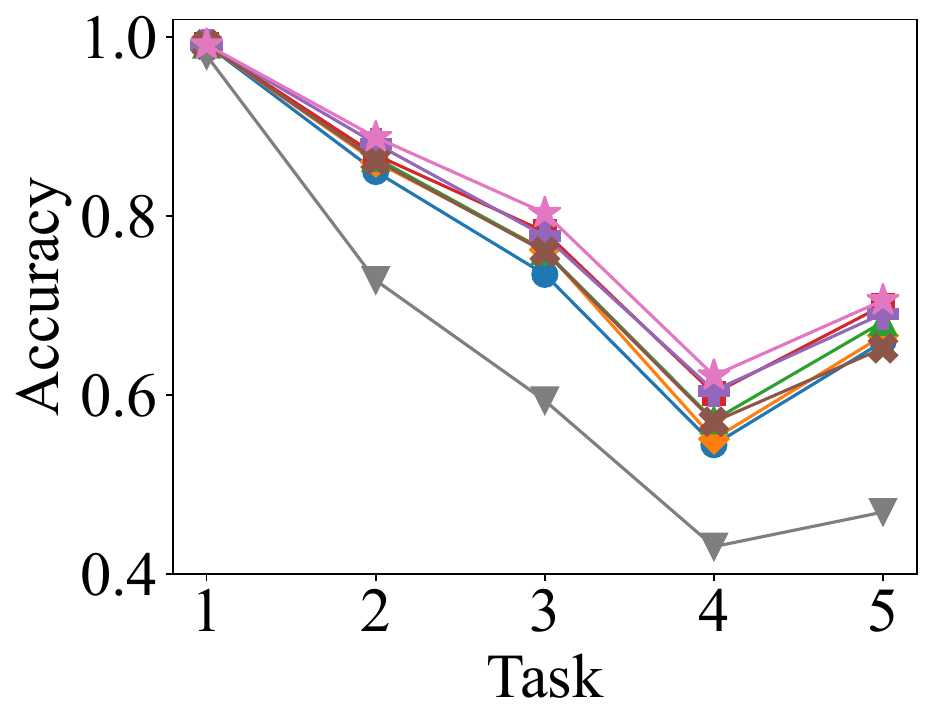}
        \vskip -0.1in
        \caption*{(b) FMNIST.}
        \label{fig:seq_fmnist}
    \end{minipage}
    \begin{minipage}[t]{0.235\textwidth}
        \centering
        \includegraphics[width=\textwidth]{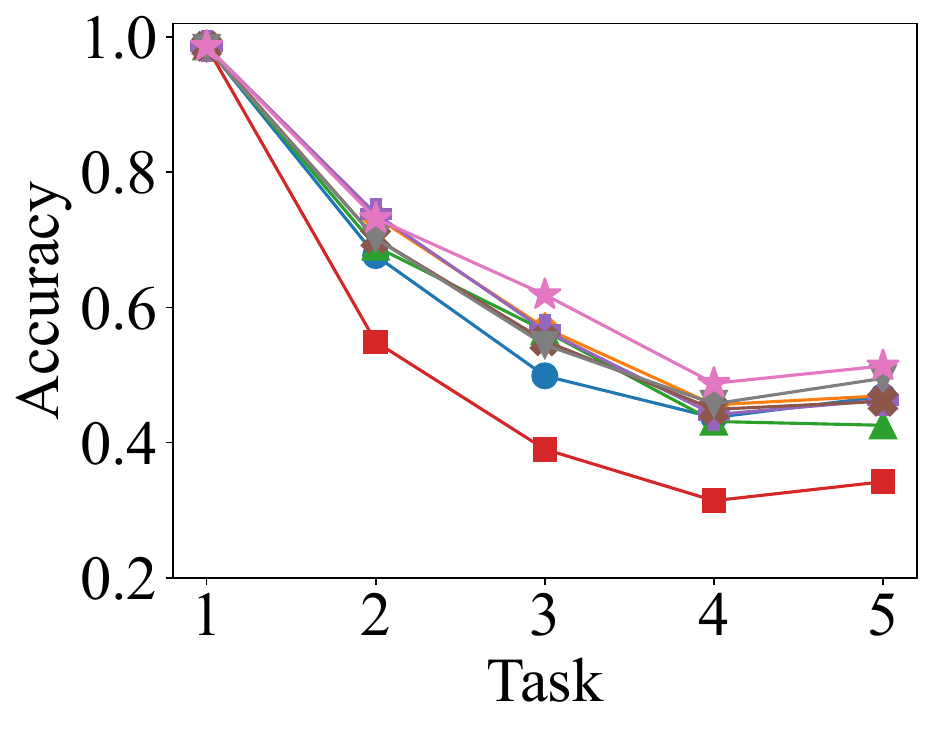}
        \vskip -0.1in
        \caption*{(c) CIFAR-10.}
        \label{fig:seq_cifar-10}
    \end{minipage}
    \vspace{0.01\textwidth}
    \begin{minipage}[t]{0.235\textwidth}
        \centering
        \includegraphics[width=\textwidth]{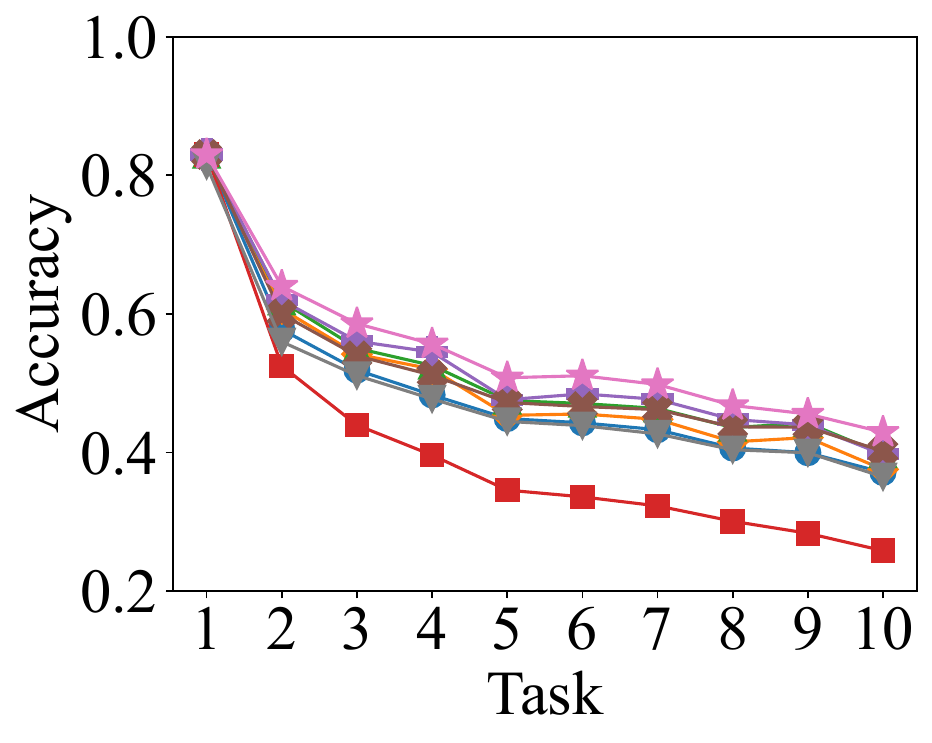}
        \vskip -0.1in
        \caption*{(d) CIFAR-100.}
        \label{fig:seq_cifar-100}
    \end{minipage}
    \vspace{-0.3cm}
    \caption{Sequential accuracy results of \method{} and baselines on the MNIST, FMNIST, CIFAR-10, and CIFAR-100 datasets as new tasks are learned.}
    \label{fig:seq_exp_main}
    \vspace{-0.4cm}
\end{figure}

\begin{figure}[t]
  \centering
  \begin{minipage}{0.475\textwidth}
    \centering
    \setlength{\tabcolsep}{8.3pt}
    \captionof{table}{Average accuracy results on the CIFAR-10, CIFAR-100, and Tiny ImageNet datasets. We use MIR, GSS, DER, FOSTER, and MEMO as basis methods of class-incremental learning and integrate \method{}.}
    \begin{tabular}{l|ccc}
      \toprule
      {Method} & {\sf CIFAR-10} & {\sf CIFAR-100} & {\sf Tiny ImageNet} \\
      \cmidrule{1-4}
      {MIR} & {0.620}\tiny{$\pm$0.009} & {0.493}\tiny{$\pm$0.021} & {0.386}\tiny{$\pm$0.009} \\
      {\bf + GradMix} & \textbf{{0.669}\tiny{$\pm$0.013}} & \textbf{{0.546}\tiny{$\pm$0.007}} & \textbf{{0.411}\tiny{$\pm$0.009}} \\
      \cmidrule{1-4}
      {GSS} & {0.627}\tiny{$\pm$0.009} & {0.512}\tiny{$\pm$0.013} & {0.397}\tiny{$\pm$0.008} \\
      {\bf + GradMix} & \textbf{{0.674}\tiny{$\pm$0.011}} & \textbf{{0.563}\tiny{$\pm$0.007}} & \textbf{{0.418}\tiny{$\pm$0.006}} \\
      \cmidrule{1-4}
      {DER} & {0.678}\tiny{$\pm$0.010} & {0.560}\tiny{$\pm$0.009} & {0.422}\tiny{$\pm$0.008} \\
      {\bf + GradMix} & \textbf{{0.714}\tiny{$\pm$0.009}} & \textbf{{0.604}\tiny{$\pm$0.006}} & \textbf{{0.452}\tiny{$\pm$0.006}} \\
      \cmidrule{1-4}
      {FOSTER} & {0.732}\tiny{$\pm$0.010} & {0.618}\tiny{$\pm$0.007} & {0.512}\tiny{$\pm$0.007} \\
      {\bf + GradMix} & \textbf{{0.769}\tiny{$\pm$0.009}} & \textbf{{0.661}\tiny{$\pm$0.006}} & \textbf{{0.547}\tiny{$\pm$0.006}} \\
      \cmidrule{1-4}
      {MEMO} & {0.751}\tiny{$\pm$0.009} & {0.634}\tiny{$\pm$0.008} & {0.520}\tiny{$\pm$0.007} \\
      {\bf + GradMix} & \textbf{{0.782}\tiny{$\pm$0.010}} & \textbf{{0.673}\tiny{$\pm$0.007}} & \textbf{{0.552}\tiny{$\pm$0.006}} \\
      \bottomrule
    \end{tabular}
    \label{tbl:compatibility_main}
  \end{minipage}
  \vspace{-0.55cm}
\end{figure}

Among the mixup-based methods, {\it Manifold Mixup} consistently shows robust results. However, another mixup-based method, {\it CutMix}, shows the worst performance on CIFAR-10, CIFAR-100, and Tiny ImageNet as it significantly reduces accuracy on previous tasks. We believe that cutting images into patches in the limited previous task data results in significant information loss. The two imbalance-aware mixup methods, {\it Remix} and {\it Balanced-MixUp}, achieve better results than {\it Mixup} as they address the data imbalance between previous and current tasks using disentangled mixing and balanced sampling techniques, respectively. The policy-based method, {\it RandAugment}, also shows underperforming results. For MNIST and FMNIST, elastic distortions across scale, position, and orientation are known to be effective, but other transformations in {\it RandAugment} like flipping and random cropping may have negative effects. For CIFAR-10, CIFAR-100, and Tiny ImageNet, the transformation technique in {\it RandAugment} that removes a region of an image is harmful to retaining the knowledge of the limited buffer data.

Considering the three types of data augmentation baselines, the two imbalance-aware mixup methods, {\it Remix} and {\it Balanced-MixUp}, show consistently solid performance against other baselines by more effectively exploiting the limited buffer data for previous tasks. However, the random selection of sample pairs for mixup leads to greater catastrophic forgetting of previous tasks and limits the improvement in average accuracy. In comparison, \method{} performs gradient-based selective mixup to minimize catastrophic forgetting of previous tasks, which leads to significantly better model accuracy results for all datasets. We also show that \method{} is effective on large-scale datasets (e.g., ImageNet-1K\,\cite{DBLP:conf/cvpr/DengDSLL009}) and variations in backbone models (e.g., ViT\,\cite{DBLP:conf/iclr/DosovitskiyB0WZ21}) in Sec.~\ref{appendix:large_scale} and Sec.~\ref{appendix:backbone_model}, respectively.

\begin{figure*}[t]
  \centering
  \begin{minipage}{1\textwidth}
    \centering
    \begin{minipage}[t]{0.24\textwidth}
      \centering
      \includegraphics[width=\linewidth]{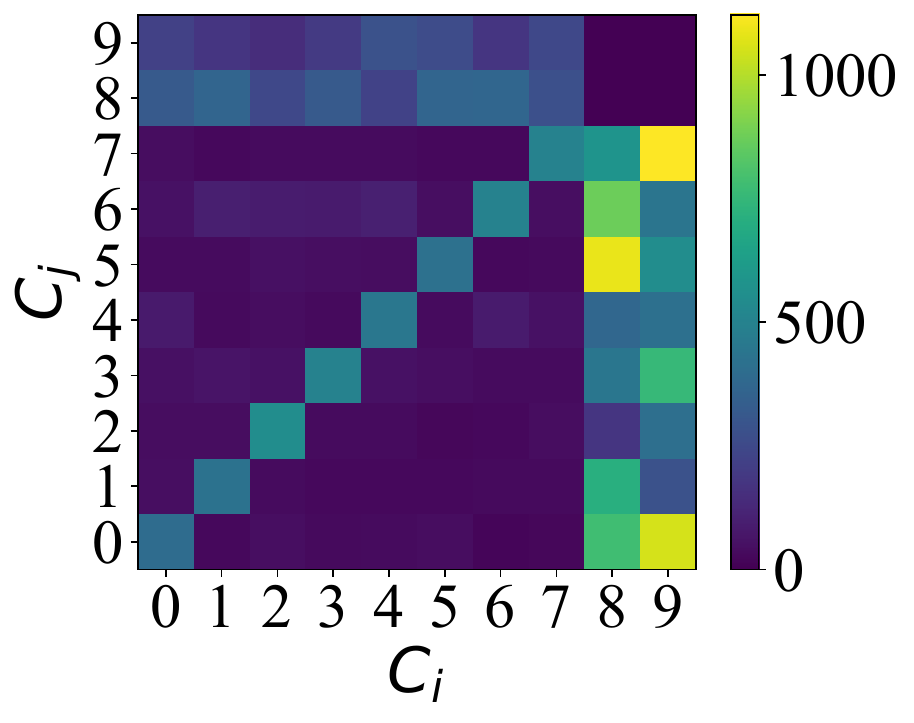}
      \vskip -0.1in
      \caption*{(a) MNIST.}
    \end{minipage}
    \begin{minipage}[t]{0.24\textwidth}
      \centering
      \includegraphics[width=\linewidth]{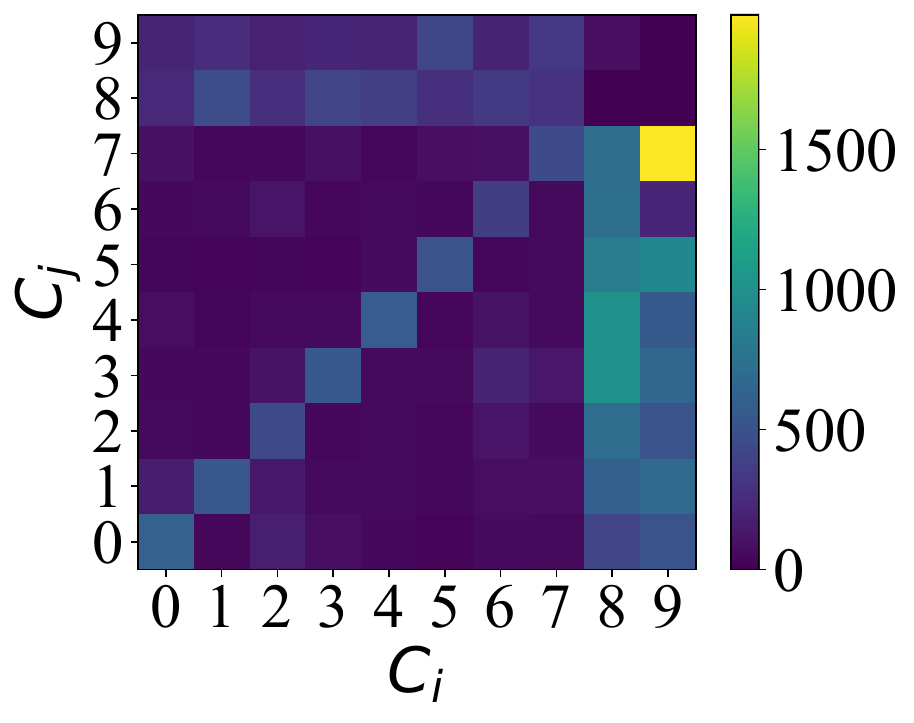}
      \vskip -0.1in
      \caption*{(b) FMNIST.}
    \end{minipage}
    \begin{minipage}[t]{0.24\textwidth}
      \centering
      \includegraphics[width=\linewidth]{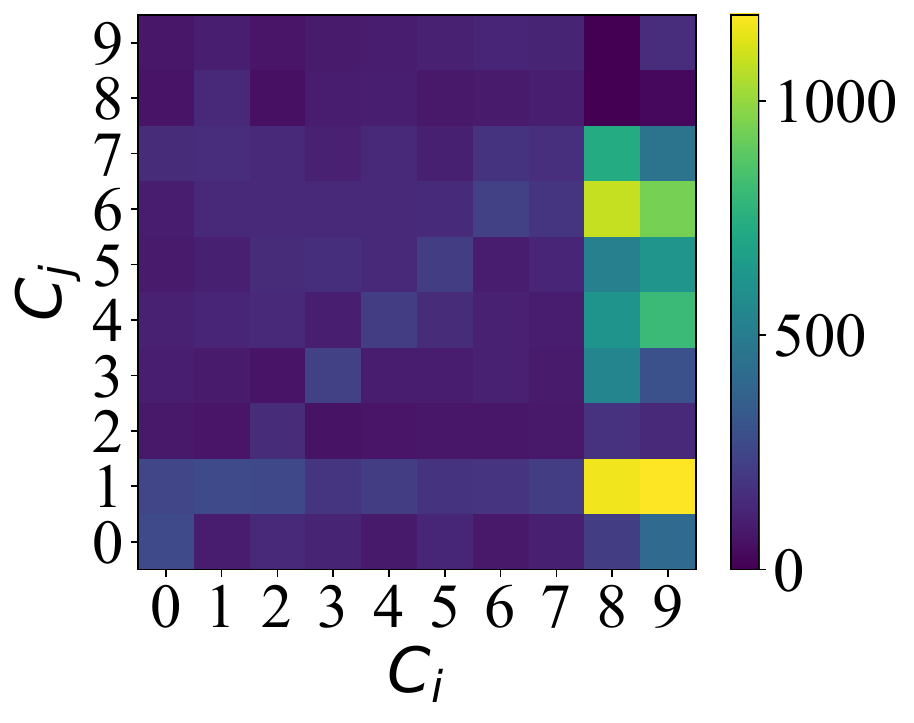}
      \vskip -0.1in
      \caption*{(c) CIFAR-10.}
    \end{minipage}
    \begin{minipage}[t]{0.23\textwidth}
      \centering
      \includegraphics[width=\linewidth]{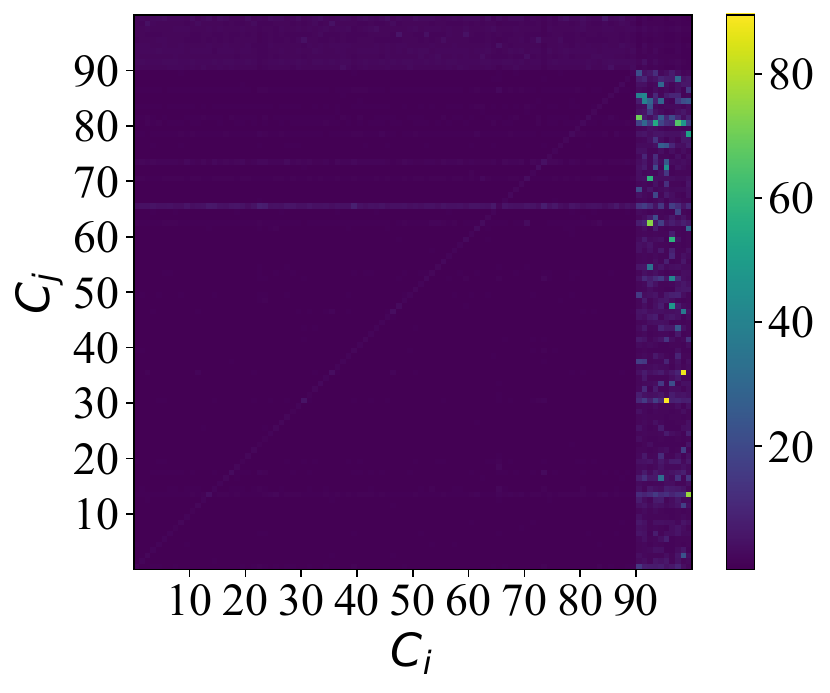}
      \vskip -0.1in
      \caption*{(d) CIFAR-100.}
    \end{minipage}
    \vskip -0.05in
    \caption{Selective mixup results of \method{} on the last task for the MNIST, FMNIST, CIFAR-10, and CIFAR-100 datasets.}
    \label{fig:heatmap_exp_main}
  \end{minipage}
\end{figure*}

\begin{table*}[t]
  \setlength{\tabcolsep}{20.2pt}
  \caption{Ablation study on the MNIST, FMNIST, CIFAR-10, CIFAR-100, and Tiny ImageNet datasets with two ablation scenarios when randomly selected sample pairs worsen catastrophic forgetting in Algorithm~\ref{alg:gradmix} (Steps 17--19): (1) using original samples without mixup; and (2) retaining random mixup. We use ER as a basis method.}
  \centering
  \begin{tabular}{l|ccccc}
  \toprule
    {Method} & {\sf MNIST} & {\sf FMNIST} & {\sf CIFAR-10} & {\sf CIFAR-100} & {\sf Tiny ImageNet} \\
    \midrule
    {Original} & {0.904}\tiny{$\pm$0.005} & {0.773}\tiny{$\pm$0.005} & {0.621}\tiny{$\pm$0.008} & {0.518}\tiny{$\pm$0.010} & {0.383}\tiny{$\pm$0.010} \\
    {Random} & {0.905}\tiny{$\pm$0.004} & {0.789}\tiny{$\pm$0.005} & {0.643}\tiny{$\pm$0.008} & {0.528}\tiny{$\pm$0.007} & {0.391}\tiny{$\pm$0.008} \\
    {\bf GradMix (ours)} & \textbf{{0.918}\tiny{$\pm$0.004}} & \textbf{{0.802}\tiny{$\pm$0.009}} & \textbf{{0.667}\tiny{$\pm$0.025}} & \textbf{{0.546}\tiny{$\pm$0.008}} & \textbf{{0.405}\tiny{$\pm$0.006}} \\
    \bottomrule
  \end{tabular}
  \label{tbl:ablation_study}
\end{table*}

\begin{table*}[t]
  \setlength{\tabcolsep}{3.65pt}
  \caption{Average accuracy results on the MNIST, FMNIST, CIFAR-10, CIFAR-100, and Tiny ImageNet datasets with respect to the buffer size per class. We use ER as a basis method of class-incremental learning and compare \method{} with various data augmentation baselines.}
  \centering
  \begin{tabular}{l|cc|cc|cc|cc|cc}
  \toprule
  {Method} 
  & \multicolumn{2}{c|}{\sf MNIST} 
  & \multicolumn{2}{c|}{\sf FMNIST} 
  & \multicolumn{2}{c|}{\sf CIFAR-10}
  & \multicolumn{2}{c|}{\sf CIFAR-100} 
  & \multicolumn{2}{c}{\sf Tiny ImageNet} \\
  \cmidrule{1-11}
  {} & {$N$=16} & {$N$=64} & {$N$=16} & {$N$=64} & {$N$=16} & {$N$=64} & {$N$=16} & {$N$=64} & {$N$=16} & {$N$=64} \\
  \midrule
  {ER} 
  & {0.843}\tiny{$\pm$0.011} & {0.934}\tiny{$\pm$0.001}
  & {0.716}\tiny{$\pm$0.006} & {0.792}\tiny{$\pm$0.004}
  & {0.546}\tiny{$\pm$0.022} & {0.682}\tiny{$\pm$0.005}
  & {0.412}\tiny{$\pm$0.021} & {0.555}\tiny{$\pm$0.021}
  & \textbf{{0.322}\tiny{$\pm$0.012}} & {0.444}\tiny{$\pm$0.009} \\
  \cmidrule{1-11}
  {+ Mixup} 
  & {0.807}\tiny{$\pm$0.016} & {0.921}\tiny{$\pm$0.002}
  & {0.733}\tiny{$\pm$0.006} & {0.811}\tiny{$\pm$0.005}
  & {0.569}\tiny{$\pm$0.022} & {0.718}\tiny{$\pm$0.010}
  & {0.413}\tiny{$\pm$0.014} & {0.587}\tiny{$\pm$0.016}
  & {0.281}\tiny{$\pm$0.009} & {0.446}\tiny{$\pm$0.011} \\
  {+ Manifold Mixup} 
  & {0.832}\tiny{$\pm$0.012} & {0.942}\tiny{$\pm$0.002}
  & {0.739}\tiny{$\pm$0.004} & {0.818}\tiny{$\pm$0.003}
  & {0.563}\tiny{$\pm$0.012} & {0.688}\tiny{$\pm$0.011}
  & {0.419}\tiny{$\pm$0.013} & {0.603}\tiny{$\pm$0.011}
  & {0.292}\tiny{$\pm$0.011} & {0.455}\tiny{$\pm$0.009} \\
  {+ CutMix} 
  & {0.798}\tiny{$\pm$0.014} & {0.933}\tiny{$\pm$0.004}
  & {0.743}\tiny{$\pm$0.011} & {0.830}\tiny{$\pm$0.003}
  & {0.475}\tiny{$\pm$0.009} & {0.602}\tiny{$\pm$0.016}
  & {0.347}\tiny{$\pm$0.007} & {0.505}\tiny{$\pm$0.006}
  & {0.289}\tiny{$\pm$0.009} & {0.414}\tiny{$\pm$0.010} \\
  \cmidrule{1-11}
  {+ Remix} 
  & {0.848}\tiny{$\pm$0.014} & {0.944}\tiny{$\pm$0.004}
  & {0.751}\tiny{$\pm$0.006} & {0.825}\tiny{$\pm$0.005}
  & {0.565}\tiny{$\pm$0.017} & {0.711}\tiny{$\pm$0.011}
  & {0.423}\tiny{$\pm$0.010} & {0.607}\tiny{$\pm$0.012}
  & {0.299}\tiny{$\pm$0.009} & {0.460}\tiny{$\pm$0.009} \\
  {+ Balanced-MixUp} 
  & {0.823}\tiny{$\pm$0.013} & {0.937}\tiny{$\pm$0.002}
  & {0.724}\tiny{$\pm$0.005} & {0.818}\tiny{$\pm$0.005}
  & {0.560}\tiny{$\pm$0.028} & {0.696}\tiny{$\pm$0.006}
  & {0.426}\tiny{$\pm$0.016} & {0.597}\tiny{$\pm$0.016}
  & {0.286}\tiny{$\pm$0.010} & {0.449}\tiny{$\pm$0.012} \\
  \cmidrule{1-11}
  {+ RandAugment} 
  & {0.689}\tiny{$\pm$0.016} & {0.833}\tiny{$\pm$0.003}
  & {0.600}\tiny{$\pm$0.008} & {0.696}\tiny{$\pm$0.005}
  & {0.584}\tiny{$\pm$0.008} & {0.703}\tiny{$\pm$0.010}
  & {0.415}\tiny{$\pm$0.013} & {0.560}\tiny{$\pm$0.009}
  & {0.304}\tiny{$\pm$0.007} & {0.450}\tiny{$\pm$0.010} \\
  \cmidrule{1-11}
  {\bf + GradMix (ours)} 
  & \textbf{{0.852}\tiny{$\pm$0.010}} & \textbf{{0.950}\tiny{$\pm$0.002}}
  & \textbf{{0.764}\tiny{$\pm$0.009}} & \textbf{{0.837}\tiny{$\pm$0.007}}
  & \textbf{{0.588}\tiny{$\pm$0.021}} & \textbf{{0.739}\tiny{$\pm$0.009}}
  & \textbf{{0.437}\tiny{$\pm$0.009}} & \textbf{{0.620}\tiny{$\pm$0.011}}
  & {0.299}\tiny{$\pm$0.008} & \textbf{{0.470}\tiny{$\pm$0.008}} \\
  \bottomrule
  \end{tabular}
  \label{tbl:buffer_size_main}
  \vspace{-0.1cm}
\end{table*}

\subsection{Selective Mixup Analysis}
\label{analysis}

We next analyze how \method{} selectively mixes different pairs of classes. We provide selective mixup results on the last task of the MNIST, FMNIST, CIFAR-10, and CIFAR-100 datasets as shown in Fig.~\ref{fig:heatmap_exp_main}, and the results for the Tiny ImageNet dataset are shown in Sec.~\ref{appendix:analysis}. Since the selective mixup results for other tasks are similar, we show the results for the last task without loss of generality. As the selective mixup criterion changes each epoch based on gradients, we sum the total number of mixup operations for each class pair combination, denoted as $(C_i, C_j)$, during training and present the results as a heatmap.


Since \method{} performs selective mixup, each class has preferred classes for mixup instead of mixing with random classes as shown in Fig.~\ref{fig:heatmap_exp_main}. For example, in MNIST, previous classes (0th to 7th) are more frequently mixed with the same or current classes than with different previous classes. For CIFAR-10, the current classes (8th and 9th), representing `ship' and `truck', are more likely to be mixed with the first class, `automobile'. The selective mixup results for the other datasets also demonstrate similar trends to those observed in MNIST and CIFAR-10. This result suggests that augmenting limited previous classes by mixing them with sufficient and similar current classes helps retain knowledge of the previous classes more effectively.

\subsection{Ablation Study}


We perform an ablation study to verify the effectiveness of using selective mixup to minimize catastrophic forgetting in \method{} as shown in Table~\ref{tbl:ablation_study}. We consider two ablation scenarios when randomly selected sample pairs worsen catastrophic forgetting in Algorithm~\ref{alg:gradmix} (Steps 17--19): (1) using the original samples without mixup, and (2) retaining random mixup. As a result, both scenarios cannot prevent the occurrence of catastrophic forgetting and lead to lower accuracy results. In comparison, \method{} matches appropriate mixing sample pairs that can minimize catastrophic forgetting and achieves the best accuracy performance.

\subsection{Varying the Buffer Size}
\label{buffer_size}


We compare the performance of \method{} with the baselines when varying the number of buffer samples per class, denoted as $N$, from 32 to 16 and 64 as shown in Table~\ref{tbl:buffer_size_main}. We observe that \method{} shows the best performance for all datasets regardless of the buffer size. In addition, as we store more buffer samples and use them for training, the average accuracy of all the baselines and \method{} increases. For the results on the Tiny ImageNet dataset with a buffer size of 16 samples per class, we believe this buffer size is too small to train the model sufficiently, as none of the data augmentation methods improve accuracy compared to ER.


\section{Conclusion}


We proposed a novel data augmentation method that is robust against catastrophic forgetting in class-incremental learning. We first theoretically and empirically identified that simply combining vanilla Mixup with experience replay methods is vulnerable to catastrophic forgetting. To address the problems of limited buffer data and catastrophic forgetting, our proposed \method{} uses gradient-based selective mixup to enhance the impact of buffer data while mitigating catastrophic forgetting. In the experiments, \method{} shows better accuracy results on various datasets compared to other data augmentation baselines by effectively minimizing catastrophic forgetting.

\begin{acks}
This work was supported by the Institute of Information \& Communications Technology Planning \& Evaluation\,(IITP) grant funded by the Korea government\,(MSIT) (No.\ RS-2022-II220157, Robust, Fair, Extensible Data-Centric Continual Learning). This work was supported by the Institute of Information \& Communications Technology Planning \& Evaluation\,(IITP) grant funded by the Korea government\,(MSIT) (No.\ RS-2024-00444862, Non-invasive near-infrared based AI technology for the diagnosis and treatment of brain diseases).
\end{acks}

\bibliographystyle{ACM-Reference-Format}
\bibliography{gradmix-sigconf}


\begin{thebibliography}{85}


\ifx \showCODEN    \undefined \def \showCODEN     #1{\unskip}     \fi
\ifx \showDOI      \undefined \def \showDOI       #1{#1}\fi
\ifx \showISBNx    \undefined \def \showISBNx     #1{\unskip}     \fi
\ifx \showISBNxiii \undefined \def \showISBNxiii  #1{\unskip}     \fi
\ifx \showISSN     \undefined \def \showISSN      #1{\unskip}     \fi
\ifx \showLCCN     \undefined \def \showLCCN      #1{\unskip}     \fi
\ifx \shownote     \undefined \def \shownote      #1{#1}          \fi
\ifx \showarticletitle \undefined \def \showarticletitle #1{#1}   \fi
\ifx \showURL      \undefined \def \showURL       {\relax}        \fi
\providecommand\bibfield[2]{#2}
\providecommand\bibinfo[2]{#2}
\providecommand\natexlab[1]{#1}
\providecommand\showeprint[2][]{arXiv:#2}

\bibitem[Abraham and Robins(2005)]%
        {ABRAHAM200573}
\bibfield{author}{\bibinfo{person}{Wickliffe~C. Abraham} {and} \bibinfo{person}{Anthony Robins}.} \bibinfo{year}{2005}\natexlab{}.
\newblock \showarticletitle{Memory retention – the synaptic stability versus plasticity dilemma}.
\newblock \bibinfo{journal}{\emph{Trends in Neurosciences}} \bibinfo{volume}{28}, \bibinfo{number}{2} (\bibinfo{year}{2005}), \bibinfo{pages}{73--78}.
\newblock


\bibitem[Aljundi et~al\mbox{.}(2018)]%
        {DBLP:conf/eccv/AljundiBERT18}
\bibfield{author}{\bibinfo{person}{Rahaf Aljundi}, \bibinfo{person}{Francesca Babiloni}, \bibinfo{person}{Mohamed Elhoseiny}, \bibinfo{person}{Marcus Rohrbach}, {and} \bibinfo{person}{Tinne Tuytelaars}.} \bibinfo{year}{2018}\natexlab{}.
\newblock \showarticletitle{Memory Aware Synapses: Learning What (not) to Forget}. In \bibinfo{booktitle}{\emph{ECCV}}, Vol.~\bibinfo{volume}{11207}. \bibinfo{pages}{144--161}.
\newblock


\bibitem[Aljundi et~al\mbox{.}(2019a)]%
        {DBLP:conf/nips/AljundiBTCCLP19}
\bibfield{author}{\bibinfo{person}{Rahaf Aljundi}, \bibinfo{person}{Eugene Belilovsky}, \bibinfo{person}{Tinne Tuytelaars}, \bibinfo{person}{Laurent Charlin}, \bibinfo{person}{Massimo Caccia}, \bibinfo{person}{Min Lin}, {and} \bibinfo{person}{Lucas Page{-}Caccia}.} \bibinfo{year}{2019}\natexlab{a}.
\newblock \showarticletitle{Online Continual Learning with Maximal Interfered Retrieval}. In \bibinfo{booktitle}{\emph{NeurIPS}}. \bibinfo{pages}{11849--11860}.
\newblock


\bibitem[Aljundi et~al\mbox{.}(2019b)]%
        {DBLP:conf/nips/AljundiLGB19}
\bibfield{author}{\bibinfo{person}{Rahaf Aljundi}, \bibinfo{person}{Min Lin}, \bibinfo{person}{Baptiste Goujaud}, {and} \bibinfo{person}{Yoshua Bengio}.} \bibinfo{year}{2019}\natexlab{b}.
\newblock \showarticletitle{Gradient based sample selection for online continual learning}. In \bibinfo{booktitle}{\emph{NeurIPS}}. \bibinfo{pages}{11816--11825}.
\newblock


\bibitem[Ash et~al\mbox{.}(2020)]%
        {DBLP:conf/iclr/AshZK0A20}
\bibfield{author}{\bibinfo{person}{Jordan~T. Ash}, \bibinfo{person}{Chicheng Zhang}, \bibinfo{person}{Akshay Krishnamurthy}, \bibinfo{person}{John Langford}, {and} \bibinfo{person}{Alekh Agarwal}.} \bibinfo{year}{2020}\natexlab{}.
\newblock \showarticletitle{Deep Batch Active Learning by Diverse, Uncertain Gradient Lower Bounds}. In \bibinfo{booktitle}{\emph{ICLR}}.
\newblock


\bibitem[Bang et~al\mbox{.}(2021)]%
        {DBLP:conf/cvpr/BangKY0C21}
\bibfield{author}{\bibinfo{person}{Jihwan Bang}, \bibinfo{person}{Heesu Kim}, \bibinfo{person}{Youngjoon Yoo}, \bibinfo{person}{Jung{-}Woo Ha}, {and} \bibinfo{person}{Jonghyun Choi}.} \bibinfo{year}{2021}\natexlab{}.
\newblock \showarticletitle{Rainbow Memory: Continual Learning With a Memory of Diverse Samples}. In \bibinfo{booktitle}{\emph{CVPR}}. \bibinfo{pages}{8218--8227}.
\newblock


\bibitem[Buzzega et~al\mbox{.}(2020)]%
        {DBLP:conf/nips/BuzzegaBPAC20}
\bibfield{author}{\bibinfo{person}{Pietro Buzzega}, \bibinfo{person}{Matteo Boschini}, \bibinfo{person}{Angelo Porrello}, \bibinfo{person}{Davide Abati}, {and} \bibinfo{person}{Simone Calderara}.} \bibinfo{year}{2020}\natexlab{}.
\newblock \showarticletitle{Dark Experience for General Continual Learning: a Strong, Simple Baseline}. In \bibinfo{booktitle}{\emph{NeurIPS}}.
\newblock


\bibitem[Cai et~al\mbox{.}(2022)]%
        {DBLP:conf/sigir/CaiZDD0TZ22}
\bibfield{author}{\bibinfo{person}{Guohao Cai}, \bibinfo{person}{Jieming Zhu}, \bibinfo{person}{Quanyu Dai}, \bibinfo{person}{Zhenhua Dong}, \bibinfo{person}{Xiuqiang He}, \bibinfo{person}{Ruiming Tang}, {and} \bibinfo{person}{Rui Zhang}.} \bibinfo{year}{2022}\natexlab{}.
\newblock \showarticletitle{ReLoop: {A} Self-Correction Continual Learning Loop for Recommender Systems}. In \bibinfo{booktitle}{\emph{SIGIR}}. \bibinfo{pages}{2692--2697}.
\newblock


\bibitem[Cai and Zhang(2019)]%
        {0ceee346-049e-3934-a519-92e72b194bcf}
\bibfield{author}{\bibinfo{person}{T.~Tony Cai} {and} \bibinfo{person}{Linjun Zhang}.} \bibinfo{year}{2019}\natexlab{}.
\newblock \showarticletitle{High dimensional linear discriminant analysis: optimality, adaptive algorithm and missing data}.
\newblock \bibinfo{journal}{\emph{Journal of the Royal Statistical Society. Series B (Statistical Methodology)}} \bibinfo{volume}{81}, \bibinfo{number}{4} (\bibinfo{year}{2019}), \bibinfo{pages}{675--705}.
\newblock


\bibitem[Cai and Zhang(2021)]%
        {10.1214/20-AOS2012}
\bibfield{author}{\bibinfo{person}{T.~Tony Cai} {and} \bibinfo{person}{Linjun Zhang}.} \bibinfo{year}{2021}\natexlab{}.
\newblock \showarticletitle{{A convex optimization approach to high-dimensional sparse quadratic discriminant analysis}}.
\newblock \bibinfo{journal}{\emph{The Annals of Statistics}} \bibinfo{volume}{49}, \bibinfo{number}{3} (\bibinfo{year}{2021}), \bibinfo{pages}{1537 -- 1568}.
\newblock


\bibitem[Chaudhry et~al\mbox{.}(2018)]%
        {DBLP:conf/eccv/ChaudhryDAT18}
\bibfield{author}{\bibinfo{person}{Arslan Chaudhry}, \bibinfo{person}{Puneet~Kumar Dokania}, \bibinfo{person}{Thalaiyasingam Ajanthan}, {and} \bibinfo{person}{Philip H.~S. Torr}.} \bibinfo{year}{2018}\natexlab{}.
\newblock \showarticletitle{Riemannian Walk for Incremental Learning: Understanding Forgetting and Intransigence}. In \bibinfo{booktitle}{\emph{ECCV}}, Vol.~\bibinfo{volume}{11215}. \bibinfo{pages}{556--572}.
\newblock


\bibitem[Chaudhry et~al\mbox{.}(2019a)]%
        {DBLP:conf/iclr/ChaudhryRRE19}
\bibfield{author}{\bibinfo{person}{Arslan Chaudhry}, \bibinfo{person}{Marc'Aurelio Ranzato}, \bibinfo{person}{Marcus Rohrbach}, {and} \bibinfo{person}{Mohamed Elhoseiny}.} \bibinfo{year}{2019}\natexlab{a}.
\newblock \showarticletitle{Efficient Lifelong Learning with {A-GEM}}. In \bibinfo{booktitle}{\emph{ICLR}}.
\newblock


\bibitem[Chaudhry et~al\mbox{.}(2019b)]%
        {DBLP:journals/corr/abs-1902-10486}
\bibfield{author}{\bibinfo{person}{Arslan Chaudhry}, \bibinfo{person}{Marcus Rohrbach}, \bibinfo{person}{Mohamed Elhoseiny}, \bibinfo{person}{Thalaiyasingam Ajanthan}, \bibinfo{person}{Puneet~Kumar Dokania}, \bibinfo{person}{Philip H.~S. Torr}, {and} \bibinfo{person}{Marc'Aurelio Ranzato}.} \bibinfo{year}{2019}\natexlab{b}.
\newblock \showarticletitle{Continual Learning with Tiny Episodic Memories}.
\newblock \bibinfo{journal}{\emph{CoRR}}  \bibinfo{volume}{abs/1902.10486} (\bibinfo{year}{2019}).
\newblock


\bibitem[Chen and Chang(2023)]%
        {DBLP:conf/iccv/ChenC23}
\bibfield{author}{\bibinfo{person}{Xiuwei Chen} {and} \bibinfo{person}{Xiaobin Chang}.} \bibinfo{year}{2023}\natexlab{}.
\newblock \showarticletitle{Dynamic Residual Classifier for Class Incremental Learning}. In \bibinfo{booktitle}{\emph{ICCV}}. \bibinfo{pages}{18697--18706}.
\newblock


\bibitem[Chen and Liu(2016)]%
        {DBLP:series/synthesis/2016Chen}
\bibfield{author}{\bibinfo{person}{Zhiyuan Chen} {and} \bibinfo{person}{Bing Liu}.} \bibinfo{year}{2016}\natexlab{}.
\newblock \bibinfo{booktitle}{\emph{Lifelong Machine Learning}}.
\newblock


\bibitem[Chou et~al\mbox{.}(2020)]%
        {DBLP:conf/eccv/ChouCPWJ20}
\bibfield{author}{\bibinfo{person}{Hsin{-}Ping Chou}, \bibinfo{person}{Shih{-}Chieh Chang}, \bibinfo{person}{Jia{-}Yu Pan}, \bibinfo{person}{Wei Wei}, {and} \bibinfo{person}{Da{-}Cheng Juan}.} \bibinfo{year}{2020}\natexlab{}.
\newblock \showarticletitle{Remix: Rebalanced Mixup}. In \bibinfo{booktitle}{\emph{ECCV Workshops}}, Vol.~\bibinfo{volume}{12540}. \bibinfo{pages}{95--110}.
\newblock


\bibitem[Cubuk et~al\mbox{.}(2019)]%
        {DBLP:conf/cvpr/CubukZMVL19}
\bibfield{author}{\bibinfo{person}{Ekin~D. Cubuk}, \bibinfo{person}{Barret Zoph}, \bibinfo{person}{Dandelion Man{\'{e}}}, \bibinfo{person}{Vijay Vasudevan}, {and} \bibinfo{person}{Quoc~V. Le}.} \bibinfo{year}{2019}\natexlab{}.
\newblock \showarticletitle{AutoAugment: Learning Augmentation Strategies From Data}. In \bibinfo{booktitle}{\emph{CVPR}}. \bibinfo{pages}{113--123}.
\newblock


\bibitem[Cubuk et~al\mbox{.}(2020)]%
        {DBLP:conf/nips/CubukZS020}
\bibfield{author}{\bibinfo{person}{Ekin~Dogus Cubuk}, \bibinfo{person}{Barret Zoph}, \bibinfo{person}{Jonathon Shlens}, {and} \bibinfo{person}{Quoc Le}.} \bibinfo{year}{2020}\natexlab{}.
\newblock \showarticletitle{RandAugment: Practical Automated Data Augmentation with a Reduced Search Space}. In \bibinfo{booktitle}{\emph{NeurIPS}}.
\newblock


\bibitem[Deng et~al\mbox{.}(2009)]%
        {DBLP:conf/cvpr/DengDSLL009}
\bibfield{author}{\bibinfo{person}{Jia Deng}, \bibinfo{person}{Wei Dong}, \bibinfo{person}{Richard Socher}, \bibinfo{person}{Li{-}Jia Li}, \bibinfo{person}{Kai Li}, {and} \bibinfo{person}{Li Fei{-}Fei}.} \bibinfo{year}{2009}\natexlab{}.
\newblock \showarticletitle{ImageNet: {A} large-scale hierarchical image database}. In \bibinfo{booktitle}{\emph{CVPR}}. \bibinfo{pages}{248--255}.
\newblock


\bibitem[Dosovitskiy et~al\mbox{.}(2021)]%
        {DBLP:conf/iclr/DosovitskiyB0WZ21}
\bibfield{author}{\bibinfo{person}{Alexey Dosovitskiy}, \bibinfo{person}{Lucas Beyer}, \bibinfo{person}{Alexander Kolesnikov}, \bibinfo{person}{Dirk Weissenborn}, \bibinfo{person}{Xiaohua Zhai}, \bibinfo{person}{Thomas Unterthiner}, \bibinfo{person}{Mostafa Dehghani}, \bibinfo{person}{Matthias Minderer}, \bibinfo{person}{Georg Heigold}, \bibinfo{person}{Sylvain Gelly}, \bibinfo{person}{Jakob Uszkoreit}, {and} \bibinfo{person}{Neil Houlsby}.} \bibinfo{year}{2021}\natexlab{}.
\newblock \showarticletitle{An Image is Worth 16x16 Words: Transformers for Image Recognition at Scale}. In \bibinfo{booktitle}{\emph{ICLR}}.
\newblock


\bibitem[Douillard et~al\mbox{.}(2022)]%
        {DBLP:conf/cvpr/DouillardRCC22}
\bibfield{author}{\bibinfo{person}{Arthur Douillard}, \bibinfo{person}{Alexandre Ram{\'{e}}}, \bibinfo{person}{Guillaume Couairon}, {and} \bibinfo{person}{Matthieu Cord}.} \bibinfo{year}{2022}\natexlab{}.
\newblock \showarticletitle{DyTox: Transformers for Continual Learning with DYnamic TOken eXpansion}. In \bibinfo{booktitle}{\emph{CVPR}}. \bibinfo{pages}{9275--9285}.
\newblock


\bibitem[Farquhar and Gal(2018)]%
        {DBLP:journals/corr/abs-1805-09733}
\bibfield{author}{\bibinfo{person}{Sebastian Farquhar} {and} \bibinfo{person}{Yarin Gal}.} \bibinfo{year}{2018}\natexlab{}.
\newblock \showarticletitle{Towards Robust Evaluations of Continual Learning}.
\newblock \bibinfo{journal}{\emph{CoRR}}  \bibinfo{volume}{abs/1805.09733} (\bibinfo{year}{2018}).
\newblock


\bibitem[Galdran et~al\mbox{.}(2021)]%
        {DBLP:conf/miccai/GaldranCB21}
\bibfield{author}{\bibinfo{person}{Adrian Galdran}, \bibinfo{person}{Gustavo Carneiro}, {and} \bibinfo{person}{Miguel {\'{A}}ngel~Gonz{\'{a}}lez Ballester}.} \bibinfo{year}{2021}\natexlab{}.
\newblock \showarticletitle{Balanced-MixUp for Highly Imbalanced Medical Image Classification}. In \bibinfo{booktitle}{\emph{MICCAI}}, Vol.~\bibinfo{volume}{12905}. \bibinfo{pages}{323--333}.
\newblock


\bibitem[Gao and Liu(2023)]%
        {DBLP:conf/icml/GaoL23a}
\bibfield{author}{\bibinfo{person}{Rui Gao} {and} \bibinfo{person}{Weiwei Liu}.} \bibinfo{year}{2023}\natexlab{}.
\newblock \showarticletitle{{DDGR:} Continual Learning with Deep Diffusion-based Generative Replay}. In \bibinfo{booktitle}{\emph{ICML}}, Vol.~\bibinfo{volume}{202}. \bibinfo{publisher}{{PMLR}}, \bibinfo{pages}{10744--10763}.
\newblock


\bibitem[Gionis et~al\mbox{.}(1999)]%
        {DBLP:conf/vldb/GionisIM99}
\bibfield{author}{\bibinfo{person}{Aristides Gionis}, \bibinfo{person}{Piotr Indyk}, {and} \bibinfo{person}{Rajeev Motwani}.} \bibinfo{year}{1999}\natexlab{}.
\newblock \showarticletitle{Similarity Search in High Dimensions via Hashing}. In \bibinfo{booktitle}{\emph{VLDB}}. \bibinfo{pages}{518--529}.
\newblock


\bibitem[Harris et~al\mbox{.}(2021)]%
        {harris2021fmixenhancingmixedsample}
\bibfield{author}{\bibinfo{person}{Ethan Harris}, \bibinfo{person}{Antonia Marcu}, \bibinfo{person}{Matthew Painter}, \bibinfo{person}{Mahesan Niranjan}, \bibinfo{person}{Adam Prügel-Bennett}, {and} \bibinfo{person}{Jonathon Hare}.} \bibinfo{year}{2021}\natexlab{}.
\newblock \bibinfo{title}{FMix: Enhancing Mixed Sample Data Augmentation}.
\newblock
\newblock
\showeprint[arxiv]{2002.12047}


\bibitem[He et~al\mbox{.}(2016a)]%
        {DBLP:conf/cvpr/HeZRS16}
\bibfield{author}{\bibinfo{person}{Kaiming He}, \bibinfo{person}{Xiangyu Zhang}, \bibinfo{person}{Shaoqing Ren}, {and} \bibinfo{person}{Jian Sun}.} \bibinfo{year}{2016}\natexlab{a}.
\newblock \showarticletitle{Deep Residual Learning for Image Recognition}. In \bibinfo{booktitle}{\emph{CVPR}}. \bibinfo{pages}{770--778}.
\newblock


\bibitem[He et~al\mbox{.}(2016b)]%
        {DBLP:conf/eccv/HeZRS16}
\bibfield{author}{\bibinfo{person}{Kaiming He}, \bibinfo{person}{Xiangyu Zhang}, \bibinfo{person}{Shaoqing Ren}, {and} \bibinfo{person}{Jian Sun}.} \bibinfo{year}{2016}\natexlab{b}.
\newblock \showarticletitle{Identity Mappings in Deep Residual Networks}. In \bibinfo{booktitle}{\emph{ECCV}}, Vol.~\bibinfo{volume}{9908}. \bibinfo{pages}{630--645}.
\newblock


\bibitem[Hou et~al\mbox{.}(2019)]%
        {DBLP:conf/cvpr/HouPLWL19}
\bibfield{author}{\bibinfo{person}{Saihui Hou}, \bibinfo{person}{Xinyu Pan}, \bibinfo{person}{Chen~Change Loy}, \bibinfo{person}{Zilei Wang}, {and} \bibinfo{person}{Dahua Lin}.} \bibinfo{year}{2019}\natexlab{}.
\newblock \showarticletitle{Learning a Unified Classifier Incrementally via Rebalancing}. In \bibinfo{booktitle}{\emph{CVPR}}. \bibinfo{pages}{831--839}.
\newblock


\bibitem[Indyk and Motwani(1998)]%
        {DBLP:conf/stoc/IndykM98}
\bibfield{author}{\bibinfo{person}{Piotr Indyk} {and} \bibinfo{person}{Rajeev Motwani}.} \bibinfo{year}{1998}\natexlab{}.
\newblock \showarticletitle{Approximate Nearest Neighbors: Towards Removing the Curse of Dimensionality}. In \bibinfo{booktitle}{\emph{STOC}}. \bibinfo{pages}{604--613}.
\newblock


\bibitem[Johnson and Wichern(1988)]%
        {10.5555/59551}
\bibfield{author}{\bibinfo{person}{R.~A. Johnson} {and} \bibinfo{person}{D.~W. Wichern}.} \bibinfo{year}{1988}\natexlab{}.
\newblock \bibinfo{booktitle}{\emph{Applied multivariate statistical analysis}}.
\newblock \bibinfo{publisher}{Prentice-Hall, Inc.}
\newblock


\bibitem[Katharopoulos and Fleuret(2018)]%
        {DBLP:conf/icml/KatharopoulosF18}
\bibfield{author}{\bibinfo{person}{Angelos Katharopoulos} {and} \bibinfo{person}{Fran{\c{c}}ois Fleuret}.} \bibinfo{year}{2018}\natexlab{}.
\newblock \showarticletitle{Not All Samples Are Created Equal: Deep Learning with Importance Sampling}. In \bibinfo{booktitle}{\emph{ICML}}, Vol.~\bibinfo{volume}{80}. \bibinfo{publisher}{{PMLR}}, \bibinfo{pages}{2530--2539}.
\newblock


\bibitem[Killamsetty et~al\mbox{.}(2021b)]%
        {DBLP:conf/icml/KillamsettySRDI21}
\bibfield{author}{\bibinfo{person}{KrishnaTeja Killamsetty}, \bibinfo{person}{Durga Sivasubramanian}, \bibinfo{person}{Ganesh Ramakrishnan}, \bibinfo{person}{Abir De}, {and} \bibinfo{person}{Rishabh~K. Iyer}.} \bibinfo{year}{2021}\natexlab{b}.
\newblock \showarticletitle{{GRAD-MATCH:} Gradient Matching based Data Subset Selection for Efficient Deep Model Training}. In \bibinfo{booktitle}{\emph{ICML}}, Vol.~\bibinfo{volume}{139}. \bibinfo{publisher}{{PMLR}}, \bibinfo{pages}{5464--5474}.
\newblock


\bibitem[Killamsetty et~al\mbox{.}(2021a)]%
        {DBLP:conf/aaai/KillamsettySRI21}
\bibfield{author}{\bibinfo{person}{KrishnaTeja Killamsetty}, \bibinfo{person}{Durga Sivasubramanian}, \bibinfo{person}{Ganesh Ramakrishnan}, {and} \bibinfo{person}{Rishabh~K. Iyer}.} \bibinfo{year}{2021}\natexlab{a}.
\newblock \showarticletitle{{GLISTER:} Generalization based Data Subset Selection for Efficient and Robust Learning}. In \bibinfo{booktitle}{\emph{AAAI}}. \bibinfo{pages}{8110--8118}.
\newblock


\bibitem[Kim and Han(2023)]%
        {DBLP:conf/cvpr/KimH23}
\bibfield{author}{\bibinfo{person}{Dongwan Kim} {and} \bibinfo{person}{Bohyung Han}.} \bibinfo{year}{2023}\natexlab{}.
\newblock \showarticletitle{On the Stability-Plasticity Dilemma of Class-Incremental Learning}. In \bibinfo{booktitle}{\emph{CVPR}}. \bibinfo{pages}{20196--20204}.
\newblock


\bibitem[Kim et~al\mbox{.}(2022)]%
        {DBLP:conf/nips/KimXKK022}
\bibfield{author}{\bibinfo{person}{Gyuhak Kim}, \bibinfo{person}{Changnan Xiao}, \bibinfo{person}{Tatsuya Konishi}, \bibinfo{person}{Zixuan Ke}, {and} \bibinfo{person}{Bing Liu}.} \bibinfo{year}{2022}\natexlab{}.
\newblock \showarticletitle{A Theoretical Study on Solving Continual Learning}. In \bibinfo{booktitle}{\emph{NeurIPS}}.
\newblock


\bibitem[Kim et~al\mbox{.}(2024a)]%
        {DBLP:conf/cvpr/KimCKTB24}
\bibfield{author}{\bibinfo{person}{Junsu Kim}, \bibinfo{person}{Hoseong Cho}, \bibinfo{person}{Jihyeon Kim}, \bibinfo{person}{Yihalem~Yimolal Tiruneh}, {and} \bibinfo{person}{Seungryul Baek}.} \bibinfo{year}{2024}\natexlab{a}.
\newblock \showarticletitle{{SDDGR:} Stable Diffusion-Based Deep Generative Replay for Class Incremental Object Detection}. In \bibinfo{booktitle}{\emph{CVPR}}. \bibinfo{publisher}{{IEEE}}, \bibinfo{pages}{28772--28781}.
\newblock


\bibitem[Kim et~al\mbox{.}(2024b)]%
        {DBLP:conf/aaai/KimHW24}
\bibfield{author}{\bibinfo{person}{Minsu Kim}, \bibinfo{person}{Seonghyeon Hwang}, {and} \bibinfo{person}{Steven~Euijong Whang}.} \bibinfo{year}{2024}\natexlab{b}.
\newblock \showarticletitle{Quilt: Robust Data Segment Selection against Concept Drifts}. In \bibinfo{booktitle}{\emph{AAAI}}. \bibinfo{pages}{21249--21257}.
\newblock


\bibitem[Kim et~al\mbox{.}(2024c)]%
        {DBLP:conf/aaai/KimPH24}
\bibfield{author}{\bibinfo{person}{Taehoon Kim}, \bibinfo{person}{Jaeyoo Park}, {and} \bibinfo{person}{Bohyung Han}.} \bibinfo{year}{2024}\natexlab{c}.
\newblock \showarticletitle{Cross-Class Feature Augmentation for Class Incremental Learning}. In \bibinfo{booktitle}{\emph{AAAI}}. \bibinfo{pages}{13168--13176}.
\newblock


\bibitem[Kirkpatrick et~al\mbox{.}(2016)]%
        {DBLP:journals/corr/KirkpatrickPRVD16}
\bibfield{author}{\bibinfo{person}{James Kirkpatrick}, \bibinfo{person}{Razvan Pascanu}, \bibinfo{person}{Neil~C. Rabinowitz}, \bibinfo{person}{Joel Veness}, \bibinfo{person}{Guillaume Desjardins}, \bibinfo{person}{Andrei~A. Rusu}, \bibinfo{person}{Kieran Milan}, \bibinfo{person}{John Quan}, \bibinfo{person}{Tiago Ramalho}, \bibinfo{person}{Agnieszka Grabska{-}Barwinska}, \bibinfo{person}{Demis Hassabis}, \bibinfo{person}{Claudia Clopath}, \bibinfo{person}{Dharshan Kumaran}, {and} \bibinfo{person}{Raia Hadsell}.} \bibinfo{year}{2016}\natexlab{}.
\newblock \showarticletitle{Overcoming catastrophic forgetting in neural networks}.
\newblock \bibinfo{journal}{\emph{CoRR}}  \bibinfo{volume}{abs/1612.00796} (\bibinfo{year}{2016}).
\newblock


\bibitem[Kolouri et~al\mbox{.}(2020)]%
        {DBLP:conf/iclr/KolouriKSP20}
\bibfield{author}{\bibinfo{person}{Soheil Kolouri}, \bibinfo{person}{Nicholas~A. Ketz}, \bibinfo{person}{Andrea Soltoggio}, {and} \bibinfo{person}{Praveen~K. Pilly}.} \bibinfo{year}{2020}\natexlab{}.
\newblock \showarticletitle{Sliced Cramer Synaptic Consolidation for Preserving Deeply Learned Representations}. In \bibinfo{booktitle}{\emph{ICLR}}.
\newblock


\bibitem[Krizhevsky(2009)]%
        {Krizhevsky2009LearningML}
\bibfield{author}{\bibinfo{person}{Alex Krizhevsky}.} \bibinfo{year}{2009}\natexlab{}.
\newblock \showarticletitle{Learning Multiple Layers of Features from Tiny Images}.
\newblock


\bibitem[Kumari et~al\mbox{.}(2022)]%
        {DBLP:conf/nips/KumariW0B22}
\bibfield{author}{\bibinfo{person}{Lilly Kumari}, \bibinfo{person}{Shengjie Wang}, \bibinfo{person}{Tianyi Zhou}, {and} \bibinfo{person}{Jeff~A. Bilmes}.} \bibinfo{year}{2022}\natexlab{}.
\newblock \showarticletitle{Retrospective Adversarial Replay for Continual Learning}. In \bibinfo{booktitle}{\emph{NeurIPS}}.
\newblock


\bibitem[Lange et~al\mbox{.}(2022)]%
        {DBLP:journals/pami/LangeAMPJLST22}
\bibfield{author}{\bibinfo{person}{Matthias~De Lange}, \bibinfo{person}{Rahaf Aljundi}, \bibinfo{person}{Marc Masana}, \bibinfo{person}{Sarah Parisot}, \bibinfo{person}{Xu Jia}, \bibinfo{person}{Ales Leonardis}, \bibinfo{person}{Gregory~G. Slabaugh}, {and} \bibinfo{person}{Tinne Tuytelaars}.} \bibinfo{year}{2022}\natexlab{}.
\newblock \showarticletitle{A Continual Learning Survey: Defying Forgetting in Classification Tasks}.
\newblock \bibinfo{journal}{\emph{{IEEE} Trans. Pattern Anal. Mach. Intell.}} \bibinfo{volume}{44}, \bibinfo{number}{7} (\bibinfo{year}{2022}), \bibinfo{pages}{3366--3385}.
\newblock


\bibitem[LeCun et~al\mbox{.}(1998)]%
        {DBLP:journals/pieee/LeCunBBH98}
\bibfield{author}{\bibinfo{person}{Yann LeCun}, \bibinfo{person}{L{\'{e}}on Bottou}, \bibinfo{person}{Yoshua Bengio}, {and} \bibinfo{person}{Patrick Haffner}.} \bibinfo{year}{1998}\natexlab{}.
\newblock \showarticletitle{Gradient-based learning applied to document recognition}.
\newblock \bibinfo{journal}{\emph{Proc. {IEEE}}} \bibinfo{volume}{86}, \bibinfo{number}{11} (\bibinfo{year}{1998}), \bibinfo{pages}{2278--2324}.
\newblock


\bibitem[Lesort et~al\mbox{.}(2020)]%
        {DBLP:journals/inffus/LesortLSMFR20}
\bibfield{author}{\bibinfo{person}{Timoth{\'{e}}e Lesort}, \bibinfo{person}{Vincenzo Lomonaco}, \bibinfo{person}{Andrei Stoian}, \bibinfo{person}{Davide Maltoni}, \bibinfo{person}{David Filliat}, {and} \bibinfo{person}{Natalia~D{\'{\i}}az Rodr{\'{\i}}guez}.} \bibinfo{year}{2020}\natexlab{}.
\newblock \showarticletitle{Continual learning for robotics: Definition, framework, learning strategies, opportunities and challenges}.
\newblock \bibinfo{journal}{\emph{Inf. Fusion}}  \bibinfo{volume}{58} (\bibinfo{year}{2020}), \bibinfo{pages}{52--68}.
\newblock


\bibitem[Li and Hoiem(2016)]%
        {DBLP:conf/eccv/LiH16}
\bibfield{author}{\bibinfo{person}{Zhizhong Li} {and} \bibinfo{person}{Derek Hoiem}.} \bibinfo{year}{2016}\natexlab{}.
\newblock \showarticletitle{Learning Without Forgetting}. In \bibinfo{booktitle}{\emph{ECCV}}, Vol.~\bibinfo{volume}{9908}. \bibinfo{pages}{614--629}.
\newblock


\bibitem[Lopez{-}Paz and Ranzato(2017)]%
        {DBLP:conf/nips/Lopez-PazR17}
\bibfield{author}{\bibinfo{person}{David Lopez{-}Paz} {and} \bibinfo{person}{Marc'Aurelio Ranzato}.} \bibinfo{year}{2017}\natexlab{}.
\newblock \showarticletitle{Gradient Episodic Memory for Continual Learning}. In \bibinfo{booktitle}{\emph{NIPS}}. \bibinfo{pages}{6467--6476}.
\newblock


\bibitem[Mai et~al\mbox{.}(2021)]%
        {DBLP:conf/cvpr/MaiLKS21}
\bibfield{author}{\bibinfo{person}{Zheda Mai}, \bibinfo{person}{Ruiwen Li}, \bibinfo{person}{Hyunwoo Kim}, {and} \bibinfo{person}{Scott Sanner}.} \bibinfo{year}{2021}\natexlab{}.
\newblock \showarticletitle{Supervised Contrastive Replay: Revisiting the Nearest Class Mean Classifier in Online Class-Incremental Continual Learning}. In \bibinfo{booktitle}{\emph{CVPR Workshops}}. \bibinfo{pages}{3589--3599}.
\newblock


\bibitem[McCloskey and Cohen(1989)]%
        {MCCLOSKEY1989109}
\bibfield{author}{\bibinfo{person}{Michael McCloskey} {and} \bibinfo{person}{Neal~J. Cohen}.} \bibinfo{year}{1989}\natexlab{}.
\newblock \showarticletitle{Catastrophic Interference in Connectionist Networks: The Sequential Learning Problem}.
\newblock Vol.~\bibinfo{volume}{24}. \bibinfo{pages}{109--165}.
\newblock


\bibitem[Mermillod et~al\mbox{.}(2013)]%
        {10.3389/fpsyg.2013.00504}
\bibfield{author}{\bibinfo{person}{Martial Mermillod}, \bibinfo{person}{Aurélia Bugaiska}, {and} \bibinfo{person}{Patrick BONIN}.} \bibinfo{year}{2013}\natexlab{}.
\newblock \showarticletitle{The stability-plasticity dilemma: investigating the continuum from catastrophic forgetting to age-limited learning effects}.
\newblock \bibinfo{journal}{\emph{Frontiers in Psychology}}  \bibinfo{volume}{4} (\bibinfo{year}{2013}).
\newblock


\bibitem[Mi et~al\mbox{.}(2020)]%
        {DBLP:conf/cvpr/MiKLYF20}
\bibfield{author}{\bibinfo{person}{Fei Mi}, \bibinfo{person}{Lingjing Kong}, \bibinfo{person}{Tao Lin}, \bibinfo{person}{Kaicheng Yu}, {and} \bibinfo{person}{Boi Faltings}.} \bibinfo{year}{2020}\natexlab{}.
\newblock \showarticletitle{Generalized Class Incremental Learning}. In \bibinfo{booktitle}{\emph{CVPR Workshops}}. \bibinfo{pages}{970--974}.
\newblock


\bibitem[Mirzadeh et~al\mbox{.}(2020)]%
        {DBLP:conf/nips/MirzadehFPG20}
\bibfield{author}{\bibinfo{person}{Seyed{-}Iman Mirzadeh}, \bibinfo{person}{Mehrdad Farajtabar}, \bibinfo{person}{Razvan Pascanu}, {and} \bibinfo{person}{Hassan Ghasemzadeh}.} \bibinfo{year}{2020}\natexlab{}.
\newblock \showarticletitle{Understanding the Role of Training Regimes in Continual Learning}. In \bibinfo{booktitle}{\emph{NeurIPS}}.
\newblock


\bibitem[Mirzasoleiman et~al\mbox{.}(2020)]%
        {DBLP:conf/icml/MirzasoleimanBL20}
\bibfield{author}{\bibinfo{person}{Baharan Mirzasoleiman}, \bibinfo{person}{Jeff~A. Bilmes}, {and} \bibinfo{person}{Jure Leskovec}.} \bibinfo{year}{2020}\natexlab{}.
\newblock \showarticletitle{Coresets for Data-efficient Training of Machine Learning Models}. In \bibinfo{booktitle}{\emph{ICML}}, Vol.~\bibinfo{volume}{119}. \bibinfo{publisher}{{PMLR}}, \bibinfo{pages}{6950--6960}.
\newblock


\bibitem[M{\"{u}}ller and Hutter(2021)]%
        {DBLP:conf/iccv/MullerH21}
\bibfield{author}{\bibinfo{person}{Samuel~G. M{\"{u}}ller} {and} \bibinfo{person}{Frank Hutter}.} \bibinfo{year}{2021}\natexlab{}.
\newblock \showarticletitle{TrivialAugment: Tuning-free Yet State-of-the-Art Data Augmentation}. In \bibinfo{booktitle}{\emph{ICCV}}. \bibinfo{pages}{754--762}.
\newblock


\bibitem[Qiang et~al\mbox{.}(2023)]%
        {DBLP:conf/aaai/QiangH0L0Z23}
\bibfield{author}{\bibinfo{person}{Sunyuan Qiang}, \bibinfo{person}{Jiayi Hou}, \bibinfo{person}{Jun Wan}, \bibinfo{person}{Yanyan Liang}, \bibinfo{person}{Zhen Lei}, {and} \bibinfo{person}{Du Zhang}.} \bibinfo{year}{2023}\natexlab{}.
\newblock \showarticletitle{Mixture Uniform Distribution Modeling and Asymmetric Mix Distillation for Class Incremental Learning}. In \bibinfo{booktitle}{\emph{AAAI}}. \bibinfo{pages}{9498--9506}.
\newblock


\bibitem[Rebuffi et~al\mbox{.}(2017)]%
        {DBLP:conf/cvpr/RebuffiKSL17}
\bibfield{author}{\bibinfo{person}{Sylvestre{-}Alvise Rebuffi}, \bibinfo{person}{Alexander Kolesnikov}, \bibinfo{person}{Georg Sperl}, {and} \bibinfo{person}{Christoph~H. Lampert}.} \bibinfo{year}{2017}\natexlab{}.
\newblock \showarticletitle{iCaRL: Incremental Classifier and Representation Learning}. In \bibinfo{booktitle}{\emph{CVPR}}. \bibinfo{pages}{5533--5542}.
\newblock


\bibitem[Rolnick et~al\mbox{.}(2019)]%
        {DBLP:conf/nips/RolnickASLW19}
\bibfield{author}{\bibinfo{person}{David Rolnick}, \bibinfo{person}{Arun Ahuja}, \bibinfo{person}{Jonathan Schwarz}, \bibinfo{person}{Timothy~P. Lillicrap}, {and} \bibinfo{person}{Gregory Wayne}.} \bibinfo{year}{2019}\natexlab{}.
\newblock \showarticletitle{Experience Replay for Continual Learning}. In \bibinfo{booktitle}{\emph{NeurIPS}}. \bibinfo{pages}{348--358}.
\newblock


\bibitem[Russakovsky et~al\mbox{.}(2015)]%
        {DBLP:journals/ijcv/RussakovskyDSKS15}
\bibfield{author}{\bibinfo{person}{Olga Russakovsky}, \bibinfo{person}{Jia Deng}, \bibinfo{person}{Hao Su}, \bibinfo{person}{Jonathan Krause}, \bibinfo{person}{Sanjeev Satheesh}, \bibinfo{person}{Sean Ma}, \bibinfo{person}{Zhiheng Huang}, \bibinfo{person}{Andrej Karpathy}, \bibinfo{person}{Aditya Khosla}, \bibinfo{person}{Michael~S. Bernstein}, \bibinfo{person}{Alexander~C. Berg}, {and} \bibinfo{person}{Li Fei{-}Fei}.} \bibinfo{year}{2015}\natexlab{}.
\newblock \showarticletitle{ImageNet Large Scale Visual Recognition Challenge}.
\newblock \bibinfo{journal}{\emph{Int. J. Comput. Vis.}} \bibinfo{volume}{115}, \bibinfo{number}{3} (\bibinfo{year}{2015}), \bibinfo{pages}{211--252}.
\newblock


\bibitem[Shim et~al\mbox{.}(2021)]%
        {DBLP:conf/aaai/ShimMJSKJ21}
\bibfield{author}{\bibinfo{person}{Dongsub Shim}, \bibinfo{person}{Zheda Mai}, \bibinfo{person}{Jihwan Jeong}, \bibinfo{person}{Scott Sanner}, \bibinfo{person}{Hyunwoo Kim}, {and} \bibinfo{person}{Jongseong Jang}.} \bibinfo{year}{2021}\natexlab{}.
\newblock \showarticletitle{Online Class-Incremental Continual Learning with Adversarial Shapley Value}. In \bibinfo{booktitle}{\emph{AAAI}}. \bibinfo{pages}{9630--9638}.
\newblock


\bibitem[Shin et~al\mbox{.}(2017)]%
        {DBLP:conf/nips/ShinLKK17}
\bibfield{author}{\bibinfo{person}{Hanul Shin}, \bibinfo{person}{Jung~Kwon Lee}, \bibinfo{person}{Jaehong Kim}, {and} \bibinfo{person}{Jiwon Kim}.} \bibinfo{year}{2017}\natexlab{}.
\newblock \showarticletitle{Continual Learning with Deep Generative Replay}. In \bibinfo{booktitle}{\emph{NIPS}}. \bibinfo{pages}{2990--2999}.
\newblock


\bibitem[Shorten and Khoshgoftaar(2019)]%
        {DBLP:journals/jbd/ShortenK19}
\bibfield{author}{\bibinfo{person}{Connor Shorten} {and} \bibinfo{person}{Taghi~M. Khoshgoftaar}.} \bibinfo{year}{2019}\natexlab{}.
\newblock \showarticletitle{A survey on Image Data Augmentation for Deep Learning}.
\newblock \bibinfo{journal}{\emph{J. Big Data}}  \bibinfo{volume}{6} (\bibinfo{year}{2019}), \bibinfo{pages}{60}.
\newblock


\bibitem[van~de Ven and Tolias(2019)]%
        {DBLP:journals/corr/abs-1904-07734}
\bibfield{author}{\bibinfo{person}{Gido~M. van~de Ven} {and} \bibinfo{person}{Andreas~S. Tolias}.} \bibinfo{year}{2019}\natexlab{}.
\newblock \showarticletitle{Three scenarios for continual learning}.
\newblock \bibinfo{journal}{\emph{CoRR}}  \bibinfo{volume}{abs/1904.07734} (\bibinfo{year}{2019}).
\newblock


\bibitem[Verma et~al\mbox{.}(2019)]%
        {DBLP:conf/icml/VermaLBNMLB19}
\bibfield{author}{\bibinfo{person}{Vikas Verma}, \bibinfo{person}{Alex Lamb}, \bibinfo{person}{Christopher Beckham}, \bibinfo{person}{Amir Najafi}, \bibinfo{person}{Ioannis Mitliagkas}, \bibinfo{person}{David Lopez{-}Paz}, {and} \bibinfo{person}{Yoshua Bengio}.} \bibinfo{year}{2019}\natexlab{}.
\newblock \showarticletitle{Manifold Mixup: Better Representations by Interpolating Hidden States}. In \bibinfo{booktitle}{\emph{ICML}}, Vol.~\bibinfo{volume}{97}. \bibinfo{pages}{6438--6447}.
\newblock


\bibitem[Verwimp et~al\mbox{.}(2023)]%
        {DBLP:journals/nn/VerwimpYPHMPLT23}
\bibfield{author}{\bibinfo{person}{Eli Verwimp}, \bibinfo{person}{Kuo Yang}, \bibinfo{person}{Sarah Parisot}, \bibinfo{person}{Lanqing Hong}, \bibinfo{person}{Steven McDonagh}, \bibinfo{person}{Eduardo P{\'{e}}rez{-}Pellitero}, \bibinfo{person}{Matthias~De Lange}, {and} \bibinfo{person}{Tinne Tuytelaars}.} \bibinfo{year}{2023}\natexlab{}.
\newblock \showarticletitle{{CLAD:} {A} realistic Continual Learning benchmark for Autonomous Driving}.
\newblock \bibinfo{journal}{\emph{Neural Networks}}  \bibinfo{volume}{161} (\bibinfo{year}{2023}), \bibinfo{pages}{659--669}.
\newblock


\bibitem[Wang et~al\mbox{.}(2022)]%
        {DBLP:conf/eccv/WangZYZ22}
\bibfield{author}{\bibinfo{person}{Fu{-}Yun Wang}, \bibinfo{person}{Da{-}Wei Zhou}, \bibinfo{person}{Han{-}Jia Ye}, {and} \bibinfo{person}{De{-}Chuan Zhan}.} \bibinfo{year}{2022}\natexlab{}.
\newblock \showarticletitle{{FOSTER:} Feature Boosting and Compression for Class-Incremental Learning}. In \bibinfo{booktitle}{\emph{ECCV}}, Vol.~\bibinfo{volume}{13685}. \bibinfo{pages}{398--414}.
\newblock


\bibitem[Wang et~al\mbox{.}(2024)]%
        {DBLP:journals/pami/WangZSZ24}
\bibfield{author}{\bibinfo{person}{Liyuan Wang}, \bibinfo{person}{Xingxing Zhang}, \bibinfo{person}{Hang Su}, {and} \bibinfo{person}{Jun Zhu}.} \bibinfo{year}{2024}\natexlab{}.
\newblock \showarticletitle{A Comprehensive Survey of Continual Learning: Theory, Method and Application}.
\newblock \bibinfo{journal}{\emph{{IEEE} Trans. Pattern Anal. Mach. Intell.}} \bibinfo{volume}{46}, \bibinfo{number}{8} (\bibinfo{year}{2024}), \bibinfo{pages}{5362--5383}.
\newblock


\bibitem[Wu et~al\mbox{.}(2019)]%
        {DBLP:conf/cvpr/WuCWYLGF19}
\bibfield{author}{\bibinfo{person}{Yue Wu}, \bibinfo{person}{Yinpeng Chen}, \bibinfo{person}{Lijuan Wang}, \bibinfo{person}{Yuancheng Ye}, \bibinfo{person}{Zicheng Liu}, \bibinfo{person}{Yandong Guo}, {and} \bibinfo{person}{Yun Fu}.} \bibinfo{year}{2019}\natexlab{}.
\newblock \showarticletitle{Large Scale Incremental Learning}. In \bibinfo{booktitle}{\emph{CVPR}}. \bibinfo{pages}{374--382}.
\newblock


\bibitem[Xiao et~al\mbox{.}(2017)]%
        {DBLP:journals/corr/abs-1708-07747}
\bibfield{author}{\bibinfo{person}{Han Xiao}, \bibinfo{person}{Kashif Rasul}, {and} \bibinfo{person}{Roland Vollgraf}.} \bibinfo{year}{2017}\natexlab{}.
\newblock \showarticletitle{Fashion-MNIST: a Novel Image Dataset for Benchmarking Machine Learning Algorithms}.
\newblock \bibinfo{journal}{\emph{CoRR}}  \bibinfo{volume}{abs/1708.07747} (\bibinfo{year}{2017}).
\newblock


\bibitem[Xie et~al\mbox{.}(2020)]%
        {DBLP:conf/sc/XieRLYXWLAXS20}
\bibfield{author}{\bibinfo{person}{Minhui Xie}, \bibinfo{person}{Kai Ren}, \bibinfo{person}{Youyou Lu}, \bibinfo{person}{Guangxu Yang}, \bibinfo{person}{Qingxing Xu}, \bibinfo{person}{Bihai Wu}, \bibinfo{person}{Jiazhen Lin}, \bibinfo{person}{Hongbo Ao}, \bibinfo{person}{Wanhong Xu}, {and} \bibinfo{person}{Jiwu Shu}.} \bibinfo{year}{2020}\natexlab{}.
\newblock \showarticletitle{Kraken: memory-efficient continual learning for large-scale real-time recommendations}. In \bibinfo{booktitle}{\emph{SC}}. \bibinfo{pages}{21}.
\newblock


\bibitem[Xu et~al\mbox{.}(2020)]%
        {DBLP:conf/aaai/XuZNLWTZ20}
\bibfield{author}{\bibinfo{person}{Minghao Xu}, \bibinfo{person}{Jian Zhang}, \bibinfo{person}{Bingbing Ni}, \bibinfo{person}{Teng Li}, \bibinfo{person}{Chengjie Wang}, \bibinfo{person}{Qi Tian}, {and} \bibinfo{person}{Wenjun Zhang}.} \bibinfo{year}{2020}\natexlab{}.
\newblock \showarticletitle{Adversarial Domain Adaptation with Domain Mixup}. In \bibinfo{booktitle}{\emph{AAAI}}. \bibinfo{pages}{6502--6509}.
\newblock


\bibitem[Yan et~al\mbox{.}(2021)]%
        {DBLP:conf/cvpr/YanX021}
\bibfield{author}{\bibinfo{person}{Shipeng Yan}, \bibinfo{person}{Jiangwei Xie}, {and} \bibinfo{person}{Xuming He}.} \bibinfo{year}{2021}\natexlab{}.
\newblock \showarticletitle{{DER:} Dynamically Expandable Representation for Class Incremental Learning}. In \bibinfo{booktitle}{\emph{CVPR}}. \bibinfo{pages}{3014--3023}.
\newblock


\bibitem[Yang et~al\mbox{.}(2022)]%
        {DBLP:journals/corr/abs-2204-08610}
\bibfield{author}{\bibinfo{person}{Suorong Yang}, \bibinfo{person}{Weikang Xiao}, \bibinfo{person}{Mengchen Zhang}, \bibinfo{person}{Suhan Guo}, \bibinfo{person}{Jian Zhao}, {and} \bibinfo{person}{Furao Shen}.} \bibinfo{year}{2022}\natexlab{}.
\newblock \showarticletitle{Image Data Augmentation for Deep Learning: {A} Survey}.
\newblock \bibinfo{journal}{\emph{CoRR}}  \bibinfo{volume}{abs/2204.08610} (\bibinfo{year}{2022}).
\newblock


\bibitem[Yao et~al\mbox{.}(2022)]%
        {DBLP:conf/icml/Yao0LZL0F22}
\bibfield{author}{\bibinfo{person}{Huaxiu Yao}, \bibinfo{person}{Yu Wang}, \bibinfo{person}{Sai Li}, \bibinfo{person}{Linjun Zhang}, \bibinfo{person}{Weixin Liang}, \bibinfo{person}{James Zou}, {and} \bibinfo{person}{Chelsea Finn}.} \bibinfo{year}{2022}\natexlab{}.
\newblock \showarticletitle{Improving Out-of-Distribution Robustness via Selective Augmentation}. In \bibinfo{booktitle}{\emph{ICML}}, Vol.~\bibinfo{volume}{162}. \bibinfo{pages}{25407--25437}.
\newblock


\bibitem[Yun et~al\mbox{.}(2019)]%
        {DBLP:conf/iccv/YunHCOYC19}
\bibfield{author}{\bibinfo{person}{Sangdoo Yun}, \bibinfo{person}{Dongyoon Han}, \bibinfo{person}{Sanghyuk Chun}, \bibinfo{person}{Seong~Joon Oh}, \bibinfo{person}{Youngjoon Yoo}, {and} \bibinfo{person}{Junsuk Choe}.} \bibinfo{year}{2019}\natexlab{}.
\newblock \showarticletitle{CutMix: Regularization Strategy to Train Strong Classifiers With Localizable Features}. In \bibinfo{booktitle}{\emph{ICCV}}. \bibinfo{pages}{6022--6031}.
\newblock


\bibitem[Zenke et~al\mbox{.}(2017)]%
        {DBLP:conf/icml/ZenkePG17}
\bibfield{author}{\bibinfo{person}{Friedemann Zenke}, \bibinfo{person}{Ben Poole}, {and} \bibinfo{person}{Surya Ganguli}.} \bibinfo{year}{2017}\natexlab{}.
\newblock \showarticletitle{Continual Learning Through Synaptic Intelligence}. In \bibinfo{booktitle}{\emph{ICML}}, Vol.~\bibinfo{volume}{70}. \bibinfo{pages}{3987--3995}.
\newblock


\bibitem[Zhang et~al\mbox{.}(2018)]%
        {DBLP:conf/iclr/ZhangCDL18}
\bibfield{author}{\bibinfo{person}{Hongyi Zhang}, \bibinfo{person}{Moustapha Ciss{\'{e}}}, \bibinfo{person}{Yann~N. Dauphin}, {and} \bibinfo{person}{David Lopez{-}Paz}.} \bibinfo{year}{2018}\natexlab{}.
\newblock \showarticletitle{mixup: Beyond Empirical Risk Minimization}. In \bibinfo{booktitle}{\emph{ICLR}}.
\newblock


\bibitem[Zhang et~al\mbox{.}(2021)]%
        {DBLP:conf/iclr/ZhangDKG021}
\bibfield{author}{\bibinfo{person}{Linjun Zhang}, \bibinfo{person}{Zhun Deng}, \bibinfo{person}{Kenji Kawaguchi}, \bibinfo{person}{Amirata Ghorbani}, {and} \bibinfo{person}{James Zou}.} \bibinfo{year}{2021}\natexlab{}.
\newblock \showarticletitle{How Does Mixup Help With Robustness and Generalization?}. In \bibinfo{booktitle}{\emph{ICLR}}.
\newblock


\bibitem[Zhang et~al\mbox{.}(2022a)]%
        {DBLP:conf/icml/ZhangDK022}
\bibfield{author}{\bibinfo{person}{Linjun Zhang}, \bibinfo{person}{Zhun Deng}, \bibinfo{person}{Kenji Kawaguchi}, {and} \bibinfo{person}{James Zou}.} \bibinfo{year}{2022}\natexlab{a}.
\newblock \showarticletitle{When and How Mixup Improves Calibration}. In \bibinfo{booktitle}{\emph{ICML}}, Vol.~\bibinfo{volume}{162}. \bibinfo{pages}{26135--26160}.
\newblock


\bibitem[Zhang et~al\mbox{.}(2022b)]%
        {DBLP:conf/nips/ZhangPFBLJ22}
\bibfield{author}{\bibinfo{person}{Yaqian Zhang}, \bibinfo{person}{Bernhard Pfahringer}, \bibinfo{person}{Eibe Frank}, \bibinfo{person}{Albert Bifet}, \bibinfo{person}{Nick Jin~Sean Lim}, {and} \bibinfo{person}{Yunzhe Jia}.} \bibinfo{year}{2022}\natexlab{b}.
\newblock \showarticletitle{A simple but strong baseline for online continual learning: Repeated Augmented Rehearsal}. In \bibinfo{booktitle}{\emph{NeurIPS}}.
\newblock


\bibitem[Zhao et~al\mbox{.}(2020)]%
        {DBLP:conf/cvpr/ZhaoXGZX20}
\bibfield{author}{\bibinfo{person}{Bowen Zhao}, \bibinfo{person}{Xi Xiao}, \bibinfo{person}{Guojun Gan}, \bibinfo{person}{Bin Zhang}, {and} \bibinfo{person}{Shu{-}Tao Xia}.} \bibinfo{year}{2020}\natexlab{}.
\newblock \showarticletitle{Maintaining Discrimination and Fairness in Class Incremental Learning}. In \bibinfo{booktitle}{\emph{CVPR}}. \bibinfo{pages}{13205--13214}.
\newblock


\bibitem[Zheng et~al\mbox{.}(2024)]%
        {DBLP:conf/icml/Zheng0YZ24}
\bibfield{author}{\bibinfo{person}{Bowen Zheng}, \bibinfo{person}{Da{-}Wei Zhou}, \bibinfo{person}{Han{-}Jia Ye}, {and} \bibinfo{person}{De{-}Chuan Zhan}.} \bibinfo{year}{2024}\natexlab{}.
\newblock \showarticletitle{Multi-layer Rehearsal Feature Augmentation for Class-Incremental Learning}. In \bibinfo{booktitle}{\emph{ICML}}.
\newblock


\bibitem[Zheng et~al\mbox{.}(2025)]%
        {DBLP:conf/cvpr/Zheng0YZ25}
\bibfield{author}{\bibinfo{person}{Bowen Zheng}, \bibinfo{person}{Da{-}Wei Zhou}, \bibinfo{person}{Han{-}Jia Ye}, {and} \bibinfo{person}{De{-}Chuan Zhan}.} \bibinfo{year}{2025}\natexlab{}.
\newblock \showarticletitle{Task-Agnostic Guided Feature Expansion for Class-Incremental Learning}. In \bibinfo{booktitle}{\emph{CVPR}}. \bibinfo{pages}{10099--10109}.
\newblock


\bibitem[Zhou et~al\mbox{.}(2023)]%
        {DBLP:conf/iclr/0001WYZ23}
\bibfield{author}{\bibinfo{person}{Da{-}Wei Zhou}, \bibinfo{person}{Qi{-}Wei Wang}, \bibinfo{person}{Han{-}Jia Ye}, {and} \bibinfo{person}{De{-}Chuan Zhan}.} \bibinfo{year}{2023}\natexlab{}.
\newblock \showarticletitle{A Model or 603 Exemplars: Towards Memory-Efficient Class-Incremental Learning}. In \bibinfo{booktitle}{\emph{ICLR}}.
\newblock


\bibitem[Zhu et~al\mbox{.}(2021)]%
        {DBLP:conf/nips/ZhuCZL21}
\bibfield{author}{\bibinfo{person}{Fei Zhu}, \bibinfo{person}{Zhen Cheng}, \bibinfo{person}{Xu{-}Yao Zhang}, {and} \bibinfo{person}{Cheng{-}Lin Liu}.} \bibinfo{year}{2021}\natexlab{}.
\newblock \showarticletitle{Class-Incremental Learning via Dual Augmentation}. In \bibinfo{booktitle}{\emph{NeurIPS}}. \bibinfo{pages}{14306--14318}.
\newblock


\end{thebibliography}

\newpage
\appendix
\onecolumn

\section{Appendix -- Theory}
\label{appendix:theory}

\subsection{Theoretical Analysis of Applying Mixup in Class-Incremental Learning} 
\label{appendix:theoretical_analysis_mixup}

Continuing from Sec.~\ref{mixup_theoretical}, we prove the theorem stating the sufficient condition under which applying Mixup to an experience replay method worsens catastrophic forgetting in class-incremental learning.

\begin{theorem*}
Let $\tilde{d}_{ij}$ be a mixed sample by mixing an original training sample $d_i$ and a randomly selected sample $d_j$. If the gradient of the mixed sample satisfies the following condition: $\sum_{y \in \mathbb{Y}_p} \nabla_{\theta} \ell(f_{\theta}^{l-1}, G_y)^\top (\nabla_{\theta} \ell(f_{\theta}^{l-1}, d_i) - \nabla_{\theta} \ell(f_{\theta}^{l-1}, \tilde{d}_{ij})) > 0$, then training with the mixed sample leads to a higher average loss for previous tasks and worsens catastrophic forgetting compared to training with the original sample.
\end{theorem*}

\begin{proof}
We first derive the updated average loss of the data from previous tasks during the training procedure of the experience replay method (ER) when training with the original sample $d_i$. By employing Eq.~\ref{eq:loss} and replacing the term for the average gradient of the training data from the current task, $\nabla_{\theta} \ell(f_{\theta}^{l-1}, T_l)$, with the gradient of the original training sample, $\nabla_{\theta} \ell(f_{\theta}^{l-1}, d_i)$, in the equation,
\begin{equation} \label{loss_er}
\begin{aligned}
L_{\text{ER}} &= \frac{1}{|\mathbb{Y}_p|} \sum_{y \in \mathbb{Y}_p} \tilde{\ell}(f_{\theta}, G_y) \\
&= \frac{1}{|\mathbb{Y}_p|} \sum_{y \in \mathbb{Y}_p} \Bigl(\ell(f_{\theta}^{l-1}, G_y) - \eta \nabla_{\theta} \ell(f_{\theta}^{l-1}, G_y)^\top \nabla_{\theta} \ell(f_{\theta}^{l-1}, d_i)\Bigr) \\
&= \frac{1}{|\mathbb{Y}_p|} \sum_{y \in \mathbb{Y}_p} \ell(f_{\theta}^{l-1}, G_y) - \eta \frac{1}{|\mathbb{Y}_p|} \sum_{y \in \mathbb{Y}_p} \nabla_{\theta} \ell(f_{\theta}^{l-1}, G_y)^\top \nabla_{\theta} \ell(f_{\theta}^{l-1}, d_i).
\end{aligned}
\end{equation}

We next consider applying Mixup within the framework of ER. By using the mixed sample $\tilde{d}_{ij}$ for training, the updated average loss of previous tasks is derived as follows:
\begin{equation} \label{loss_mixup}
L_{\text{Mixup}} = \frac{1}{|\mathbb{Y}_p|} \sum_{y \in \mathbb{Y}_p} \ell(f_{\theta}^{l-1}, G_y) - \eta \frac{1}{|\mathbb{Y}_p|} \sum_{y \in \mathbb{Y}_p} \nabla_{\theta} \ell(f_{\theta}^{l-1}, G_y)^\top \nabla_{\theta} \ell(f_{\theta}^{l-1}, \tilde{d}_{ij}).
\end{equation}

By subtracting Eq.~\ref{loss_mixup} from Eq.~\ref{loss_er}, 
\begin{equation} \label{loss_subtract}
\begin{aligned}
L_{\text{ER}} - L_{\text{Mixup}} &= - \eta \frac{1}{|\mathbb{Y}_p|} \sum_{y \in \mathbb{Y}_p} \Bigl(\nabla_{\theta} \ell(f_{\theta}^{l-1}, G_y)^\top \nabla_{\theta} \ell(f_{\theta}^{l-1}, d_i) - \nabla_{\theta} \ell(f_{\theta}^{l-1}, G_y)^\top \nabla_{\theta} \ell(f_{\theta}^{l-1}, \tilde{d}_{ij})\Bigr) \\
&= - \eta \frac{1}{|\mathbb{Y}_p|} \sum_{y \in \mathbb{Y}_p} \Bigl(\nabla_{\theta} \ell(f_{\theta}^{l-1}, G_y)^\top (\nabla_{\theta} \ell(f_{\theta}^{l-1}, d_i) - \nabla_{\theta} \ell(f_{\theta}^{l-1}, \tilde{d}_{ij}))\Bigr).
\end{aligned}
\end{equation}

If the gradient of the mixed sample satisfies the following condition: $\sum_{y \in \mathbb{Y}_p} \nabla_{\theta} \ell(f_{\theta}^{l-1}, G_y)^\top (\nabla_{\theta} \ell(f_{\theta}^{l-1}, d_i) - \nabla_{\theta} \ell(f_{\theta}^{l-1}, \tilde{d}_{ij})) > 0$, then Eq.~\ref{loss_subtract} becomes: 
\begin{equation*}
L_{\text{ER}} - L_{\text{Mixup}} < 0.
\end{equation*}

Therefore, if the gradient of the mixed sample satisfies the sufficient condition, training with the mixed sample leads to a higher average loss (i.e., lower average accuracy) for previous tasks and worsens catastrophic forgetting compared to training with the original sample.

\end{proof}

\subsection{Theoretical Analysis of the Gradients of Mixed Samples} 
\label{appendix:theoretical_analysis_gradient}

Continuing from Sec.~\ref{gradient_mixup}, we first prove the lemma on the approximated gradient vector of the original sample.

\begin{lemma*}
Let $w$ and $b$ denote the weights and biases of the last layer of the model. By using the last layer approximation, the gradient of the original sample $d = (x, y)$ is approximated as follows:
\begin{equation*}
g = (\nabla_{b} \ell(f_{\theta}^{l-1}, d), \nabla_{w} \ell(f_{\theta}^{l-1}, d))^\top = (\hat{y} - y, (\hat{y} - y)^\top X)^\top,
\end{equation*}
where $X$ is the feature embedding before the last layer and $\hat{y}$ is the model output of $x$ with respect to the model $f_{\theta}^{l-1}$.
\end{lemma*}
\begin{proof}
We employ the last layer approximation that approximates the overall gradients using the gradients with respect to the last layer's weights and biases. Let the weights and biases of the model's last linear layer be denoted by $w$ and $b$, respectively. Denoting the logit by $z$, the feature embedding before the last linear layer by $X$, and the softmax function by $\sigma$, the model output $\hat{y}$ is derived as follows:
\begin{equation*}
\hat{y} = \sigma(z) = \sigma(w \cdot X + b).
\end{equation*}

We denote the cross-entropy loss as $\ell(f_{\theta}^{l-1}, d) = - \sum_{i=1}^{c} y_i \log (\hat{y}_i)$, where $y$ is the true label and $\hat{y}$ is the model prediction for the sample $d = (x, y)$, and $c$ is the number of classes. We first compute the gradients with respect to the logit $z$ as follows:
\begin{equation} \label{grad_logit1}
\begin{aligned}
\frac{\partial \ell}{\partial z_j} = - \frac{\partial}{\partial z_j} \sum_{i=1}^{c} y_i \cdot \log (\hat{y}_i) = - \sum_{i=1}^{c} y_i \cdot \frac{\partial}{\partial z_j} \log (\hat{y}_i).
\end{aligned}
\end{equation}

Since there are $c$ classes, we can represent the model prediction as $\hat{y} = [\hat{y}_1, \ldots , \hat{y}_c]$, where $\hat{y}_i = \frac{e^{z_i}}{\sum_{n=1}^{c} e^{z_n}}$ for $\forall i \in \{1, \ldots, c\}$. By substituting this equation into Eq.~\ref{grad_logit1}, we rearrange the term $\frac{\partial}{\partial z_j} \log (\hat{y}_i)$ as follows:
\begin{equation} \label{grad_logit2}
\begin{aligned}
\frac{\partial}{\partial z_j} \log (\hat{y}_i) &= \frac{\partial}{\partial z_j} \log \biggl(\frac{e^{z_i}}{\sum_{n=1}^{c} e^{z_n}}\biggr) \\
&= \frac{\partial z_i}{\partial z_j} - \frac{\partial}{\partial z_j} \log \Biggl(\sum_{n=1}^{c} e^{z_n}\Biggr) \\
&= \mathbbm{1} \{i = j\} - \frac{1}{\sum_{n=1}^{c} e^{z_n}} \cdot \Biggl(\frac{\partial}{\partial z_j} \sum_{n=1}^{c} e^{z_n}\Biggr) \\
&= \mathbbm{1} \{i = j\} - \frac{e^{z_j}}{\sum_{n=1}^{c} e^{z_n}} \\
&= \mathbbm{1} \{i = j\} - \hat{y}_j,
\end{aligned}
\end{equation}

where $\mathbbm{1} \{ \cdot \}$ is an indicator function. By substituting Eq.~\ref{grad_logit2} into Eq.~\ref{grad_logit1}, we get
\begin{equation} \label{grad_logit3}
\begin{aligned}
\frac{\partial \ell}{\partial z_j} &= - \sum_{i=1}^{c} y_i \cdot (\mathbbm{1} \{i = j\} - \hat{y}_j) \\
&= \sum_{i=1}^{c} y_i \cdot \hat{y}_j - \sum_{i=1}^{c} y_i \cdot \mathbbm{1} \{i = j\} \\
&= \hat{y}_j \cdot \sum_{i=1}^{c} y_i - \sum_{i=1}^{c} y_i \cdot \mathbbm{1} \{i = j\}.
\end{aligned}
\end{equation}

Since the sum of the elements in the one-hot vector $y$ equals to 1, we have $\sum_{i=1}^{c} y_i = 1$. Next, the indicator function $\mathbbm{1} \{i = j\}$ is 1 if $i$ equals to $j$ and 0 otherwise. Then, $\sum_{i=1}^{c} y_i \cdot \mathbbm{1} \{i = j\} = y_j$. By rearranging the terms in Eq.~\ref{grad_logit3} using these equations, we get the followings: 
\begin{equation*}
\begin{aligned}
\frac{\partial \ell}{\partial z_j} = \hat{y}_j - y_j.
\end{aligned}
\end{equation*}

In the vector notation, we get:
\begin{equation*}
\frac{\partial \ell}{\partial z} = \hat{y} - y.
\end{equation*}

By applying the chain rule and using the equation $z = w \cdot X + b$, we derive the gradients with respect to the last layer's weights $w$ and biases $b$ as follows:
\begin{equation*}
\begin{aligned}
\frac{\partial \ell}{\partial b} &= \frac{\partial \ell}{\partial z} \cdot \frac{\partial z}{\partial b} = \hat{y} - y, \\
\frac{\partial \ell}{\partial w} &= \frac{\partial \ell}{\partial z} \cdot \frac{\partial z}{\partial w} = (\hat{y} - y) \cdot X.
\end{aligned}
\end{equation*}

Therefore, we derive the gradient vector of the original sample with respect to the last layer as follows:

\begin{equation} \label{last_layer_gradient}
g = (\nabla_{b} \ell(f_{\theta}^{l-1}, d), \nabla_{w} \ell(f_{\theta}^{l-1}, d))^\top = (\hat{y} - y, (\hat{y} - y)^\top X)^\top.
\end{equation}

\end{proof}

Using the Lemma, we next prove the theorem on the approximated gradient vector of a mixed sample.

\begin{theorem*}
If the model is linear, the gradient of the mixed sample $\tilde{d}_{ij} = (\tilde{x}_{ij}, \tilde{y}_{ij})$, generated by mixing two samples $d_i = (x_i, y_i)$ and $d_j = (x_j, y_j)$ with a mixing ratio $\lambda$, is approximated as follows:
\begin{equation*}
\begin{aligned}
\tilde{g}_{ij} &= (\nabla_{b} \ell(f_{\theta}^{l-1}, \tilde{d}_{ij}), \nabla_{w} \ell(f_{\theta}^{l-1}, \tilde{d}_{ij}))^\top \\
&= (\lambda (\hat{y}_{ij} - y_i) + (1 - \lambda) (\hat{y}_{ij} - y_j), (\lambda (\hat{y}_{ij} - y_i) + (1 - \lambda) (\hat{y}_{ij} - y_j))^\top X_{ij})^\top,
\end{aligned}
\end{equation*}
where $\hat{y}_{ij} = \biggl[\frac{p_1^{\lambda} q_1^{1 - \lambda}}{p_1^{\lambda} q_1^{1 - \lambda} + \cdots + p_c^{\lambda} q_c^{1 - \lambda}}, \ldots , \frac{p_c^{\lambda} q_c^{1 - \lambda}}{p_1^{\lambda} q_1^{1 - \lambda} + \cdots + p_c^{\lambda} q_c^{1 - \lambda}} \biggr]$ with $\hat{y}_i = [p_1, \ldots , p_c]$ and $\hat{y}_j = [q_1, \ldots , q_c]$ for $c$ classes. $X_{ij} = \lambda X_i + (1 - \lambda) X_j$, where $X_i$ and $X_j$ are the feature embeddings of $x_i$ and $x_j$ before the last layer.
\end{theorem*}
\begin{proof}
Mixup generates a virtual sample $(\tilde{x}_{ij}, \tilde{y}_{ij})$ using two samples $(x_i, y_i)$ and $(x_j, y_j)$ as follows:
\begin{equation*}
\begin{aligned}
\tilde{x}_{ij} = \lambda x_i + (1 - \lambda) x_j, \ \tilde{y}_{ij} =  \lambda y_i + (1 - \lambda) y_j.
\end{aligned}
\end{equation*}

By employing Eq.~\ref{last_layer_gradient}, we get the last layer gradients of the two original samples with respect to the model $f_{\theta}^{l-1}$ as follows:
\begin{equation*}
g_i = (\hat{y}_i - y_i, (\hat{y}_i - y_i)^\top X_i)^\top, \ g_j = (\hat{y}_j - y_j, (\hat{y}_j - y_j)^\top X_j)^\top,
\end{equation*}
where $\hat{y}_i$ and $\hat{y}_j$ are the model predictions, and $X_i$ and $X_j$ are the feature embeddings before the last layer of the model for $x_i$ and $x_j$, respectively. 

We next derive the last layer gradient of the mixed sample $(\tilde{x}_{ij}, \tilde{y}_{ij})$ as follows:
\begin{equation} \label{eq:mixed_gradient}
\begin{aligned}
\tilde{g}_{ij} & = (\hat{y}_{ij} - \tilde{y}_{ij}, (\hat{y}_{ij} - \tilde{y}_{ij})^\top X_{ij})^\top \\
&= (\hat{y}_{ij} - (\lambda y_i + (1 - \lambda) y_j), (\hat{y}_{ij} - (\lambda y_i + (1 - \lambda) y_j))^\top X_{ij})^\top \\
&= ((\lambda \hat{y}_{ij} + (1 - \lambda) \hat{y}_{ij}) - (\lambda y_i + (1 - \lambda) y_j), ((\lambda \hat{y}_{ij} + (1 - \lambda) \hat{y}_{ij}) - (\lambda y_i + (1 - \lambda) y_j))^\top X_{ij})^\top \\
&= (\lambda (\hat{y}_{ij} - y_i) + (1 - \lambda) (\hat{y}_{ij} - y_j), (\lambda (\hat{y}_{ij} - y_i) + (1 - \lambda) (\hat{y}_{ij} - y_j))^\top X_{ij})^\top,
\end{aligned}
\end{equation}
where $\hat{y}_{ij}$ is the model prediction of $\tilde{x}_{ij}$ and $X_{ij}$ is the feature embedding of $\tilde{x}_{ij}$ before the last layer of the model. 

Let the linear model be denoted as $f_{\theta}^{l-1}(x) = \theta^\top x$, and the softmax function as $\sigma$. Then, the model prediction $\hat{y}$ is derived as follows:
\begin{equation*}
\hat{y} = \sigma(f_{\theta}^{l-1}(x)) = \sigma(\theta^\top x).
\end{equation*} 

Since there are $c$ classes, we can represent the model predictions for $x_i$ and $x_j$ using the probability values of $p$ and $q$ as,
\begin{equation*}
\begin{aligned}
\hat{y}_i &= \sigma(\theta^\top x_i) = [p_1, \ldots , p_c], \ \text{where} \ \sum_{n=1}^{c} p_n = 1 \ \text{and} \ 0 \leq p_n \leq 1, \\
\hat{y}_j &= \sigma(\theta^\top x_j) = [q_1, \ldots , q_c], \ \text{where} \ \sum_{n=1}^{c} q_n = 1 \ \text{and} \ 0 \leq q_n \leq 1.
\end{aligned}
\end{equation*}

Then, the logits for $x_i$ and $x_j$ become:
\begin{equation*}
z_i = \theta^\top x_i = [\log p_1, \ldots , \log p_c], \ z_j = \theta^\top x_j = [\log q_1, \ldots , \log q_c].
\end{equation*}

Since the mixed sample is generated using the equation $\tilde{x}_{ij} = \lambda x_i + (1 - \lambda) x_j$, we derive the logit of the mixed sample as follows:
\begin{equation*}
\begin{aligned}
z_{ij} &= \theta^\top \tilde{x}_{ij} \\
&= \theta^\top (\lambda x_i + (1 - \lambda) x_j) \\
&= \lambda \theta^\top x_i + (1 - \lambda) \theta^\top x_j \\
&= \lambda [\log p_1, \ldots , \log p_c] + (1 - \lambda) [\log q_1, \ldots , \log q_c] \\
&= [\lambda \log p_1 + (1 - \lambda) \log q_1, \ldots , \lambda \log p_c + (1 - \lambda) \log q_c] \\
&= [\log p_1^{\lambda} + \log q_1^{1 - \lambda}, \ldots , \log p_c^{\lambda} + \log q_c^{1 - \lambda}] \\
&= [\log (p_1^{\lambda} q_1^{1 - \lambda}), \ldots , \log (p_c^{\lambda} q_c^{1 - \lambda})].
\end{aligned}
\end{equation*}

By applying the softmax function to the logit, we get the model prediction of the mixed sample:
\begin{equation} \label{eq:output}
\begin{aligned}
\hat{y}_{ij} &= \sigma(z_{ij}) \\
&= \sigma([\log (p_1^{\lambda} q_1^{1 - \lambda}), \ldots , \log (p_c^{\lambda} q_c^{1 - \lambda})]) \\
&= \biggl[\frac{p_1^{\lambda} q_1^{1 - \lambda}}{p_1^{\lambda} q_1^{1 - \lambda} + \cdots + p_c^{\lambda} q_c^{1 - \lambda}}, \ldots , \frac{p_c^{\lambda} q_c^{1 - \lambda}}{p_1^{\lambda} q_1^{1 - \lambda} + \cdots + p_c^{\lambda} q_c^{1 - \lambda}} \biggr].
\end{aligned}
\end{equation}

Using the linearity of the model, the embedding of the mixed sample is computed as follows:
\begin{equation} \label{eq:embedding}
X_{ij} = \lambda X_i + (1 - \lambda) X_j.
\end{equation}

By substituting Eq.~\ref{eq:output} and Eq.~\ref{eq:embedding} into Eq.~\ref{eq:mixed_gradient}, we get the gradient vector of the mixed sample as follows:
\begin{equation*}
\begin{aligned}
\tilde{g}_{ij} &= (\nabla_{b} \ell(f_{\theta}^{l-1}, \tilde{d}_{ij}), \nabla_{w} \ell(f_{\theta}^{l-1}, \tilde{d}_{ij}))^\top \\
&= (\lambda (\hat{y}_{ij} - y_i) + (1 - \lambda) (\hat{y}_{ij} - y_j), (\lambda (\hat{y}_{ij} - y_i) + (1 - \lambda) (\hat{y}_{ij} - y_j))^\top X_{ij})^\top,
\end{aligned}
\end{equation*}
where $\hat{y}_{ij} = \biggl[\frac{p_1^{\lambda} q_1^{1 - \lambda}}{p_1^{\lambda} q_1^{1 - \lambda} + \cdots + p_c^{\lambda} q_c^{1 - \lambda}}, \ldots , \frac{p_c^{\lambda} q_c^{1 - \lambda}}{p_1^{\lambda} q_1^{1 - \lambda} + \cdots + p_c^{\lambda} q_c^{1 - \lambda}} \biggr]$ and $X_{ij} = \lambda X_i + (1 - \lambda) X_j$. 

\end{proof}

\newpage
\section{Appendix -- Experiments}
\label{appendix:experiments}

\subsection{More Details on Experimental Settings}
\label{appendix:exp-details}

Continuing from Sec.~\ref{parameters}, we provide more details on experimental settings. Regarding the buffer management strategy, other papers allocate the maximum buffer storage of $C \cdot N$ from the beginning and utilize it fully throughout tasks, where $C$ is the total number of classes and $N$ is the number of buffer samples per class. Then the actual number of buffer samples per class becomes larger than $N$ in the earlier tasks. However, we believe that this strategy is impractical because the total number of classes is unknown at the beginning in real-world scenarios. We thus adopt a more practical strategy that incrementally expands the buffer storage by $N$ samples for each new class. The relatively lower accuracy values reported in our paper are due to the adoption of a more practical and challenging buffer management strategy compared to other papers. If we adopt the strategy used in other papers, the results reported in our paper would be further improved.

We implement data augmentation baselines using the parameters specified in their original papers. For the mixup-based baselines, we set the $\alpha$ value of each method as follows: Mixup ($\alpha = 1$ for MNIST, FMNIST, CIFAR-10, and CIFAR-100, and $\alpha = 0.4$ for Tiny ImageNet), Manifold Mixup ($\alpha = 2$ for MNIST, FMNIST, CIFAR-10, and CIFAR-100, and $\alpha = 0.2$ for Tiny ImageNet), CutMix ($\alpha = 1$ for all datasets), Remix ($\alpha = 1$ for MNIST, FMNIST, CIFAR-10, and CIFAR-100, and $\alpha = 0.4$ for Tiny ImageNet), and Balanced-MixUp ($\alpha = 0.2$ for all datasets). For \method{}, we also use the same $\alpha$ with Mixup as $\alpha = 1$ for MNIST, FMNIST, CIFAR-10, and CIFAR-100, and $\alpha = 0.4$ for Tiny ImageNet. For RandAugment, we set the two hyperparameters for the number of applying transformations and their magnitudes to 1 and 14, respectively.

\subsection{More Details on Catastrophic Forgetting in Mixup}
\label{appendix:exp-forgetting-mixup}

Continuing from Sec.~\ref{emp_analysis}, we show the occurrence of catastrophic forgetting when simply applying Mixup to the experience replay-based method (ER) in class-incremental learning as shown in Fig.~\ref{fig:mnist_forget}--Fig.~\ref{fig:tiny_imagenet_forget}. We analyze catastrophic forgetting in both task-wise accuracy across all tasks and class-wise accuracy at the last task. For MNIST, applying Mixup decreases the model accuracy both in task-wise and class-wise evaluations. For the other datasets, although Mixup improves the overall task accuracy, the accuracy of some classes is exposed to catastrophic forgetting and results in lower accuracy at the last task. This phenomenon is not limited to the last task, but we provide the class accuracy of the last task as an example. We also provide the performance of \method{} for comparison, and our method shows robust accuracy in both task-wise and class-wise evaluations.

\vskip -0.1in

\begin{figure}[H]
    \centering
    \begin{minipage}[t]{0.4\textwidth}
        \centering
        \includegraphics[width=\textwidth]{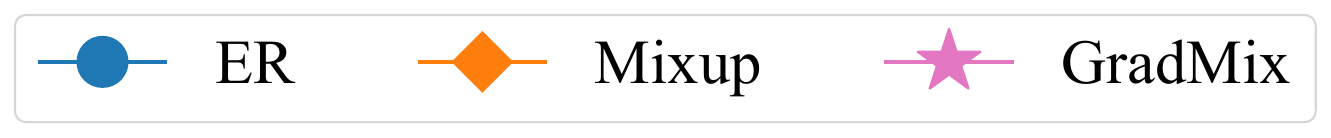}
    \end{minipage}
    \vspace{-0.6cm}
\end{figure}

\begin{figure}[H]
    \centering
    \begin{minipage}[t]{0.31\textwidth}
        \centering
        \includegraphics[width=\textwidth]{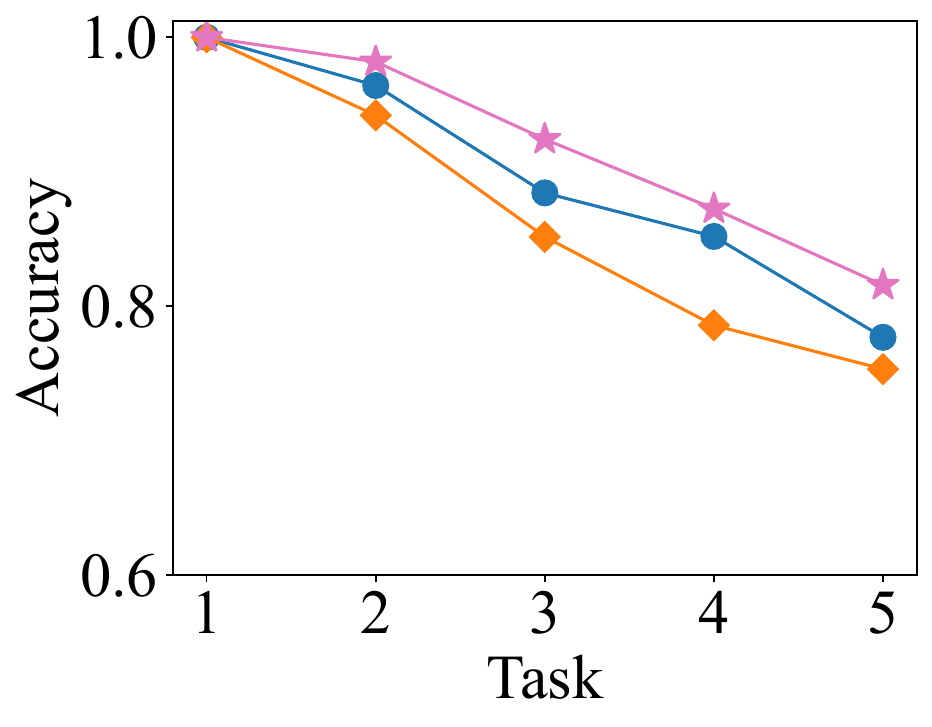}
        \vskip -0.1in
        \caption*{(a) Task-wise accuracy.}
        \label{fig:mnist_forget_task}
    \end{minipage}
    \hspace{0.05\textwidth}
    \begin{minipage}[t]{0.31\textwidth}
        \centering
        \includegraphics[width=\textwidth]{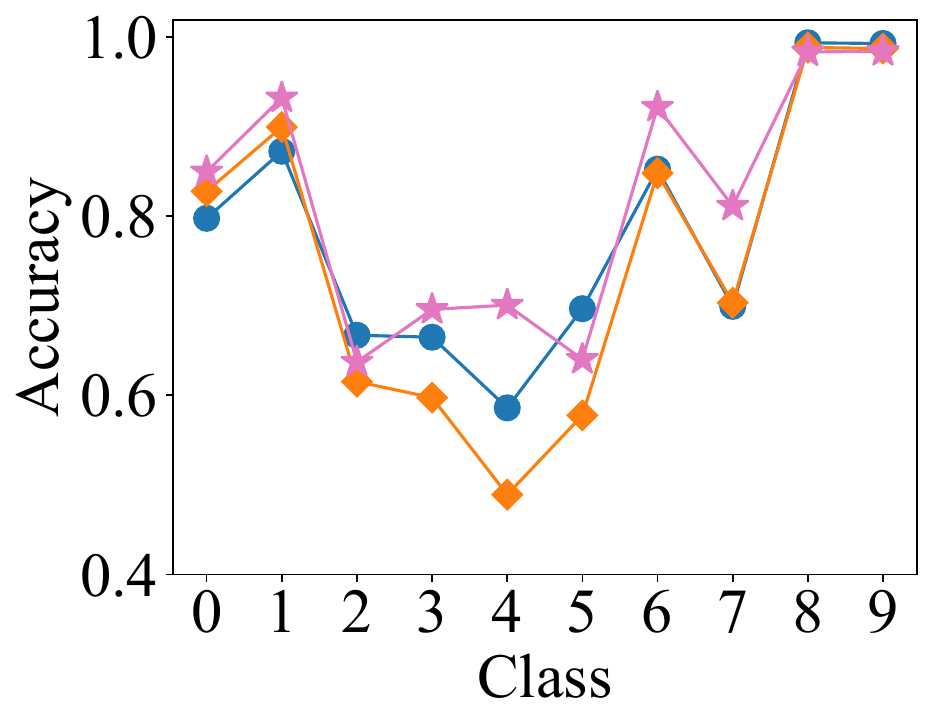}
        \vskip -0.1in
        \caption*{(b) Class-wise accuracy on the last task.}
        \label{fig:mnist_forget_class}
    \end{minipage}
    \vskip -0.1in
    \caption{Comparison between the performance of ER, Mixup, and \method{} on the MNIST dataset.}
    \label{fig:mnist_forget}
\end{figure}

\vspace{-0.3cm}


\begin{figure}[H]
    \centering
    \begin{minipage}[t]{0.31\textwidth}
        \centering
        \includegraphics[width=\textwidth]{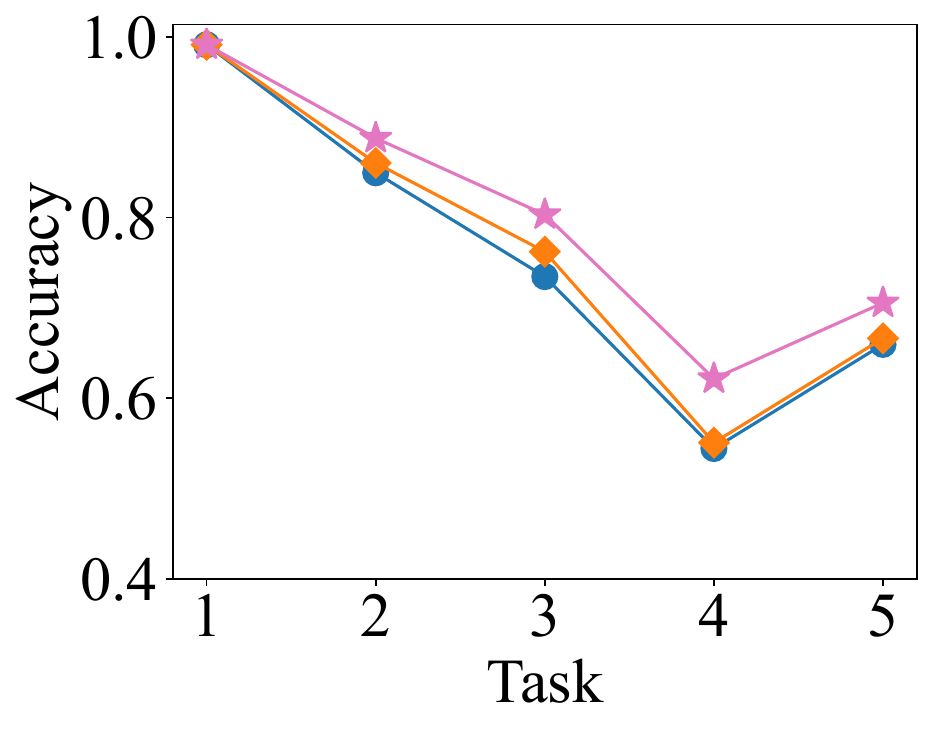}
        \vskip -0.1in
        \caption*{(a) Task-wise accuracy.}
        \label{fig:fmnist_forget_task}
    \end{minipage}
    \hspace{0.05\textwidth}
    \begin{minipage}[t]{0.31\textwidth}
        \centering
        \includegraphics[width=\textwidth]{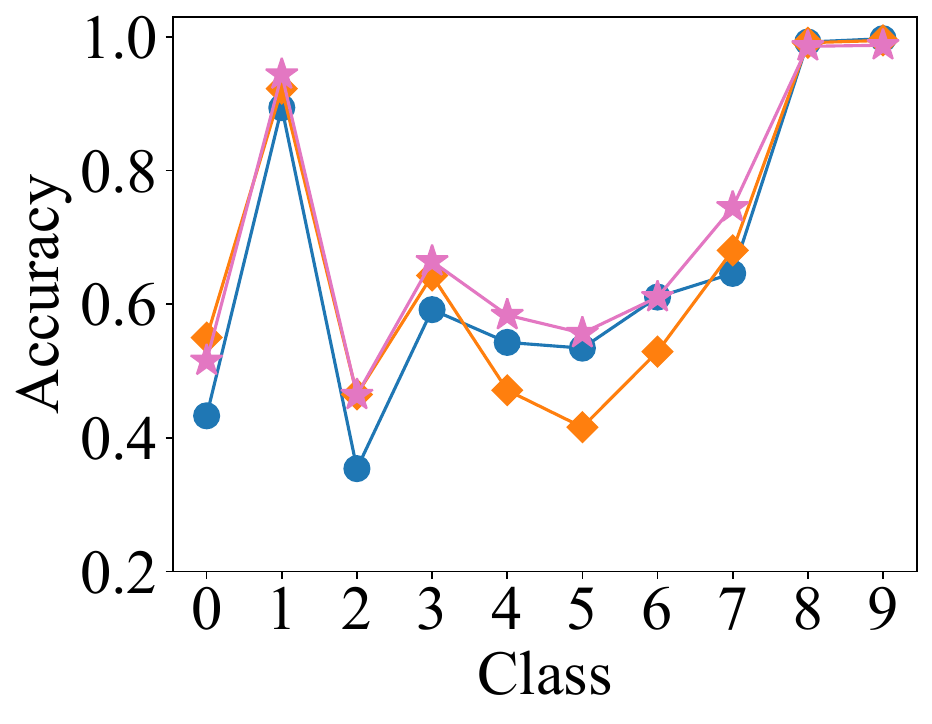}
        \vskip -0.1in
        \caption*{(b) Class-wise accuracy on the last task.}
        \label{fig:fmnist_forget_class}
    \end{minipage}
    \vskip -0.1in
    \caption{Comparison between the performance of ER, Mixup, and \method{} on the Fashion-MNIST dataset.}
    \label{fig:fmnist_forget}
\end{figure}



\begin{figure}[H]
    \centering
    \begin{minipage}[t]{0.31\textwidth}
        \centering
        \includegraphics[width=\textwidth]{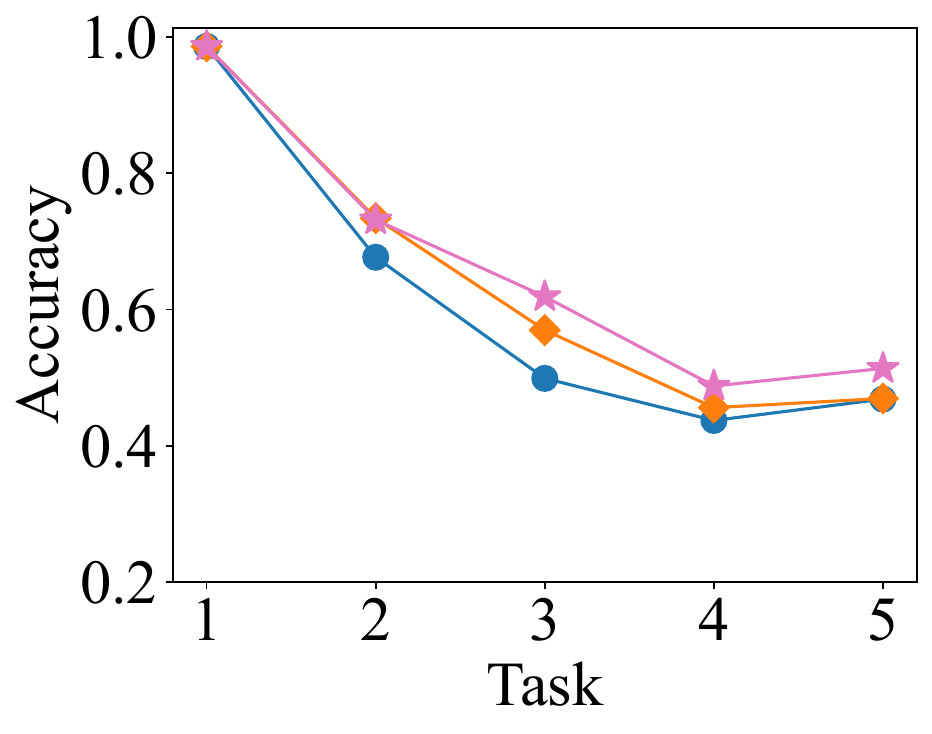}
        \vskip -0.1in
        \caption*{(a) Task-wise accuracy.}
        \label{fig:cifar10_forget_task}
    \end{minipage}
    \hspace{0.05\textwidth}
    \begin{minipage}[t]{0.31\textwidth}
        \centering
        \includegraphics[width=\textwidth]{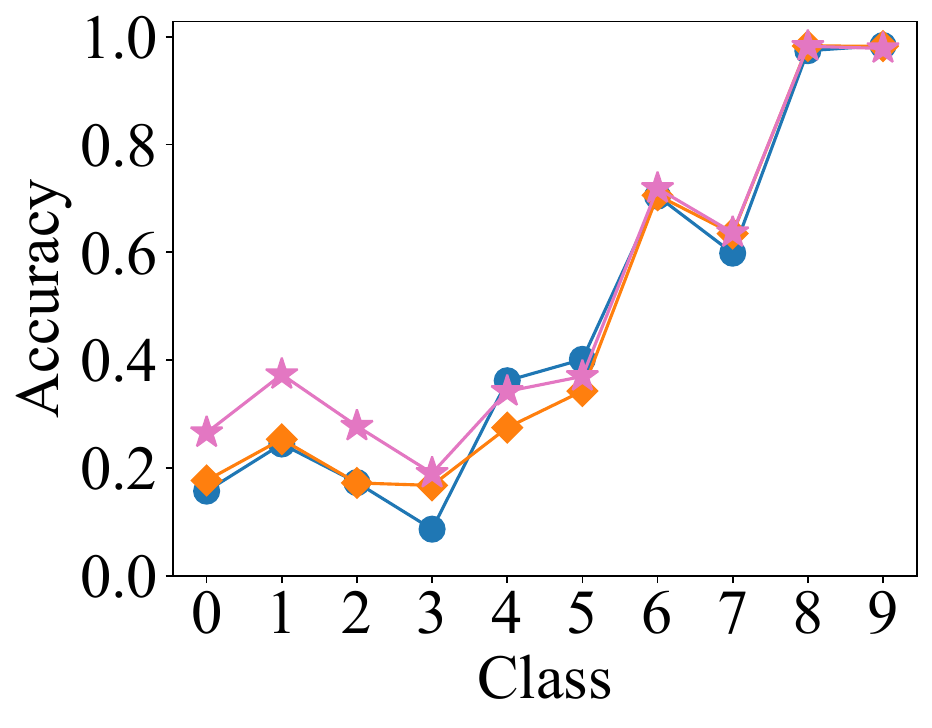}
        \vskip -0.1in
        \caption*{(b) Class-wise accuracy on the last task.}
        \label{fig:cifar10_forget_class}
    \end{minipage}
    \vskip -0.1in
    \caption{Comparison between the performance of ER, Mixup, and \method{} on the CIFAR-10 dataset.}
    \label{fig:cifar10_forget}
    \vspace{-0.8cm}
\end{figure}


\begin{figure}[H]
    \centering
    \begin{minipage}[t]{0.31\textwidth}
        \centering
        \includegraphics[width=\textwidth]{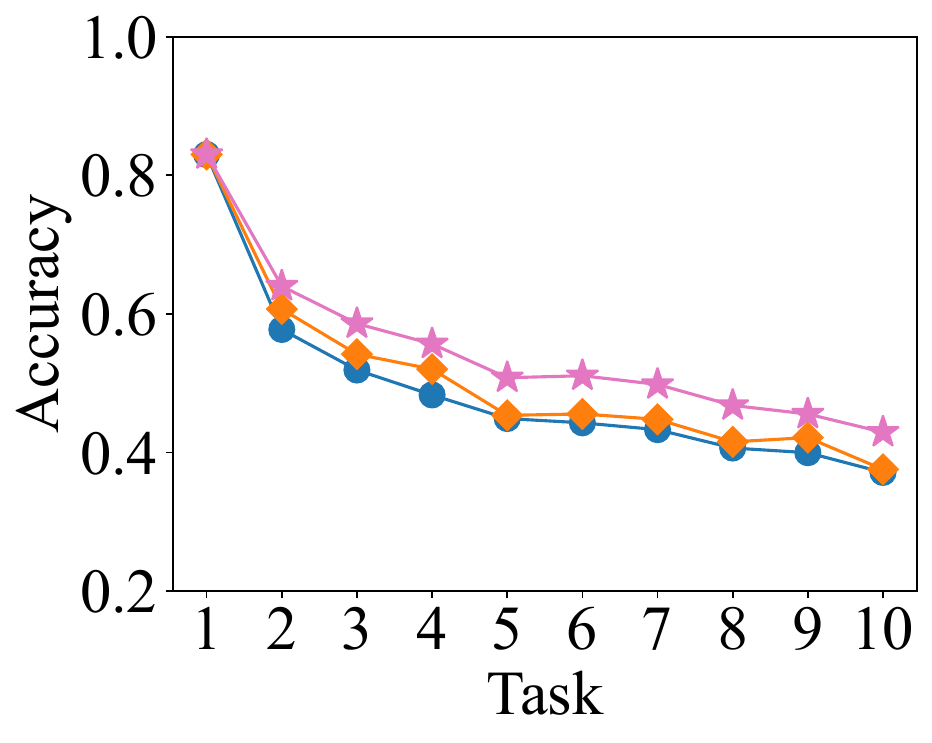}
        \vskip -0.1in
        \caption*{(a) Task-wise accuracy.}
        \label{fig:cifar100_forget_task}
    \end{minipage}
    \hspace{0.05\textwidth}
    \begin{minipage}[t]{0.315\textwidth}
        \centering
        \includegraphics[width=\textwidth]{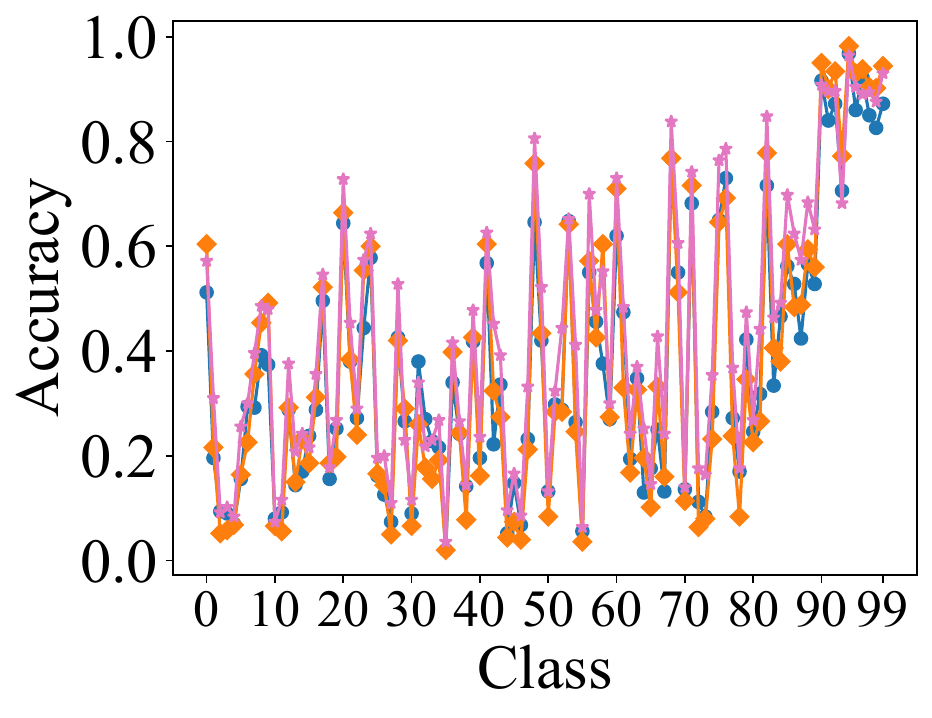}
        \vskip -0.1in
        \caption*{(b) Class-wise accuracy on the last task.}
        \label{fig:cifar100_forget_class}
    \end{minipage}
    \vskip -0.1in
    \caption{Comparison between the performance of ER, Mixup, and \method{} on the CIFAR-100 dataset.}
    \label{fig:cifar100_forget}
    \vspace{-0.8cm}
\end{figure}


\begin{figure}[H]
    \centering
    \begin{minipage}[t]{0.31\textwidth}
        \centering
        \includegraphics[width=\textwidth]{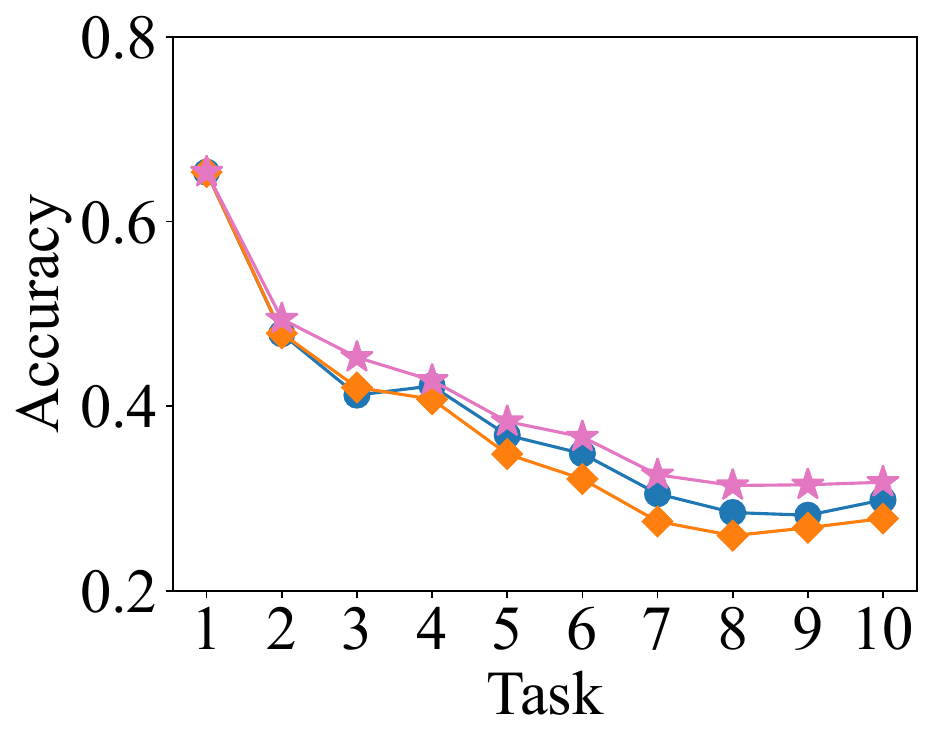}
        \vskip -0.1in
        \caption*{(a) Task-wise accuracy.}
        \label{fig:tiny_imagenet_forget_task}
    \end{minipage}
    \hspace{0.05\textwidth}
    \begin{minipage}[t]{0.32\textwidth}
        \centering
        \includegraphics[width=\textwidth]{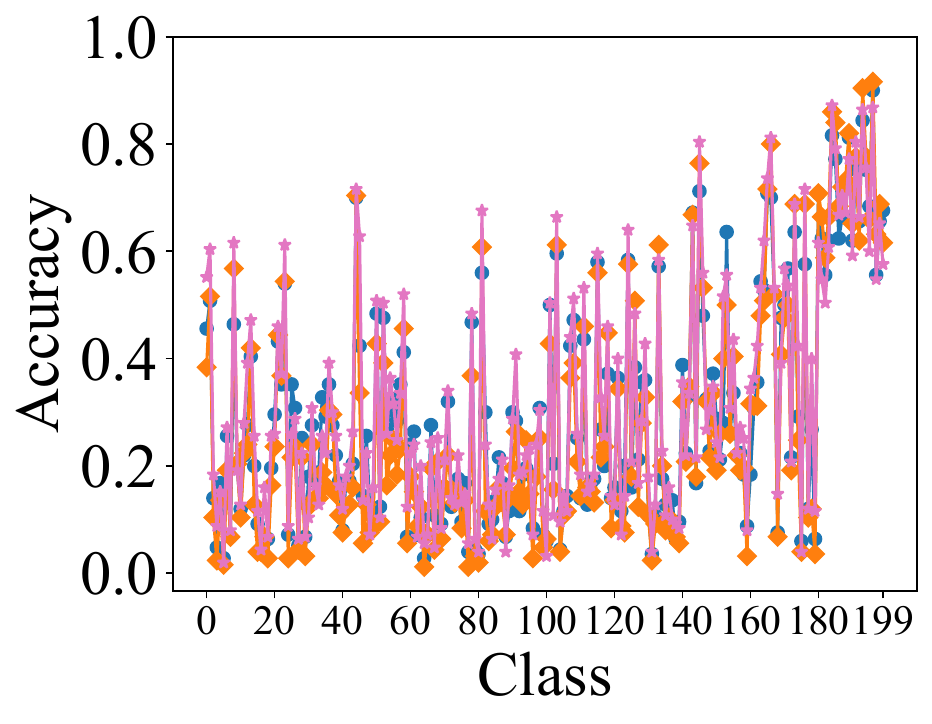}
        \vskip -0.1in
        \caption*{(b) Class-wise accuracy on the last task.}
        \label{fig:tiny_imagenet_forget_class}
    \end{minipage}
    \vskip -0.1in
    \caption{Comparison between the performance of ER, Mixup, and \method{} on the Tiny ImageNet dataset.}
    \label{fig:tiny_imagenet_forget}
\end{figure}

In addition, we perform additional experiments to analyze the rate of detrimental mixed samples that cause greater catastrophic forgetting when applying Mixup, compared to the experience replay method (ER). We compute the rate for each task and report the average rate across all tasks as shown in Table~\ref{tbl:neg_pair_rate}. As a result, an average of 47.3\% of randomly mixed samples can cause greater catastrophic forgetting than the original samples across five datasets, leading to a degradation in the accuracy of previous tasks.


\begin{table*}[h]
  \setlength{\tabcolsep}{12pt}
  \caption{Average rate of detrimental mixed samples on catastrophic forgetting when applying Mixup to the experience replay method (ER) on the MNIST, Fashion-MNIST (FMNIST), CIFAR-10, CIFAR-100, and Tiny ImageNet datasets.}
  \centering
  \begin{tabular}{l|ccccc}
  \toprule
    {Metric} & {\sf MNIST} & {\sf FMNIST} & {\sf CIFAR-10} & {\sf CIFAR-100} & {\sf Tiny ImageNet} \\
    \midrule
    {Average Rate} & 0.514 & 0.489 & 0.472 & 0.402 & 0.488 \\
    \bottomrule
  \end{tabular}
  \label{tbl:neg_pair_rate}
\end{table*}

\subsection{More Details on the Error when using Approximated Gradients of Mixed Samples}
\label{appendix:exp-error}

Continuing from Sec.~\ref{gradient_mixup}, we present approximation errors in inner products of the gradients of buffer data and mixed training data as shown in Fig.~\ref{fig:error_exp}. We compute both the approximated and true inner products when using the corresponding approximated and true gradients for the mixed samples and compare them using Root Mean Squared Error (RMSE) as a metric. A lower RMSE value means better approximation. If $\lambda$ is 0 or 1, the approximation error is 0 since the approximated gradient becomes the same as the gradient of the original sample. As the $\lambda$ value gets closer to 0.5, mixed samples are generated in the middle of original samples and the resulting approximation error increases. If $\lambda$ follows a uniform distribution between 0 and 1, the average approximation errors are 0.136, 0.110, 0.074, 0.035, and 0.037 for the MNIST, FMNIST, CIFAR-10, CIFAR-100, and Tiny ImageNet datasets, respectively. We note that the actual effect of the approximation errors on the selective mixup criterion is negligible because the sign of inner products is a key measure in \method{} rather than their exact values.
\vspace{-0.3cm}
\begin{figure}[H]
    \centering
    \begin{minipage}[t]{0.31\textwidth}
        \centering
        \includegraphics[width=\textwidth]{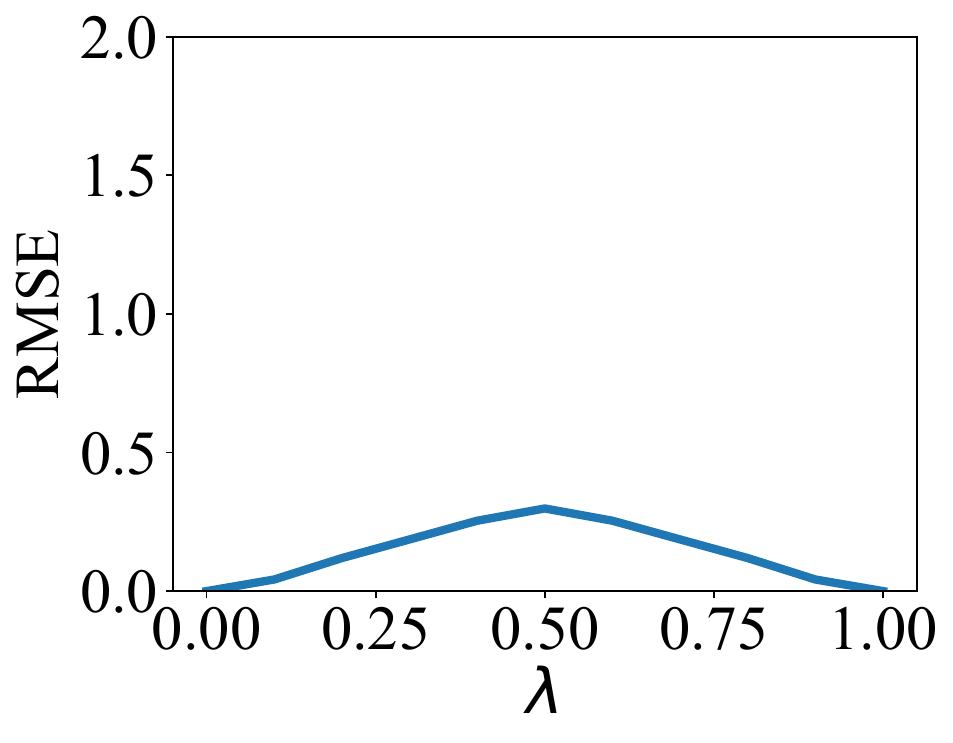}
        \vskip -0.1in
        \caption*{(a) MNIST.}
        \label{fig:error_mnist}
    \end{minipage}
    \begin{minipage}[t]{0.31\textwidth}
        \centering
        \includegraphics[width=\textwidth]{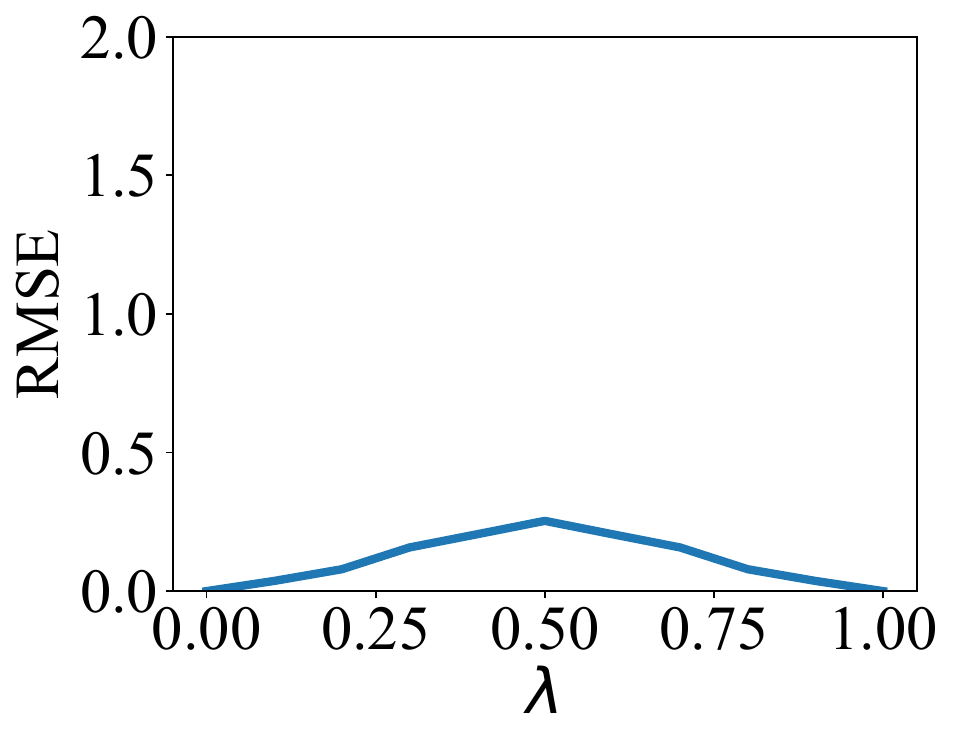}
        \vskip -0.1in
        \caption*{(b) FMNIST.}
        \label{fig:error_fmnist}
    \end{minipage}
    \begin{minipage}[t]{0.31\textwidth}
        \centering
        \includegraphics[width=\textwidth]{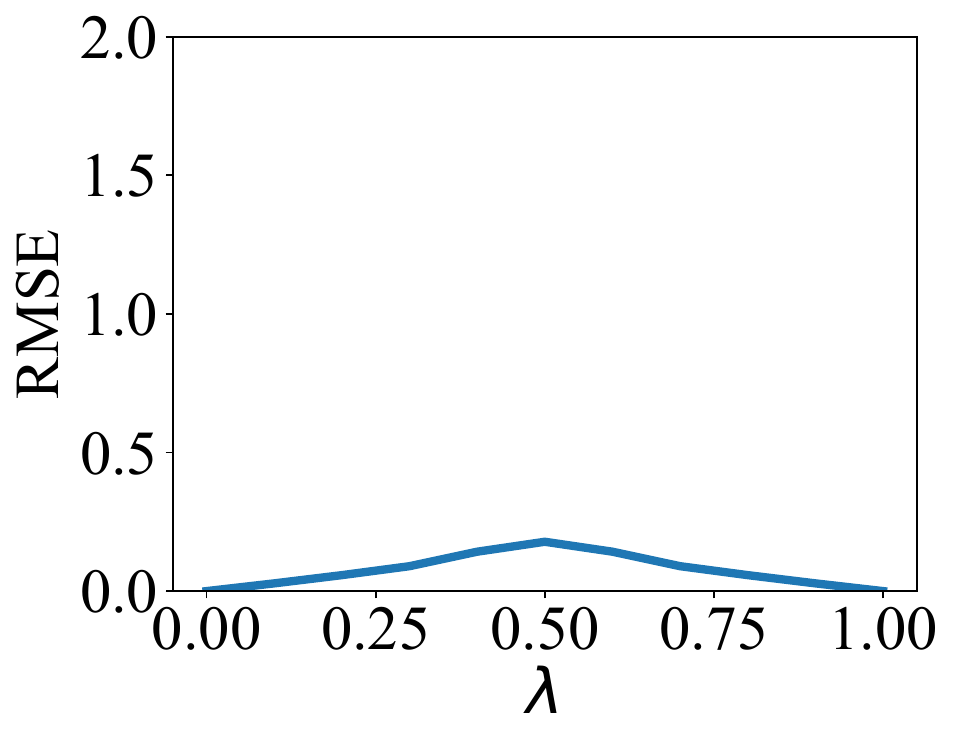}
        \vskip -0.1in
        \caption*{(c) CIFAR-10.}
        \label{fig:error_cifar-10}
    \end{minipage}
    \begin{minipage}[t]{0.31\textwidth}
        \centering
        \includegraphics[width=\textwidth]{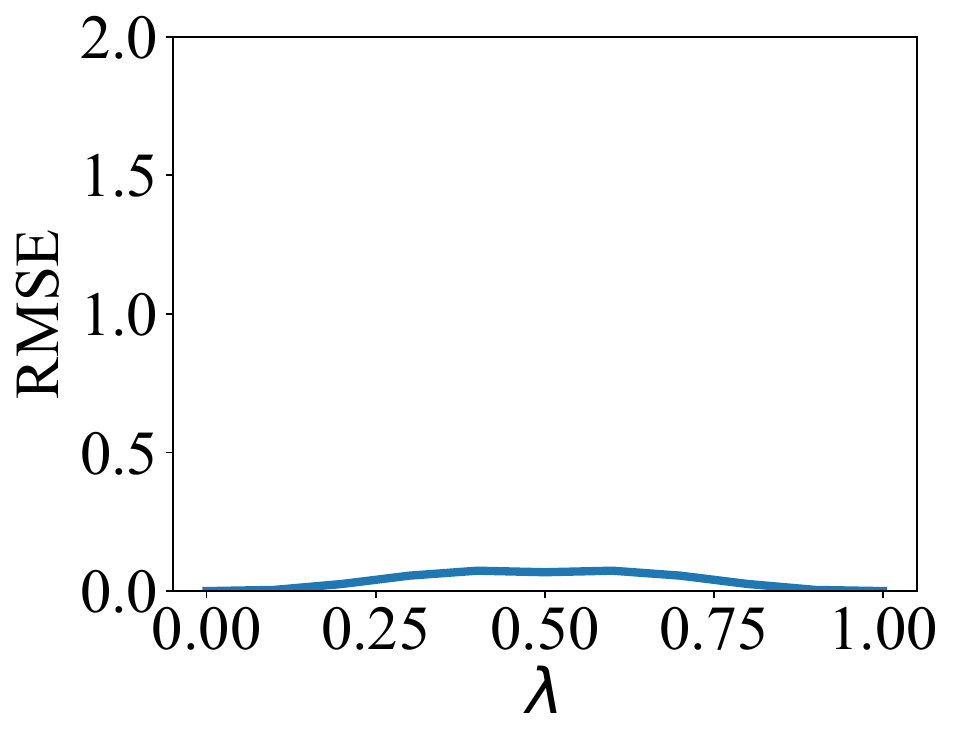}
        \vskip -0.1in
        \caption*{(d) CIFAR-100.}
        \label{fig:error_cifar-100}
    \end{minipage}
    \begin{minipage}[t]{0.31\textwidth}
        \centering
        \includegraphics[width=\textwidth]{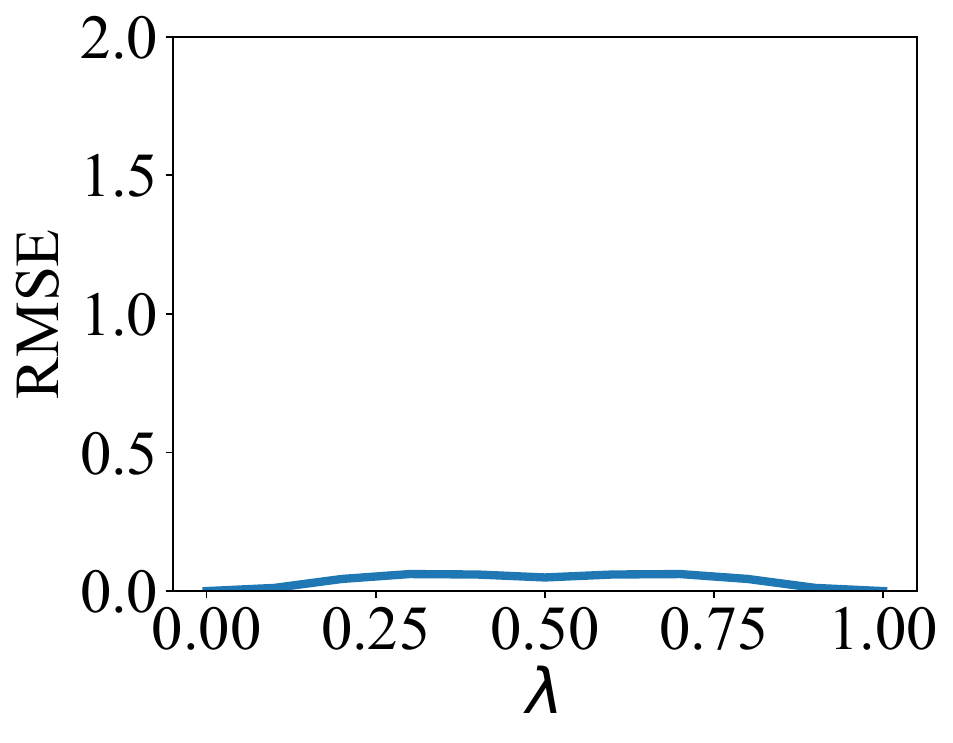}
        \vskip -0.1in
        \caption*{(e) Tiny ImageNet.}
        \label{fig:error_tiny-imagenet}
    \end{minipage}
    \vskip -0.1in
    \caption{Root Mean Squared Error (RMSE) between approximated and true inner products with respect to mixing ratio $\lambda$ on the MNIST, FMNIST, CIFAR-10, CIFAR-100, and Tiny ImageNet datasets.}
    \label{fig:error_exp}
    \vskip -0.1in
\end{figure}

In addition, we perform additional experiments to analyze the precision and recall of the optimal sample pairs derived by using the approximated gradients of mixed samples as shown in Table~\ref{tbl:precision_recall}. We define the optimal sample pairs derived by using the true gradients of mixed samples as true labels. As a result, the optimal sample pairs derived by using the approximated gradients have high precision and recall values for all datasets. This result suggests that using approximated gradients does not change the selective mixup criterion much.
\vspace{-0.1cm}

\begin{table}[H]
  \setlength{\tabcolsep}{13.3pt}
  \caption{Average precision and recall results for determining the optimal sample pairs using approximated gradients of mixed samples on the MNIST, FMNIST, CIFAR-10, CIFAR-100, and Tiny ImageNet datasets.} 
  \vspace{-0.2cm}
  \centering
  \begin{tabular}{l|ccccc}
  \toprule
    {Metric} & {\sf MNIST} & {\sf FMNIST} & {\sf CIFAR-10} & {\sf CIFAR-100} & {\sf Tiny ImageNet} \\
    \midrule
    {Precision} & 0.987 & 0.937 & 0.982 & 0.854 & 0.812 \\
    {Recall} & 0.847 & 0.852 & 0.911 & 0.827 & 0.835 \\
    \bottomrule
  \end{tabular}
  \label{tbl:precision_recall}
  \vspace{-0.1cm}
\end{table}

\subsection{More Details on Class-based Selective Mixup}
\label{appendix:exp-class}

Continuing from Sec.~\ref{alg}, we provide empirical results to show that training samples within the same class share similar gradients as shown in Fig.~\ref{fig:class_distance}. We use the cosine distance as a metric and define intra-class and inter-class distances as the average distances between gradients of samples within the same class and from different classes, respectively. As a result, intra-class distances are smaller than inter-class distances across all datasets, indicating that training samples within the same class share similar gradients. We believe that a difference between intra-class and inter-class distances still exists on large datasets because intra-class gradients are derived from similar embeddings within the same class, unlike the inter-class case.

\begin{figure}[H]
    \centering
    \includegraphics[width=0.6\textwidth]{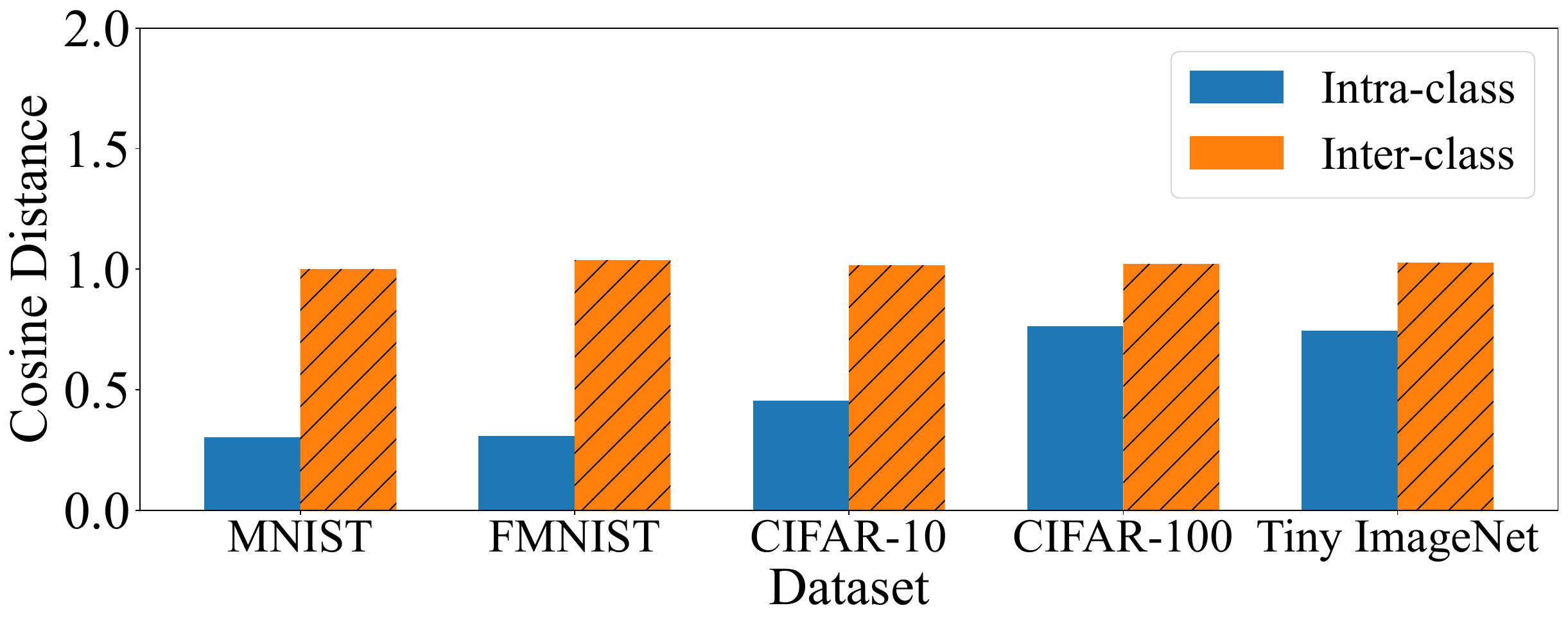}
    \vskip -0.1in
    \caption{Average cosine distances between gradients of data within the same class (intra-class) and from different classes (inter-class) on the MNIST, FMNIST, CIFAR-10, CIFAR-100, and Tiny ImageNet datasets.}
    \label{fig:class_distance}
    \vspace{-0.1cm}
\end{figure}

\subsection{More Details on Computational Complexity and Runtime Results of \method{}}
\label{appendix:exp-runtime}

Continuing from Sec.~\ref{alg}, we provide more details on an analysis of the computational complexity and runtime of the \method{} algorithm. The computational complexity of \method{} is quadratic to the number of classes when solving the optimization problem to obtain the optimal class pairs for each class. However, since we transform the data selection optimization problem into a class selection problem, the overhead is significantly reduced as the number of classes is much smaller than the number of samples. If the number of classes keeps increasing with consecutive tasks, the overhead of \method{} may also increase correspondingly. If there are too many classes, we can transform the repeated gradient inner product process into a nearest neighbor search problem and use locality-sensitive hashing (LSH)\,\cite{DBLP:conf/stoc/IndykM98, DBLP:conf/vldb/GionisIM99} to efficiently find an approximate solution with a time complexity of $\mathcal{O}(C \log C)$, where $C$ is the number of classes. 

We also provide runtime results for solving the optimization problem in \method{} per epoch with respect to the increasing number of tasks as shown in Fig.~\ref{fig:runtime_exp}. Since \method{} is not applied to the first task, we present the results starting from the second task. As the number of tasks increases, the number of classes also increases, which leads to a gradual increase in the runtime. In practice, solving the optimization problem per epoch takes a few seconds, resulting in a total overhead of a few minutes during training on the MNIST, FMNIST, CIFAR-10, CIFAR-100, and Tiny ImageNet datasets. If the number of classes becomes too large, we suggest clustering similar classes and assigning the selective mixup criterion based on clusters rather than classes as a potential solution to reduce the computational overhead.

\vspace{-0.2cm}

\begin{figure}[H]
    \centering
    \begin{minipage}[t]{0.3\textwidth}
        \centering
        \includegraphics[width=\textwidth]{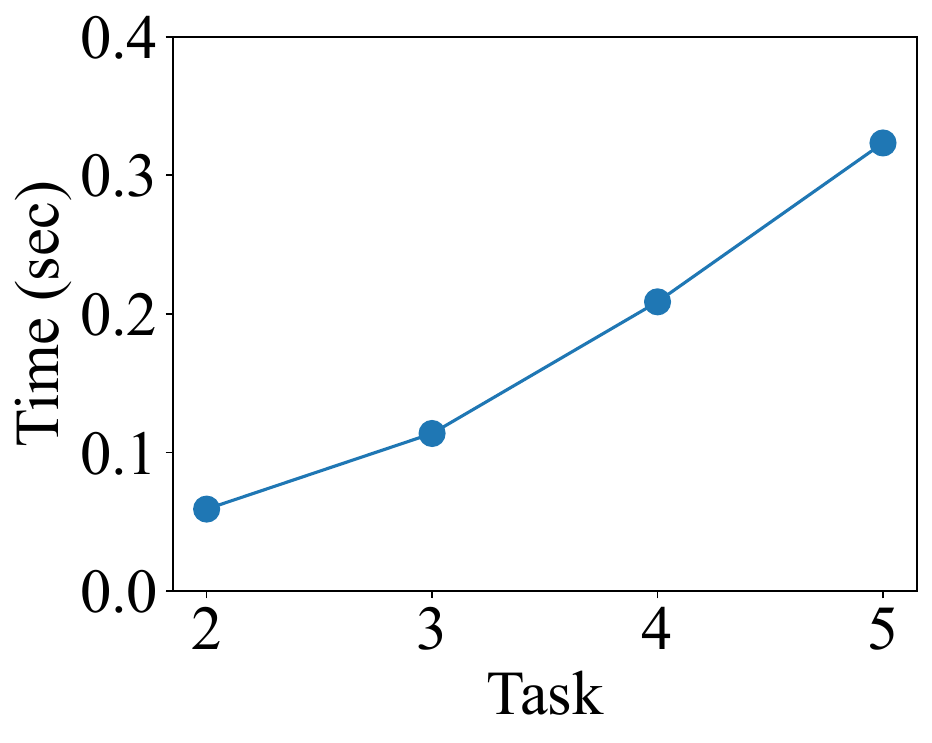}
        \vskip -0.1in
        \caption*{(a) MNIST.}
        \label{fig:runtime_mnist}
    \end{minipage}
    \begin{minipage}[t]{0.3\textwidth}
        \centering
        \includegraphics[width=\textwidth]{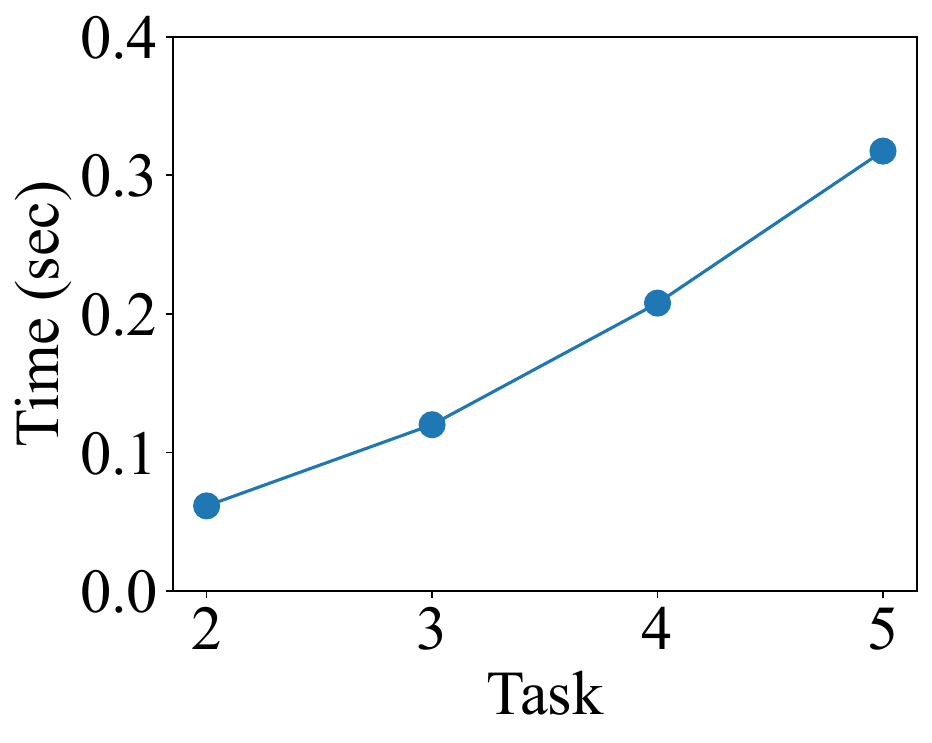}
        \vskip -0.1in
        \caption*{(b) FMNIST.}
        \label{fig:runtime_fmnist}
    \end{minipage}
    \begin{minipage}[t]{0.3\textwidth}
        \centering
        \includegraphics[width=\textwidth]{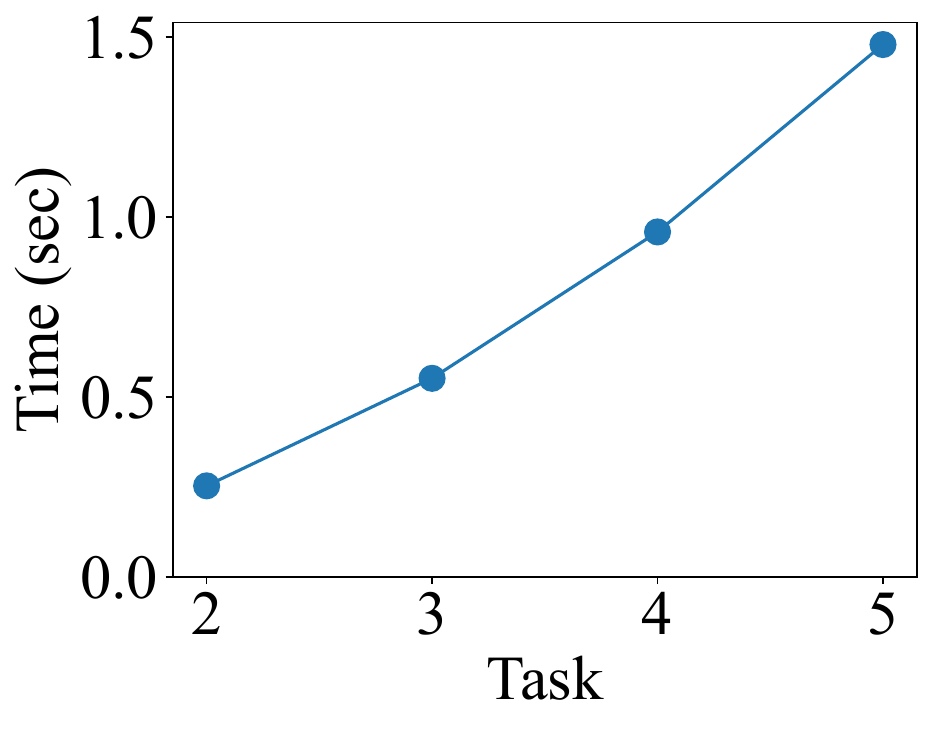}
        \vskip -0.1in
        \caption*{(c) CIFAR-10.}
        \label{fig:runtime_cifar-10}
    \end{minipage}
    \begin{minipage}[t]{0.3\textwidth}
        \centering
        \includegraphics[width=\textwidth]{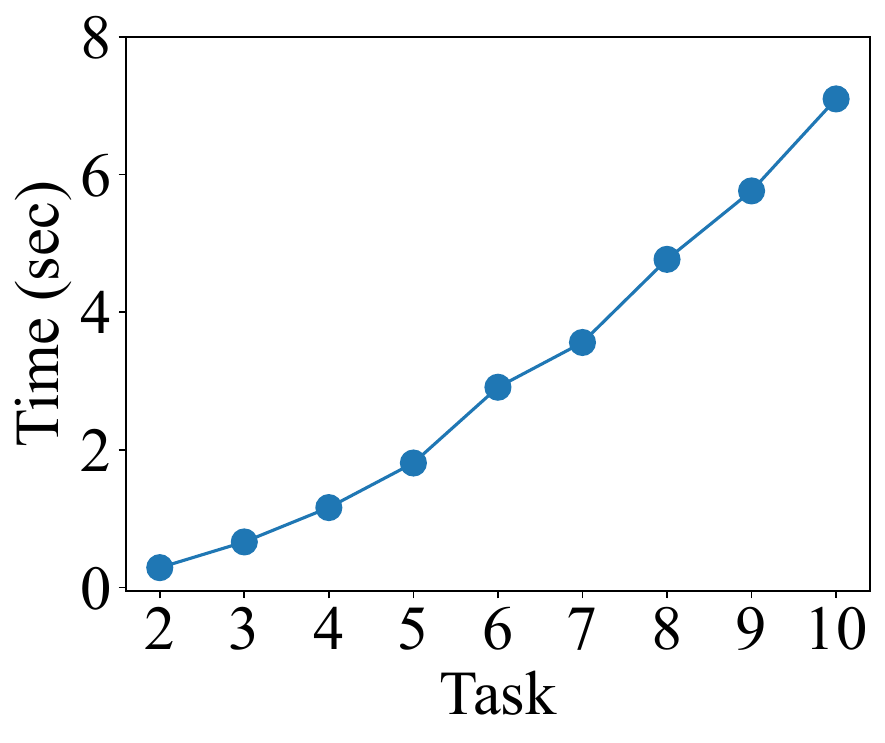}
        \vskip -0.1in
        \caption*{(d) CIFAR-100.}
        \label{fig:runtime_cifar-100}
    \end{minipage}
    \begin{minipage}[t]{0.31\textwidth}
        \centering
        \includegraphics[width=\textwidth]{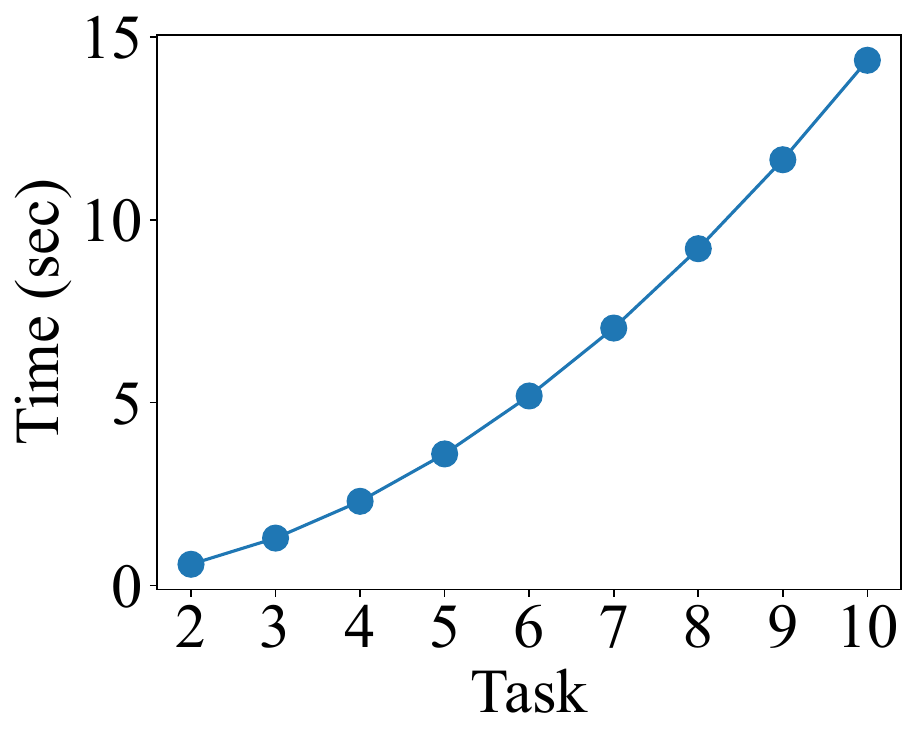}
        \vskip -0.1in
        \caption*{(e) Tiny ImageNet.}
        \label{fig:runtime_tiny-imagenet}
    \end{minipage}
    \vskip -0.1in
    \caption{Runtime results for solving the optimization problem in \method{} per epoch with respect to the number of tasks on the MNIST, FMNIST, CIFAR-10, CIFAR-100, and Tiny ImageNet datasets.}
    \label{fig:runtime_exp}
\end{figure}

In addition, we provide the average training time results over all tasks, both without and with GradMix, based on the experience replay method (ER) across all datasets as shown in Table~\ref{tbl:total_runtime_exp}. Note that other data augmentation baselines incur negligible computational overhead, resulting in training times similar to the method without GradMix. The computational overhead of GradMix depends on the last layer gradient computation, but our method computes gradients of only a small exemplar set for each class, not the whole training data. While GradMix incurs some additional computational cost relative to conventional training, we believe this is a worthwhile trade-off as it consistently leads to improved overall accuracy.

\begin{table*}[h]
  \caption{Average training time results over all tasks, both without and with \method{}, based on the experience replay method (ER) on the MNIST, FMNIST, CIFAR-10, CIFAR-100, and Tiny ImageNet datasets.}
  \centering
  \begin{tabular}{l|ccccc}
  \toprule
    {Method} & {\sf MNIST} & {\sf FMNIST} & {\sf CIFAR-10} & {\sf CIFAR-100} & {\sf Tiny ImageNet} \\
    \midrule
    {W/o GradMix} & 16 s & 19 s & 634 s & 1689 s & 3683 s \\
    {W/ GradMix} & 20 s & 24 s & 695 s & 2675 s & 4850 s \\
    \bottomrule
  \end{tabular}
  \label{tbl:total_runtime_exp}
\end{table*}

\subsection{More Details on Sensitivity of \method{} to the Mixup Parameter $\alpha$}
\label{appendix:exp-alpha}

Continuing from Sec.~\ref{parameters}, we perform additional experiments to analyze the sensitivity of \method{} to the Mixup parameter $\alpha$ as shown in Table~\ref{tbl:sensitivity_alpha}. As a result, the average accuracy of \method{} is not sensitive to $\alpha$. We believe that mixing the helpful pairs of classes that do not cause catastrophic forgetting might be beneficial to stabilize its performance regardless of $\alpha$. For Tiny ImageNet, a small value of $\alpha$ shows better accuracy results as in the Mixup paper\,\cite{DBLP:conf/iclr/ZhangCDL18}, and our selective mixup algorithm struggles to overcome the data characteristics.

\begin{table*}[h]
  \caption{Average accuracy results of \method{} when varying the $\alpha$ value to 0.4, 1.0, and 2.0 on the MNIST, FMNIST, CIFAR-10, CIFAR-100, and Tiny ImageNet datasets.} 
  \centering
  \begin{tabular}{l|ccccc}
  \toprule
    {Method} & {\sf MNIST} & {\sf FMNIST} & {\sf CIFAR-10} & {\sf CIFAR-100} & {\sf Tiny ImageNet} \\
    \midrule
    {GradMix ($\alpha = 0.4$)} & {0.919}\tiny{$\pm$0.005} & {0.795}\tiny{$\pm$0.006} & {0.657}\tiny{$\pm$0.010} & {0.548}\tiny{$\pm$0.006} & {0.405}\tiny{$\pm$0.006} \\
    {GradMix ($\alpha = 1.0$)} & {0.918}\tiny{$\pm$0.004} & {0.802}\tiny{$\pm$0.009} & {0.667}\tiny{$\pm$0.025} & {0.546}\tiny{$\pm$0.008} & {0.386}\tiny{$\pm$0.006} \\
    {GradMix ($\alpha = 2.0$)} & {0.917}\tiny{$\pm$0.001} & {0.809}\tiny{$\pm$0.007} & {0.660}\tiny{$\pm$0.020} & {0.543}\tiny{$\pm$0.009} & {0.372}\tiny{$\pm$0.005}
    \\
    \bottomrule
  \end{tabular}
  \label{tbl:sensitivity_alpha}
\end{table*}

\subsection{More Details on Compatibility of \method{} to Advanced Experience Replay Methods}
\label{appendix:exp-comp}

Continuing from Sec.~\ref{main-results}, we perform additional experiments when using MIR\,\cite{DBLP:conf/nips/AljundiBTCCLP19}, GSS\,\cite{DBLP:conf/nips/AljundiLGB19}, DER\,\cite{DBLP:conf/nips/BuzzegaBPAC20}, and FOSTER\,\cite{DBLP:conf/eccv/WangZYZ22} as basis experience replay methods for class-incremental learning, and the results are shown in Table~\ref{tbl:main_mir}--Table~\ref{tbl:main_foster}. As a result, \method{} still outperforms other data augmentation baselines when based on advanced experience replay methods. We believe that \method{} also can be extended to other experience replay methods including ASER\,\cite{DBLP:conf/aaai/ShimMJSKJ21} and SCR\,\cite{DBLP:conf/cvpr/MaiLKS21}.


\begin{table*}[h]
  \caption{Average accuracy results on the MNIST, FMNIST, CIFAR-10, CIFAR-100, and Tiny ImageNet datasets. We use MIR as a basis method of class-incremental learning and compare \method{} with three types of data augmentation baselines: (1) mixup-based methods: {\it Mixup}, {\it Manifold Mixup}, and {\it CutMix}; (2) imbalance-aware mixup methods: {\it Remix} and {\it Balanced-MixUp}; and (3) a policy-based method: {\it RandAugment}.} 
  \centering
  \begin{tabular}{c|l|ccccc}
  \toprule
    {Type} & {Method} & {\sf MNIST} & {\sf FMNIST} & {\sf CIFAR-10} & {\sf CIFAR-100} & {\sf Tiny ImageNet} \\
    \midrule
    {Experience Replay} & {MIR} & {0.900}\tiny{$\pm$0.006} & {0.752}\tiny{$\pm$0.006} & {0.620}\tiny{$\pm$0.009} & {0.493}\tiny{$\pm$0.021} & {0.386}\tiny{$\pm$0.009} \\
    \cmidrule{1-7}
    \multirow{3}{*}{Mixup-based} & {+ Mixup} & {0.878}\tiny{$\pm$0.010} & {0.770}\tiny{$\pm$0.006} & {0.644}\tiny{$\pm$0.008} & {0.504}\tiny{$\pm$0.011} & {0.372}\tiny{$\pm$0.009} \\
    & {+ Manifold Mixup} & {0.900}\tiny{$\pm$0.007} & {0.776}\tiny{$\pm$0.005} & {0.619}\tiny{$\pm$0.005} & {0.522}\tiny{$\pm$0.009} & {0.388}\tiny{$\pm$0.008} \\
    & {+ CutMix} & {0.884}\tiny{$\pm$0.005} & {0.787}\tiny{$\pm$0.006} & {0.514}\tiny{$\pm$0.010} & {0.406}\tiny{$\pm$0.010} & {0.349}\tiny{$\pm$0.006} \\
    \cmidrule{1-7}
    \multirow{2}{*}{Imbalance-aware} & {+ Remix} & {0.911}\tiny{$\pm$0.006} & {0.789}\tiny{$\pm$0.009} & {0.638}\tiny{$\pm$0.009} & {0.525}\tiny{$\pm$0.009} & {0.393}\tiny{$\pm$0.007} \\
    & {+ Balanced-MixUp} & {0.887}\tiny{$\pm$0.006} & {0.771}\tiny{$\pm$0.008} & {0.627}\tiny{$\pm$0.018} & {0.515}\tiny{$\pm$0.013} & {0.380}\tiny{$\pm$0.008} \\
    \cmidrule{1-7}
    {Policy-based} & {+ RandAugment} & {0.770}\tiny{$\pm$0.007} & {0.648}\tiny{$\pm$0.010} & {0.636}\tiny{$\pm$0.005} & {0.485}\tiny{$\pm$0.008} & {0.366}\tiny{$\pm$0.010} \\
    \cmidrule{1-7}
    {Selective Mixup} & {\bf + GradMix (ours)} & \textbf{{0.916}\tiny{$\pm$0.003}} & \textbf{{0.797}\tiny{$\pm$0.003}} & \textbf{{0.669}\tiny{$\pm$0.013}} & \textbf{{0.546}\tiny{$\pm$0.007}} & \textbf{{0.411}\tiny{$\pm$0.009}} \\
    \bottomrule
  \end{tabular}
  \label{tbl:main_mir}
\end{table*}

\begin{table*}[h]
  \caption{Average accuracy results on the MNIST, FMNIST, CIFAR-10, CIFAR-100, and Tiny ImageNet datasets. We use GSS as a basis method of class-incremental learning, and other experimental settings are the same as in Table~\ref{tbl:main_mir}.} 
  \centering
  \begin{tabular}{c|l|ccccc}
  \toprule
    {Type} & {Method} & {\sf MNIST} & {\sf FMNIST} & {\sf CIFAR-10} & {\sf CIFAR-100} & {\sf Tiny ImageNet} \\
    \midrule
    {Experience Replay} & {GSS} & {0.911}\tiny{$\pm$0.005} & {0.772}\tiny{$\pm$0.007} & {0.627}\tiny{$\pm$0.009} & {0.512}\tiny{$\pm$0.013} & {0.397}\tiny{$\pm$0.008} \\
    \cmidrule{1-7}
    \multirow{3}{*}{Mixup-based} & {+ Mixup} & {0.884}\tiny{$\pm$0.008} & {0.777}\tiny{$\pm$0.010} & {0.653}\tiny{$\pm$0.010} & {0.532}\tiny{$\pm$0.011} & {0.387}\tiny{$\pm$0.007} \\
    & {+ Manifold Mixup} & {0.912}\tiny{$\pm$0.006} & {0.781}\tiny{$\pm$0.006} & {0.628}\tiny{$\pm$0.009} & {0.543}\tiny{$\pm$0.007} & {0.404}\tiny{$\pm$0.007} \\
    & {+ CutMix} & {0.892}\tiny{$\pm$0.007} & {0.793}\tiny{$\pm$0.006} & {0.523}\tiny{$\pm$0.011} & {0.417}\tiny{$\pm$0.011} & {0.348}\tiny{$\pm$0.006} \\
    \cmidrule{1-7}
    \multirow{2}{*}{Imbalance-aware} & {+ Remix} & {0.904}\tiny{$\pm$0.007} & {0.797}\tiny{$\pm$0.008} & {0.647}\tiny{$\pm$0.010} & {0.549}\tiny{$\pm$0.008} & {0.406}\tiny{$\pm$0.006} \\
    & {+ Balanced-MixUp} & {0.886}\tiny{$\pm$0.007} & {0.776}\tiny{$\pm$0.007} & {0.642}\tiny{$\pm$0.012} & {0.538}\tiny{$\pm$0.010} & {0.393}\tiny{$\pm$0.007} \\
    \cmidrule{1-7}
    {Policy-based} & {+ RandAugment} & {0.768}\tiny{$\pm$0.009} & {0.659}\tiny{$\pm$0.008} & {0.643}\tiny{$\pm$0.008} & {0.497}\tiny{$\pm$0.008} & {0.381}\tiny{$\pm$0.009} \\
    \cmidrule{1-7}
    {Selective Mixup} & {\bf + GradMix (ours)} & \textbf{{0.927}\tiny{$\pm$0.005}} & \textbf{{0.808}\tiny{$\pm$0.007}} & \textbf{{0.674}\tiny{$\pm$0.011}} & \textbf{{0.563}\tiny{$\pm$0.007}} & \textbf{{0.418}\tiny{$\pm$0.006}} \\
    \bottomrule
  \end{tabular}
  \label{tbl:main_gss}
\end{table*}

\begin{table*}[h]
  \caption{Average accuracy results on the MNIST, FMNIST, CIFAR-10, CIFAR-100, and Tiny ImageNet datasets. We use DER as a basis method of class-incremental learning, and other experimental settings are the same as in Table~\ref{tbl:main_mir}.} 
  \centering
  \begin{tabular}{c|l|ccccc}
  \toprule
    {Type} & {Method} & {\sf MNIST} & {\sf FMNIST} & {\sf CIFAR-10} & {\sf CIFAR-100} & {\sf Tiny ImageNet} \\
    \midrule
    {Experience Replay} & {DER} & {0.927}\tiny{$\pm$0.004} & {0.794}\tiny{$\pm$0.007} & {0.678}\tiny{$\pm$0.010} & {0.560}\tiny{$\pm$0.009} & {0.422}\tiny{$\pm$0.008} \\
    \cmidrule{1-7}
    \multirow{3}{*}{Mixup-based} & {+ Mixup} & {0.902}\tiny{$\pm$0.006} & {0.801}\tiny{$\pm$0.009} & {0.698}\tiny{$\pm$0.011} & {0.574}\tiny{$\pm$0.008} & {0.413}\tiny{$\pm$0.007} \\
    & {+ Manifold Mixup} & {0.927}\tiny{$\pm$0.005} & {0.810}\tiny{$\pm$0.005} & {0.687}\tiny{$\pm$0.008} & {0.580}\tiny{$\pm$0.006} & {0.429}\tiny{$\pm$0.007} \\
    & {+ CutMix} & {0.913}\tiny{$\pm$0.005} & {0.819}\tiny{$\pm$0.007} & {0.655}\tiny{$\pm$0.010} & {0.445}\tiny{$\pm$0.009} & {0.392}\tiny{$\pm$0.006} \\
    \cmidrule{1-7}
    \multirow{2}{*}{Imbalance-aware} & {+ Remix} & {0.918}\tiny{$\pm$0.006} & {0.819}\tiny{$\pm$0.007} & {0.693}\tiny{$\pm$0.009} & {0.588}\tiny{$\pm$0.007} & {0.431}\tiny{$\pm$0.008} \\
    & {+ Balanced-MixUp} & {0.910}\tiny{$\pm$0.007} & {0.800}\tiny{$\pm$0.008} & {0.685}\tiny{$\pm$0.012} & {0.575}\tiny{$\pm$0.010} & {0.426}\tiny{$\pm$0.007} \\
    \cmidrule{1-7}
    {Policy-based} & {+ RandAugment} & {0.880}\tiny{$\pm$0.008} & {0.770}\tiny{$\pm$0.010} & {0.692}\tiny{$\pm$0.007} & {0.552}\tiny{$\pm$0.008} & {0.410}\tiny{$\pm$0.008} \\
    \cmidrule{1-7}
    {Selective Mixup} & {\bf + GradMix (ours)} & \textbf{{0.949}\tiny{$\pm$0.004}} & \textbf{{0.835}\tiny{$\pm$0.006}} & \textbf{{0.714}\tiny{$\pm$0.009}} & \textbf{{0.604}\tiny{$\pm$0.006}} & \textbf{{0.452}\tiny{$\pm$0.006}} \\
    \bottomrule
  \end{tabular}
  \label{tbl:main_der}
\end{table*}

\begin{table*}[h]
  \caption{Average accuracy results on the MNIST, FMNIST, CIFAR-10, CIFAR-100, and Tiny ImageNet datasets. We use FOSTER as a basis method of class-incremental learning, and other experimental settings are the same as in Table~\ref{tbl:main_mir}.} 
  \centering
  \begin{tabular}{c|l|ccccc}
  \toprule
    {Type} & {Method} & {\sf MNIST} & {\sf FMNIST} & {\sf CIFAR-10} & {\sf CIFAR-100} & {\sf Tiny ImageNet} \\
    \midrule
    {Dynamic Networks} & {FOSTER} & {0.950}\tiny{$\pm$0.003} & {0.835}\tiny{$\pm$0.005} & {0.732}\tiny{$\pm$0.010} & {0.618}\tiny{$\pm$0.007} & {0.512}\tiny{$\pm$0.007} \\
    \cmidrule{1-7}
    \multirow{3}{*}{Mixup-based} & {+ Mixup} & {0.928}\tiny{$\pm$0.005} & {0.837}\tiny{$\pm$0.007} & {0.746}\tiny{$\pm$0.011} & {0.631}\tiny{$\pm$0.009} & {0.505}\tiny{$\pm$0.006} \\
    & {+ Manifold Mixup} & {0.951}\tiny{$\pm$0.004} & {0.843}\tiny{$\pm$0.006} & {0.735}\tiny{$\pm$0.009} & {0.640}\tiny{$\pm$0.006} & {0.519}\tiny{$\pm$0.007} \\
    & {+ CutMix} & {0.936}\tiny{$\pm$0.004} & {0.852}\tiny{$\pm$0.006} & {0.701}\tiny{$\pm$0.010} & {0.502}\tiny{$\pm$0.008} & {0.486}\tiny{$\pm$0.006} \\
    \cmidrule{1-7}
    \multirow{2}{*}{Imbalance-aware} & {+ Remix} & {0.943}\tiny{$\pm$0.004} & {0.853}\tiny{$\pm$0.007} & {0.741}\tiny{$\pm$0.010} & {0.649}\tiny{$\pm$0.007} & {0.525}\tiny{$\pm$0.007} \\
    & {+ Balanced-MixUp} & {0.932}\tiny{$\pm$0.005} & {0.831}\tiny{$\pm$0.007} & {0.729}\tiny{$\pm$0.012} & {0.633}\tiny{$\pm$0.008} & {0.510}\tiny{$\pm$0.006} \\
    \cmidrule{1-7}
    {Policy-based} & {+ RandAugment} & {0.905}\tiny{$\pm$0.006} & {0.803}\tiny{$\pm$0.008} & {0.738}\tiny{$\pm$0.007} & {0.603}\tiny{$\pm$0.007} & {0.496}\tiny{$\pm$0.008} \\
    \cmidrule{1-7}
    {Selective Mixup} & {\bf + GradMix (ours)} & \textbf{{0.973}\tiny{$\pm$0.003}} & \textbf{{0.878}\tiny{$\pm$0.004}} & \textbf{{0.769}\tiny{$\pm$0.009}} & \textbf{{0.661}\tiny{$\pm$0.006}} & \textbf{{0.547}\tiny{$\pm$0.006}} \\
    \bottomrule
  \end{tabular}
  \label{tbl:main_foster}
\end{table*}

\subsection{More Details on Sequential Performance Results}
\label{appendix:exp-seq}

Continuing from Sec.~\ref{main-results}, we provide sequential accuracy results with respect to tasks as shown in Fig.~\ref{fig:seq_exp}. As a result, \method{} shows better accuracy performance than other data augmentation baselines at each task. 

\begin{figure}[H]
    \centering
    \begin{minipage}[t]{0.7\textwidth}
        \centering
        \includegraphics[width=\textwidth]{icml2024/legend_seq_acc.pdf}
    \end{minipage}
    \begin{minipage}[t]{0.3\textwidth}
        \centering
        \includegraphics[width=\textwidth]{icml2024/seq_acc_mnist.pdf}
        \vskip -0.1in
        \caption*{(a) MNIST.}
        \label{fig:seq_mnist}
    \end{minipage}
    \begin{minipage}[t]{0.3\textwidth}
        \centering
        \includegraphics[width=\textwidth]{icml2024/seq_acc_fmnist.pdf}
        \vskip -0.1in
        \caption*{(b) FMNIST.}
        \label{fig:seq_fmnist}
    \end{minipage}
    \begin{minipage}[t]{0.3\textwidth}
        \centering
        \includegraphics[width=\textwidth]{icml2024/seq_acc_cifar10.pdf}
        \vskip -0.1in
        \caption*{(c) CIFAR-10.}
        \label{fig:seq_cifar-10}
    \end{minipage}
    \vspace{0.01\textwidth}
    \begin{minipage}[t]{0.3\textwidth}
        \centering
        \includegraphics[width=\textwidth]{icml2024/seq_acc_cifar100.pdf}
        \vskip -0.1in
        \caption*{(d) CIFAR-100.}
        \label{fig:seq_cifar-100}
    \end{minipage}
    \begin{minipage}[t]{0.3\textwidth}
        \centering
        \includegraphics[width=\textwidth]{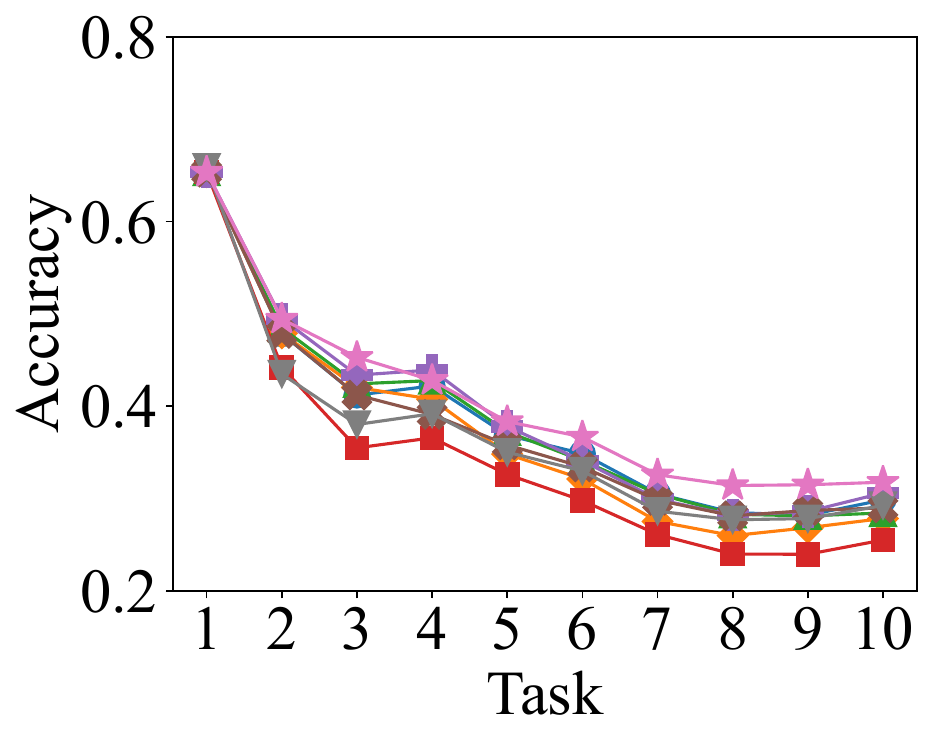}
        \vskip -0.1in
        \caption*{(e) Tiny ImageNet.}
        \label{fig:seq_tiny-imagenet}
    \end{minipage}
    \vskip -0.1in
    \caption{Sequential accuracy results of \method{} and baselines on the five datasets as new tasks are learned.}
    \label{fig:seq_exp}
\end{figure}


\subsection{More Details on the Effectiveness of \method{} on Large-Scale Datasets}
\label{appendix:large_scale}

Continuing from Sec.~\ref{main-results}, we perform additional experiments to demonstrate the effectiveness of GradMix on large-scale datasets. We use the ImageNet-1K\,\cite{DBLP:conf/cvpr/DengDSLL009} dataset with 10 incremental tasks and provide the average accuracy results in Table~\ref{tbl:main_large}. Using ER as a basis method, the experimental results show that GradMix continues to outperform other data augmentation baselines on the large-scale dataset. In general, GradMix is scalable by only computing the last layer gradients of a small exemplar set for each class.

\begin{table}[H]
  \caption{Average accuracy results on the ImageNet-1K dataset. We use ER as a basis method of class-incremental learning, and other experimental settings are the same as in Table~\ref{tbl:main_mir}.}
  \centering
  \begin{tabular}{c|l|c}
  \toprule
    {Type} & {Method} & {\sf ImageNet-1K} \\
    \midrule
    {Experience Replay} & {ER} & {0.326}\tiny{$\pm$0.012} \\
    \cmidrule{1-3}
    \multirow{3}{*}{Mixup-based} & {+ Mixup} & {0.313}\tiny{$\pm$0.013} \\
    & {+ Manifold Mixup} & {0.328}\tiny{$\pm$0.015} \\
    & {+ CutMix} & {0.281}\tiny{$\pm$0.016} \\
    \cmidrule{1-3}
    \multirow{2}{*}{Imbalance-aware} & {+ Remix} & {0.335}\tiny{$\pm$0.014} \\
    & {+ Balanced-MixUp} & {0.322}\tiny{$\pm$0.012} \\
    \cmidrule{1-3}
    {Policy-based} & {+ RandAugment} & {0.305}\tiny{$\pm$0.013} \\
    \cmidrule{1-3}
    {Selective Mixup} & {\bf + GradMix (ours)} & \textbf{{0.356}\tiny{$\pm$0.010}} \\
    \bottomrule
  \end{tabular}
  \label{tbl:main_large}
\end{table}

\subsection{More Details on Varying the Backbone Model}
\label{appendix:backbone_model}

Continuing from Sec.~\ref{main-results}, we perform additional experiments when using an initialized ViT\,\cite{DBLP:conf/iclr/DosovitskiyB0WZ21} as the backbone model on the CIFAR-100 and Tiny ImageNet datasets. The average accuracy results are shown in Table~\ref{tbl:main_backbone}, demonstrating the effectiveness of GradMix compared to the baselines when combined with ViT. However, the overall performance with ViT is lower than that with ResNet-18, which indicates that ViT may not be well-suited for continual learning settings with limited buffer data due to constraints in memory capacity, computational efficiency, and privacy. Nevertheless, transformer architectures specifically designed for continual learning, such as DyTox\,\cite{DBLP:conf/cvpr/DouillardRCC22}, have demonstrated strong performance by dynamically expanding task-specific tokens, suggesting the potential of tailored transformer designs in overcoming the limitations observed with standard ViT backbones. In addition, we believe that using a pre-trained model as the initial model would further improve overall performance, while GradMix would still be effectively adopted, as it relies on a gradient-based criterion adaptively derived from both data and the model.

\begin{table}[H]
  \caption{Average accuracy results on the CIFAR-100 and Tiny ImageNet datasets with ResNet-18 and ViT backbone models. We use ER as a basis method of class-incremental learning, and other experimental settings are the same as in Table~\ref{tbl:main_mir}.}
  \centering
  \begin{tabular}{c|l|cc|cc}
  \toprule
  {Type} & {Method} & \multicolumn{2}{c|}{{\sf CIFAR-100}} & \multicolumn{2}{c}{{\sf Tiny ImageNet}} \\
  \cmidrule{1-6}
  {} & {} & {ResNet-18} & {ViT} & {ResNet-18} & {ViT} \\
  \midrule
  {Experience Replay} & {ER} & {0.491}\tiny{$\pm$0.019} & {0.371}\tiny{$\pm$0.015} & {0.385}\tiny{$\pm$0.009} & {0.275}\tiny{$\pm$0.010} \\
  \cmidrule{1-6}
  \multirow{3}{*}{Mixup-based} & {+ Mixup} & {0.507}\tiny{$\pm$0.010} & {0.402}\tiny{$\pm$0.012} & {0.371}\tiny{$\pm$0.008} & {0.284}\tiny{$\pm$0.009} \\
  & {+ Manifold Mixup} & {0.521}\tiny{$\pm$0.008} & {0.410}\tiny{$\pm$0.011} & {0.386}\tiny{$\pm$0.009} & {0.298}\tiny{$\pm$0.011} \\
  & {+ CutMix} & {0.404}\tiny{$\pm$0.009} & {0.315}\tiny{$\pm$0.014} & {0.344}\tiny{$\pm$0.010} & {0.236}\tiny{$\pm$0.011} \\
  \cmidrule{1-6}
  \multirow{2}{*}{Imbalance-aware} & {+ Remix} & {0.528}\tiny{$\pm$0.007} & {0.426}\tiny{$\pm$0.010} & {0.391}\tiny{$\pm$0.008} & {0.305}\tiny{$\pm$0.009} \\
  & {+ Balanced-MixUp} & {0.516}\tiny{$\pm$0.009} & {0.396}\tiny{$\pm$0.013} & {0.379}\tiny{$\pm$0.008} & {0.281}\tiny{$\pm$0.010} \\
  \cmidrule{1-6}
  {Policy-based} & {+ RandAugment} & {0.484}\tiny{$\pm$0.007} & {0.353}\tiny{$\pm$0.012} & {0.368}\tiny{$\pm$0.007} & {0.263}\tiny{$\pm$0.008} \\
  \cmidrule{1-6}
  {Selective Mixup} & {\bf + GradMix (ours)} & \textbf{{0.546}\tiny{$\pm$0.008}} & \textbf{{0.460}\tiny{$\pm$0.012}} & \textbf{{0.405}\tiny{$\pm$0.006}} & \textbf{{0.327}\tiny{$\pm$0.010}} \\
  \bottomrule
  \end{tabular}
  \label{tbl:main_backbone}
\end{table}


\subsection{More Details on Selective Mixup Analysis}
\label{appendix:analysis}

Continuing from Sec.~\ref{analysis}, we provide more details on the selective mixup analysis as shown in Fig.~\ref{fig:heatmap_exp}.


\begin{figure}[H]
    \centering
    \begin{minipage}[t]{0.31\textwidth}
        \centering
        \includegraphics[width=\textwidth]{icml2024/heatmap_mnist.pdf}
        \vskip -0.1in
        \caption*{(a) MNIST.}
        \label{fig:heatmap_mnist}
    \end{minipage}
    \begin{minipage}[t]{0.31\textwidth}
        \centering
        \includegraphics[width=\textwidth]{icml2024/heatmap_fmnist.pdf}
        \vskip -0.1in
        \caption*{(b) FMNIST.}
        \label{fig:heatmap_fmnist}
    \end{minipage}
    \begin{minipage}[t]{0.31\textwidth}
        \centering
        \includegraphics[width=\textwidth]{icml2024/heatmap_cifar10.pdf}
        \vskip -0.1in
        \caption*{(c) CIFAR-10.}
        \label{fig:heatmap_cifar-10}
        \vspace{+0.3cm}
    \end{minipage}
    \begin{minipage}[t]{0.31\textwidth}
        \centering
        \includegraphics[width=\textwidth]{icml2024/heatmap_cifar100.pdf}
        \vskip -0.1in
        \caption*{(d) CIFAR-100.}
        \label{fig:heatmap_cifar-100}
    \end{minipage}
    \begin{minipage}[t]{0.31\textwidth}
        \centering
        \includegraphics[width=\textwidth]{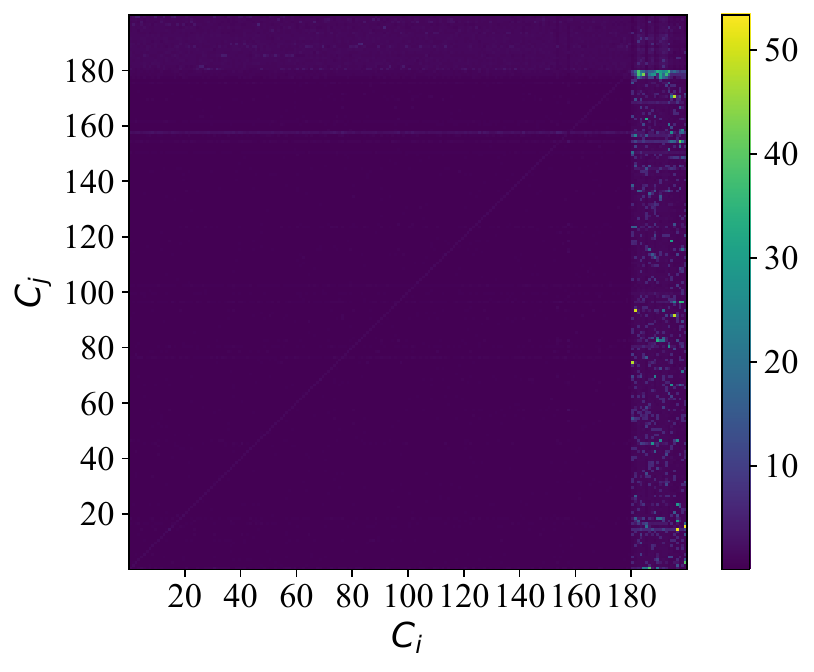}
        \vskip -0.1in
        \caption*{(e) Tiny ImageNet.}
        \label{fig:heatmap_tiny-imagenet}
    \end{minipage}
    \vskip -0.1in
    \caption{Selective mixup results of \method{} on the last task of the five datasets.}
    \label{fig:heatmap_exp}
\end{figure}

\section{Appendix -- More Related Work}
\label{appendix:related_work}

Continuing from Sec.~\ref{related_work}, we provide a more detailed discussion of the related work.

\paragraph{Data Augmentation}

Data augmentation is an effective method to enhance the generalization and robustness of models by generating additional training data. Traditional data augmentation techniques include image transformations such as random rotation, flipping, and cropping, which introduce variations to the input images\,\cite{DBLP:journals/jbd/ShortenK19, DBLP:journals/corr/abs-2204-08610}. Recent works propose an automatic strategy to find an effective data augmentation policy\,\cite{DBLP:conf/cvpr/CubukZMVL19, DBLP:conf/nips/CubukZS020, DBLP:conf/iccv/MullerH21}. However, these policy-based approaches do not consider continual learning scenarios, which may result in severe catastrophic forgetting as we show in our experiments.

Another line of research is Mixup\,\cite{DBLP:conf/iclr/ZhangCDL18}, which generates new synthetic data by linearly interpolating between pairs of features and their corresponding labels, and there are variant mixup-based methods\,\cite{DBLP:conf/icml/VermaLBNMLB19, DBLP:conf/iccv/YunHCOYC19, harris2021fmixenhancingmixedsample}. In addition, imbalance-aware mixup methods have been proposed to improve model robustness on imbalanced datasets. A recent study\,\cite{DBLP:conf/eccv/ChouCPWJ20} relaxes the mixing factor of labels to assign higher weights to the minority class while features are mixed in the same way as vanilla Mixup. Another study\,\cite{DBLP:conf/miccai/GaldranCB21} selects sample pairs for mixup using both instance-based and class-based sampling in order to be robust for learning highly imbalanced medical image datasets. Furthermore, selective mixup is introduced by selecting sample pairs for mixup based on specific criteria rather than random, and it improves the model's ability for domain adaptation\,\cite{DBLP:conf/aaai/XuZNLWTZ20} and out-of-distribution (OOD) generalization\,\cite{DBLP:conf/icml/Yao0LZL0F22}.

A recent line of research applies existing data augmentation methods in class-incremental learning to further improve accuracy. The first work\,\cite{DBLP:conf/cvpr/MiKLYF20} simply combines vanilla Mixup to experience replay methods. Another work\,\cite{DBLP:conf/cvpr/BangKY0C21} employs the combination of existing data augmentation methods to enhance the diversity of samples in a buffer. In addition, a recent work\,\cite{DBLP:conf/nips/KumariW0B22} proposes adversarial perturbation technique that moves buffer data towards the current task data and utilizes Mixup within buffer data before adding perturbation for more diversity. Extending this line of work,\,\cite{DBLP:conf/aaai/QiangH0L0Z23} improves the sampling and mixing factor strategies in Mixup to address distribution discrepancies between the imbalanced training set and the balanced test set in continual learning, which is similar to existing imbalance-aware mixup methods\,\cite{DBLP:conf/eccv/ChouCPWJ20, DBLP:conf/miccai/GaldranCB21}. Another work\,\cite{DBLP:conf/nips/KimXKK022} addresses class-incremental learning by using out-of-distribution detection methods that utilize image rotation as a data augmentation technique. Another line of research is using generative models to generate data for previous tasks\,\cite{DBLP:conf/nips/ShinLKK17, DBLP:conf/icml/GaoL23a, DBLP:conf/cvpr/KimCKTB24}, but this approach incurs additional costs for continuously training generative models across tasks. Also, there is a work\,\cite{DBLP:conf/nips/ZhuCZL21} to achieve comparable performance with experience replay methods without storing buffer data by utilizing class augmentation and semantic augmentation techniques. In the setting of online continual learning, a recent study\,\cite{DBLP:conf/nips/ZhangPFBLJ22} repeatedly trains on buffer data for each incoming task data and uses random augmentation to alleviate overfitting caused by the repeated data. However, all these works simply use existing data augmentation techniques for further improvements and do not propose any novel data augmentation methods specifically designed for continual learning scenarios. In comparison, \method{} is a robust data augmentation method designed to minimize catastrophic forgetting in class-incremental learning. Although adversarial attack-based methods are studied for feature augmentation in class-incremental learning\,\cite{DBLP:conf/aaai/KimPH24, DBLP:conf/icml/Zheng0YZ24}, they are generally regarded as a distinct line of research compared to Mixup-based approaches in the literature\,\cite{DBLP:conf/iclr/ZhangCDL18, DBLP:conf/icml/VermaLBNMLB19, DBLP:conf/iccv/YunHCOYC19, harris2021fmixenhancingmixedsample}. Since our work focuses specifically on Mixup techniques for class-incremental learning, we do not include adversarial attack-based methods in this study.

\section{Appendix -- Limitations}
\label{appendix:limitations}

The computational overhead of GradMix inevitably grows as the number of classes increases over time in class-incremental learning. Despite this limitation, we believe this is a worthwhile trade-off as GradMix consistently achieves improved overall accuracy. However, GradMix may not be effective in tasks where mixing samples is impractical, such as semantic segmentation or tasks involving discrete symbolic features, where interpolating either the feature or the label may lead to semantically inconsistent samples.

\end{document}